\documentclass{article}
\usepackage{iclr2023_conference,times}
\usepackage{graphicx} % Required for inserting images
\usepackage{amsmath,amsfonts,bm}

\def\eqref#1{equation~\ref{#1}}
\def\1{\bm{1}}

\def\ra{{\textnormal{a}}}

\def\rv{{\textnormal{v}}}
\def\rw{{\textnormal{w}}}

\def\rva{{\mathbf{a}}}

\def\rvv{{\mathbf{v}}}
\def\rvw{{\mathbf{w}}}

\def\vzero{{\bm{0}}}

\def\vtheta{{\bm{\theta}}}

\def\ve{{\bm{e}}}
\def\vf{{\bm{f}}}

\def\vh{{\bm{h}}}

\def\vr{{\bm{r}}}

\def\vw{{\bm{w}}}
\def\vx{{\bm{x}}}

\def\mA{{\bm{A}}}

\def\mI{{\bm{I}}}
\def\mJ{{\bm{J}}}

\def\mM{{\bm{M}}}

\def\mW{{\bm{W}}}

\DeclareMathAlphabet{\mathsfit}{\encodingdefault}{\sfdefault}{m}{sl}
\SetMathAlphabet{\mathsfit}{bold}{\encodingdefault}{\sfdefault}{bx}{n}

\def\sH{{\mathbb{H}}}

\def\sP{{\mathbb{P}}}

\newcommand{\E}{\mathbb{E}}

\newcommand{\R}{\mathbb{R}}

\usepackage[colorinlistoftodos]{todonotes}
\usepackage{hyperref}
\usepackage{url}
\usepackage{csquotes}
\usepackage[mathscr]{euscript}
\usepackage[english]{babel}
\usepackage{amsmath,amsfonts,amssymb}
\usepackage[noabbrev]{cleveref}
\usepackage{multirow}
\usepackage{booktabs}
\usepackage{wrapfig}
\usepackage{amsthm}
\newtheorem{theorem}{Theorem}
\newtheorem{lemma}{Lemma}
\newtheorem{remark}{Remark}
\newtheorem{corollary}{Corollary}
\usepackage{caption}
\usepackage{subcaption}
\usepackage{mathtools}
\usepackage[inline]{enumitem}
\usepackage{cleveref}
\usepackage[T1]{fontenc}
\usepackage{siunitx}
\usepackage{lipsum}

\crefrangeformat{enumi}{items #3#1#4--#5#2#6}

\makeatletter
\newtheorem*{rep@theorem}{\rep@title}
\newcommand{\newreptheorem}[2]{%
\newenvironment{rep#1}[1]{%
 \def\rep@title{#2 \ref{##1}}%
 \begin{rep@theorem}}%
 {\end{rep@theorem}}}
\makeatother

\newreptheorem{lemma}{Lemma}
\newreptheorem{theorem}{Theorem}
\newreptheorem{corollary}{Corollary}
\DeclareMathOperator{\tr}{tr}

\title{Out-of-Distribution Detection by Leveraging Between-Layer Transformation Smoothness}
\author{Fran Jelenić$^{1,2}$ \quad Josip Jukić$^{1,2}$ \quad Martin Tutek$^3$ \quad Mate Puljiz$^2$ \quad Jan Šnajder$^{1,2}$ \\
$^1$TakeLab, $^2$Faculty of Electrical Engineering and Computing, University of Zagreb, Croatia\\ 
$^3$UKP Lab, Technical University of Darmstadt, Germany\\
\texttt{\{fran.jelenic, josip.jukic, mate.puljiz, jan.snajder\}@fer.hr}\\
\texttt{tutek@ukp.informatik.tu-darmstadt.de}
}

\newcommand{\nospacetext}[1]{\makebox[0pt][l]{#1}}

\newcommand{\ourmethod}{\textsc{BLOOD}}
\newcommand{\ourmethodL}{\textsc{BLOOD}$_{L}$}
\newcommand{\ourmethodmean}{\textsc{BLOOD}$_{M}$}
\newcommand{\whitebox}{white/black-box}

\iclrfinalcopy % Uncomment for camera-ready version, but NOT for submission.
\begin{document}

\maketitle

\begin{abstract}

Effective out-of-distribution (OOD) detection is crucial for reliable machine learning models, yet most current methods are limited in practical use due to requirements like access to training data or intervention in training. We present a novel method for detecting OOD data in Transformers based on transformation smoothness between intermediate layers of a network (\ourmethod{}), which is applicable to pre-trained models without access to training data. \ourmethod{} utilizes the tendency of between-layer representation transformations of in-distribution (ID) data to be smoother than the corresponding transformations of OOD data, a property that we also demonstrate empirically. We evaluate \ourmethod{} on several text classification tasks with Transformer networks and demonstrate that it outperforms methods with comparable resource requirements. Our analysis also suggests that when learning simpler tasks, OOD data transformations maintain their original sharpness, whereas sharpness increases with more complex tasks.

%We present a novel method for detecting out-of-distribution (OOD) data in deep neural networks (DNNs) based on the difference in smoothness of how the representations of OOD and in-distribution (ID) data evolve between layers. As our approach leverages the innate property of DNNs, it does not require any interventions in the standard DNN training procedure, nor does it assume the availability of training data during evaluation\martin{1. improvement, 2. explanation; "Our approach does not require training data nor does it intervene... as it is based only on..."}. Our method outperforms existing OOD detection baselines with comparably few prerequisites and even some with much more prerequisites. Our analysis revealed differences in the evolution of OOD and ID representations through the model's layers, what causes that difference, and how different distributional shifts impact that difference. 

%\martin{Requires significant rewrite. What are the contributions? What is the domain (NLP, tasks)? In which scenarios does the method outperform baselines? Which prerequisites, and why are the prerequisites a bad thing? Clarify evolution (smoothness). Don't make too general statements (DNNs $\rightarrow$ Transformer networks for NLP)}

\end{abstract}

\section{Introduction}

Machine learning (ML) models' success rests on the assumption that the model will be evaluated on data that comes from the same distribution as the data on which it was trained, the \emph{in-distribution} (ID) data.
However, models deployed in noisy and imperfect real-world scenarios often face data that comes from a different distribution, the \emph{out-of-distribution} (OOD) data, which can hinder the models' performance. 
%On the flip side, discerning between ID and OOD data makes it possible to avoid making decisions regarding data that models are ill-equipped to handle, and instead, delegate the decision-making process to human experts. 
The task of discerning between ID and OOD data is commonly referred to as \emph{OOD detection} \citep{yang2021generalized}.

%Typically, the identification of OOD instances relies on uncertainty estimates \martin{estimating the (underlying) uncertainty of the model?} \citep{ovadia2019can}. The intuition is that a well-performing model should exhibit higher confidence when dealing with data resembling the training data, i.e., ID data.
%However, ML models are susceptible to two sources of uncertainty.
%\textit{Aleatoric} uncertainty stems from the inherent ambiguity and noise in the data, which is primarily associated with ambiguous ID data.
%In contrast, \textit{epistemic} uncertainty stems from the lack of relevant information in the training data, making it characteristic of OOD data \citep{der2009aleatory,kendall2017uncertainties}.
%The total amount of uncertainty about a given prediction, i.e., aleatoric and epistemic uncertainty, is called \textit{predictive uncertainty} \citep{gal2016uncertainty}.
%Aleatoric uncertainty stems from the inherent ambiguity and noise in the data, and it is thus irreducible by the acquisition of more data and characteristic of ambiguous ID data; In contrast, epistemic uncertainty stems from the lack of relevant information in the training data, and it is thus reducible by acquiring more relevant data and characteristic of the OOD data \citep{der2009aleatory,kendall2017uncertainties}.

Owing to their consistent state-of-the-art performance across diverse ML tasks \citep{abiodun2018state}, Deep Neural Networks (DNNs) have garnered significant attention in OOD detection research. 
While popular baselines make use of the model's posterior class probabilities \citep{hendrycksbaseline}, the issue of overconfidence in DNNs \citep{guo2017calibration} frequently erodes the credibility of these probabilities.
%However, this approach cannot reliably disentangle the epistemic uncertainty from the model's predictive uncertainty \citep{kirsch2021pitfalls}.
%Another problem with using posterior class probabilities is the overconfidence of DNNs \citep{guo2017calibration}.
An alternative is offered by the group of methods that leverage the fundamental concept of DNNs, namely, representation learning.
Because a DNN encodes similar instances closely in its representation space, an OOD instance can be identified based on the distance between its representation and the representations of other instances in the training set \citep{lee2018simple}.
%A popular way to quantify the epistemic uncertainty for OOD is by leveraging the foundational idea behind DNNs, i.e., learning representations. Since DNN encodes similar data points closely in its representation space, epistemic uncertainty is operationalized in a DNN as a data point having a very different representation than any data point in the train set, i.e., a data point with representation distant from the representations of the data in the train set is unlikely to have come from the distribution underlying the training data \citep{lee2018simple}.
The downside of these methods, however, is that they require the presence of training data during prediction or involve intervention in the model's training procedure. 
This is a significant practical limitation, as using third-party models pre-trained on non-public data is increasingly the standard practice. A case in point is the Hugging Face Transformers library \citep{wolf2020transformers}, which provides %readily available 
community models but often lacks comprehensive details about their training. %data or detailed training procedures.

An obvious way to close the resource gap is to rely on OOD detection methods with minimal prerequisites. 
However, current OOD detection research has largely ignored the differing prerequisites among OOD detection methods, often leading to comparisons that treat methods with varying prerequisites equally, disregarding the question of practical applicability.
%amount of resources methods demand is often overlooked in the literature while comparing them. This can be troublesome since more resources often imply better OOD detection capability, but also hinders how widely a method can be used in practice. 
From a practical perspective, it makes sense to group OOD detection methods into the following three categories:\footnote{\citet{gomes2022igeood} employ similar terminology to refer to which parts of the model one can access (e.g., its outputs, inputs, or intermediate representations). In contrast, we use these terms to characterize the resources an OOD detection method requires.}
(1) \emph{Black-box}, for methods capable of operating on black-box models (i.e., having access only to input-output mappings) and thus suitable for models integrated into a product; 
(2) \emph{White-box}, for methods that require access to the model's weights and have knowledge about its architecture, and are thus readily applicable to third-party pre-trained models; and 
(3) \emph{Open-box}, for methods with unrestricted access to model and training resources, allowing for interventions in the training process and/or access to training data or separate OOD train or validation sets.

In this paper, we focus on the OOD detection for the Transformer architecture \citep{vaswani2017attention}, which has emerged as the predominant architecture in numerous ML domains. We introduce a novel OOD detection method that leverages the inherent differences in how Transformers process ID and OOD data. The method is white-box and has the potential for broad practical applicability. More concretely, our \textbf{B}etween \textbf{L}ayer \textbf{O}ut-\textbf{O}f-\textbf{D}istribution (\ourmethod{}) Detection method estimates the smoothness of between-layer transformations of intermediate representation, building on the insight that 
these transformations tend to be smoother for ID data than for OOD data.
We evaluate \ourmethod{} on Transformer-based pre-trained large language models applied to text classification, the most prevalent task in natural language processing (NLP), and find that it outperforms other state-of-the-art OOD detection white-box methods and even some open-box methods. We further analyze \ourmethod{} to probe into the underlying causes of the differences between how ID and OOD intermediate representations are transformed and evaluate \ourmethod{} on two other types of distribution shifts -- semantic and background shift. We provide code and data for our experiments.%for result reproducibility.
\footnote{\url{https://github.com/fjelenic/between-layer-ood}}
%\footnote{URL omitted in the anonymized version. Code and data are provided in supplementary material.} 

The contributions of this paper are as follows:
\begin{enumerate*}[label=\textbf{(\arabic*)},ref=\arabic*]
\item We propose \ourmethod{}, a novel method for OOD detection applicable even in cases when only the model's weights are available, e.g., third-party pre-trained models which are becoming \textit{de facto} standard in many fields. \ourmethod{} uses the information about the smoothness of the between-layer transformations of intermediate representations. We quantify this smoothness using the square of the Frobenius norm of the Jacobian matrix, for which we provide an unbiased estimator to alleviate computational limitations. 
\item Our experiments on Transformer-based pre-trained large language models for the task of text classification show that \ourmethod{} outperforms other state-of-the-art white-box OOD detection methods. Additionally, our results indicate that the performance advantages are more prominent when applied to complex datasets as opposed to simpler ones. We also show that \ourmethod{} is more effective in detecting background shift than semantic shift.
\item Following our main insight that between-layer representation transformations of ID data tend to be smoother from that of OOD data, we analyze the source of this difference. We find that the learning algorithm is more focused on changing the ID region of intermediate representation space, smoothing the between-layer transformations of ID data in the process. At the same time, the OOD region of the intermediate representation space is largely left unchanged, except in some scenarios, e.g., for more complex tasks, when the OOD region of the space is also changed and sharpened as a consequence.
\end{enumerate*}

\section{Related Work}

OOD detection methods are typically categorized based on their underlying mechanism, for example, into output-based, gradient-based, distance-based, density-based, and Bayesian methods \citep{yang2021generalized}. Another, and arguably more practically relevant, categorization would factor in the necessary prerequisites for these methods, distinguishing between black-box, white-box, and open-box methods as introduced earlier. In the following,  we provide a brief overview of the most prominent OOD detection methods through this lens.

\textbf{Black-box.}~
Methods with minimal prerequisites typically rely on posterior class probabilities, %operating under the assumption 
assuming that when a model %exhibits uncertainty 
is uncertain about an instance, the instance is more likely to be OOD.
A commonly used baseline quantifies the uncertainty of an instance as the negative of the model's maximum softmax probability for that instance \citep{lee2018simple}. A straightforward modification employs the entropy of softmax probabilities rather than the maximum value. \citet{liu2020energy} proposed using energy scores instead of softmax scores to overcome the issue of DNN overconfidence.

\textbf{White-box.}~
\citet{gal2016dropout} proposed using Monte-Carlo dropout to more reliably estimate the model's uncertainty, showing that dropout \citep{srivastava2014dropout} with DNNs approximates Bayesian inference. Although Monte-Carlo dropout outperforms vanilla posterior probabilities in OOD detection \citep{ovadia2019can}, it is computationally expensive as it requires multiple forward passes. Another way of leveraging the access to model's architecture is to use gradients to implicitly measure the uncertainty of the model's predictions \citep{oberdiek2018classification,huang2021importance}. Gradient methods primarily employ the gradient norm to gauge the difference between the model's posterior distribution and the ideal distribution. %for effective discrimination between ID and OOD data.
\citet{djurisic2023extremely} detect OOD data by pruning and adjusting the representations of the model, grounded in the intuition that the representations generated by contemporary DNNs tend to be excessive for their designated tasks.
%\jan{since our method is white-box, I feel we need to give a bit more details here on these methods, including their strengths and weaknesses}

\textbf{Open-box.}~
Because DNNs posterior probabilities tend to exhibit overconfidence, \citet{guo2017calibration} suggested using temperature scaling to calibrate the model's posterior probabilities, 
which entails the usage of a separate validation set. To get higher quality predictive uncertainty estimates, \citet{lakshminarayanan2017simple} train an ensemble of differently initialized models and combine their predictions. Although ensembles are robust to different distributional shifts \citep{ovadia2019can}, they impose a significant computational and memory overhead because they require training and keeping in memory of multiple models.
\citet{agarwal2022estimating} extend the gradient-based methods by leveraging the variance of the gradient of the predicted label w.r.t.~the input through different checkpoints during training.
A popular approach to OOD detection for DNNs revolves around the utilization of information related to distances in the representation space \citep{lee2018simple,van2020uncertainty,liu2020simple, hsu2020generalized,kuan2022back,sun2022out}. However, these approaches require access to the training data or changes in the standard training procedure.
%The same is often true for density-based methods \citep{postels2020quantifying}, since they need to model the training data density in the representation space.
%\citep{gomes2022igeood} proposed a method centered on investigating the information-geometric properties of the feature space, but this method also presupposes the presence of training data.
%The simplest yet effective distance-based method is calculating the distance of a query point's representation to $k$ nearest neighbors from the train test \citep{kuan2022back,sun2022out}. However, one of the first and most popular approaches of this type is to fit a class-conditional Gaussian distribution and measure uncertainty by calculating the Mahalanobis distance of the query point to the closest class-conditional distribution as proposed by \citet{lee2018simple}.
Yet another set of methods relies on exposing the model to OOD samples during training to improve the performance on OOD detection task \citep{hendrycks2018deep,thulasidasan2021effective,roy2022does}. Still, a major practical limitation of these methods is the necessity for OOD data, whose entire distribution is typically unknown in real-world scenarios.
%Apart from the practical constraint of requiring access to OOD data, whose entire distribution is typically unknown in practice, it is debatable whether these methods classify as genuine OOD detection, given the model's prior exposure to the supposedly OOD data during training. 
Several post-hoc methods also need OOD data, but for validation sets to optimize their method's hyperparameters \citep{liang2018enhancing,sun2021react,sun2022dice}.

\section{Preliminaries}

\subsection{Problem statement}

Let instance $\vx \in \mathbb{R}^d$ be a $d$-dimensional feature vector and $y \in \{0,\dots,C-1\}$ be its corresponding class in a $C$-way classification task. We train a classifier on the dataset $\mathcal{D} = \{(\vx_n,y_n)\}^N_{n=1}$ consisting of $N$ instances i.i.d.~sampled from the distribution $p(\vx,y)$. The objective of the learning algorithm is to model the conditional distribution $p(y|\vx)$  based on $\mathcal{D}$ by estimating the parameters $\vtheta$ of the distribution $p_{\vtheta}(y|\vx)$ that is as close as possible to the true conditional distribution.

The goal of an OOD detection method is to determine the uncertainty score $\mathcal{U}_{\vx} \in \mathbb{R}$ of an instance $\vx$, such that there exist $\epsilon \in \mathbb{R}$ for which both $\sP_{\vx \sim p(\vx,y)}(\mathcal{U}_{\vx} < \epsilon)$ and $\sP_{\vx \sim q(\vx,y)} (\mathcal{U}_{\vx} > \epsilon)$ are close to unity whenever $q(\vx,y)$ is a distribution sufficiently different from $p(\vx,y)$. In practice, there can never exist a scoring function that perfectly discriminates between ID examples (generated by $p(\vx,y)$) and OOD examples (generated by $q(\vx,y)$). Nevertheless, even reasonable attempts can prove valuable in real-world scenarios.

\subsection{Intuition}

Transformers work by mapping the input features onto a high-dimensional representation space through $L$ layers using the self-attention mechanism, creating a representation of the data suitable for the task at hand. The mapping is realized as a composition of several attention layers, where each layer creates an intermediate representation of the input. 
%Representations of DNNs with useful inductive biases for the task at hand tend to gradually progress from input features towards more abstract representation levels through layers, i.e., lower layers model lower-level features while upper layers model higher-level features. 
It has been show that Transformer-based models tend to gradually progress from input features towards more abstract representation levels through layers, i.e., lower layers model lower-level features, while upper layers model higher-level features. 
%For example, when convolutional neural networks (CNNs) process images, they can first represent the simplest features in an image, e.g., lines and edges, which are then used to create representations of textures and simple shapes that are then combined to represent objects \citep{lecun1998gradient}. 
For example, \citet{peters2018dissecting,tenney2019bert,jawahar2019does} showed that large Transformer-based language models create text representations that progress gradually from representations that encode morphological and syntactic information at the lower layers to representations that encode semantic meaning in the upper layers.
Likewise, Vision Transformers (ViT) \citep{dosovitskiy2021image}, which are garnering popularity in computer vision, were shown to process images in a similar fashion \citep{ghiasi2022vision}.
%\martin{I would introduce the idea of a layers in a DNN here (first sentence: mapping features onto a $d-$dim rep.space., then adapting/nonlinearly transforming this representation through $L$ layers. Leads more naturally into which $f(x)$ is learned in each layer rep (next para). Also, there are models where $d$ is not constant (Transformers for diff. max len of batches).}

We hypothesize that during the model's training, the model learns smooth transformations between layers corresponding to natural and meaningful progressions between abstractions for ID data.
We further hypothesize that these progressions will not match OOD data, hence the transformations will not be smooth for OOD data.
%\jan{this claim is our hypothesis, too, so word it as such (eg, "We further hypothesize.."). Again, I'm all against "model's perception". Consider introducing the idea of "smooth transformations between DNN layers corresponding to natural/meaningful progressions between abstractions for ID data" and "these progressions will not match OOD data, hence the transitions will not be smooth for OOD data"? And if can link this to the way training works in general (smoothing out the representation space? or Fisher?), that would be great} 
%\jan{an example would really help drive this point home. Any ideas besides examples from CV?}
Thus, if we could measure the smoothness of transformations in representations between layers, we could in principle differentiate between ID and OOD data.
We also speculate that the difference in smoothness of transformations between ID and OOD data should be emphasized in the upper layers of a Transformer. Lower layers typically represent low-level features that are more universal, whereas upper layers tend to cluster instances around task-specific features that are not shared between ID and OOD data, potentially creating a mismatch in levels of abstraction.

%\mate{Additionally (or is this saying the same as above?), could it be that, especially for a classification task, the model in the last layers tends to group/cluster? together the ID data belonging to the same class which results in the transformation from the penultimate layer to the output layer (but also the transformations between final layers in general) being locally constant for ID data, and thus having Jacobian close to the null matrix. (is that even true? how large are the computed norms?)}

%\martin{(not exact, but something like this) $\rightarrow$ Our intuition is that the observed gradual refinement of a L-layer DNN's internal representation will result in smoother transitions between subsequent layers for in-domain data, while out-of-domain data will cause sharper transformations (transitions?) as the model has not been sufficiently exposed to such instances. To this end, we explore methods of quantifying evolution of a DNN's internal representation between layers for the purpose of discriminating between ID and OOD data.}

\subsection{Our method}

Assume an $L$-layered deep neural network $\vf : \R^{d_0} \rightarrow [0, 1]^C$ was trained to predict the probabilities of $C$ classes for a $d_0$-dimensional input $\vx$. Let $\vf$ be a composition of $L$ intermediate functions, $\vf_L \circ \dots \circ \vf_l \circ \dots \circ \vf_1$, where $\vf_l : \R^{d_{l-1}} \rightarrow \R^{d_l}$,  $l=1,\dots,L-1$, correspond to intermediate network layers, while $\vf_L$ corresponds to the last layer, mapping to a vector of logits to which softmax function is applied to obtain the conditional class probabilities. We denote the intermediate representation of $\vx$ in layer $l$ as $\vh_l$, defined as $\vh_l = (\vf_l\circ\dots\circ\vf_1)(\vx)$.

We now need to quantify how smoothly an intermediate representation is transformed from layer $l$ to layer $l+1$. To this end, we first need to define what we consider a smooth transformation. We say a representation $\vh_l$ is transformed smoothly if there is not a large difference in how it is mapped from layer $l$ onto layer $l+1$ compared to how its infinitesimally close neighborhood is mapped. %\mate{(premjestio, možda ima smisla odmah tu povezati s jakobijanom?): Since the Jacobian matrix $\frac{\partial \vf_{l+1}}{\partial \vh_l} = \mJ_{l}$ of the transformation, evaluated at $\vh_l$, holds the information about how infinitesimally changing the input $\vh_l$ (in any direction) affects the output, we use $\|\mJ(\vh_l)\|^2_F=\sum_{j}\|\mJ(\vh_l)\ve_j\|^2_2$, i.e. the sum of squared magnitudes of the displacement vectors --- one corresponding to a unit change in each of the coordinate directions -- to quantify the overall local smoothness/sharpness around the point $\vh_l$ at which the Jacobian matrix was evaluated.}

Let $\phi_l(\vx)$ be the degree of smoothness of the transformation between representation $\vh_l$ and representation $\vh_{l+1}$ for input $\vx$. To calculate $\phi_l(\vx)$, we compute the Jacobian matrix $\frac{\partial \vf_{l+1}}{\partial \vh_l} = \mJ_{l}: \R^{d_l} \rightarrow \R^{d_{l+1}\times d_l}$, and take the square of its Frobenius norm:
\begin{equation}
    \phi_l(\vx) = \| \mJ_{l}(\vh_l) \|_F^2 = \sum_{i=1}^{d_{l+1}} \sum_{j=1}^{d_{l}} \left(\frac{\partial (f_{l+1})_i} {\partial (h_{l})_j}\right)^2
\end{equation}
In the most popular ML libraries, gradients of a function are computed through automatic differentiation (AD), which comprises both forward mode and backward mode. Forward mode AD computes the values of the function and a Jacobian-vector product. Computing the full Jacobian matrix $\mJ(\vx)$ with AD is computationally expensive as it requires $d$ forward evaluations of $\mJ(\vx) \ve^{(i)}$, $i=1, \dots, d$, where $\ve^{(i)}$ are standard basis vectors, computing the Jacobian matrix one column at a time. In the case of modern DNNs with high-dimensional hidden layers, computing full Jacobians could render our method unfeasible. To reduce computational complexity, we derive an unbiased estimator of $\phi_l(\vx)$ by leveraging Jacobian-vector product computation through forward mode AD.

\begin{corollary}
   Let $\mJ(\vx) \in \R^{m \times n}$ be a Jacobian matrix, and let $\rvv \in \R^n$ and $\rvw \in \R^m$ be random vectors whose elements are independent random variables with zero mean and unit variance. Then, $\E[(\rvw^{\intercal}\mJ(\vx)\rvv)^2] = \| \mJ(\vx) \|_F^2$.
   \label{thr:estimate}
\end{corollary}
We prove \Cref{thr:estimate} in the \Cref{sec:appendix:proof} by providing a proof for more general \Cref{thr:general}.
%\mate{(brisati): As for the intuition behind the corollary, we begin by establishing the interpretation of $\|\mJ(\vx)\|^2_F$. Since $\mJ(\vx)$ holds the information about how much each dimension of the output space is changed with the change in each dimension of the input space, we use $\|\mJ(\vx)\|^2_F$ to quantify the overall local smoothness/sharpness around the point at which the Jacobian matrix was evaluated.}
As for the intuition behind the corollary, the Jacobian-vector product $\mJ(\vx)\rvv$ gives us an appropriately scaled gradient with respect to the change of the input in the direction of vector $\rvv$. Further multiplying the Jacobian-vector product $\mJ(\vx)\rvv$ by the random vector $\rvw$ from the left projects the calculated directional gradient $\mJ(\vx)\rvv$ on the vector $\rvw$, i.e., it quantifies the extent to which the output changes in the direction of $\rvw$ when the input changes in the direction of $\rvv$. Squaring the vector-Jacobian-vector product then gives an estimate of the sum of squared entries of the Jacobian, i.e., the square of its Frobenius norm. Squaring also handles negative values (in cases when the angle between the directional gradient $\mJ(\vx)\rvv$ and the vector $\rvw$ is obtuse), since we are interested in the overall smoothness as defined by Frobenius norm rather than the direction of the specific gradient.\footnote{Our notion of smoothness extends from Lipschitz continuity, where the spectral norm of the Jacobian acts as a lower bound for the Lipschitz constant \citep{rosca2020case}. Since all matrix norms are equivalent, we use the Frobenius norm, which can be efficiently computed, rather than the spectral norm to capture smoothness.}

 To calculate the unbiased estimate $\hat{\phi_l}(\vx)$ of $\phi_l(\vx)$, we use a sample of $M$ pairs of random vectors $\rvv_l \sim \mathcal{N} (\vzero_n , \mI_n)$ and $\rvw_l \sim \mathcal{N} (\vzero_m , \mI_m)$, and define $\hat{\phi_l}(\vx)$ as:
\begin{align}
    \hat{\phi_l}(\vx) &= \frac{1}{M} \sum_{i=1}^M \left(\rvw_{l,i}^{\intercal}\mJ_l(\vh_l)\rvv_{l,i} \right)^2
    %\E[\phi_l'(\vx)] &= \phi_l(\vx)
\end{align}
\ourmethod{} uses $\hat{\phi}_l(\vx)$ as the uncertainty score of an instance $\vx$.
In our experiments, we consider two variations of \ourmethod{}: (1) the average of scores for all layers \ourmethodmean{} $ = \frac{1}{L-1} \sum_{l=1}^{L-1}\hat{\phi}_l(\vx)$, and (2) the score for the projection of \ourmethodL{} $= \hat{\phi}_{L-1}(\vx)$. 
We use the two variants to assess the impact of layer choice, as we hypothesize that \ourmethod{} will perform better on upper layers, given that lower layers capture low-level, general features.

%\martin{If I understood everything, v and w act as sample input/outputs to which the Jacobian is "applied". Would it not make more sense to sample from a Gaussian fit on training data samples? Or use LayerNorm scale/shift values? Or perhaps the normal Gaussian actually is a good assumption due to normalization}

\section{Experiments}
\label{sec:experiments}

%In this section, we present our main experimental results, with additional findings in~\Cref{sec:appendix:extra-results}.

\subsection{Experimental setup}

We evaluate \ourmethod{} on several text classification datasets using two transformer-based \citep{vaswani2017attention} large pre-trained language models, RoBERTa \citep{liu2019roberta} and ELECTRA \citep{clark2020electra}, known for their state-of-the-art performance across a wide range of NLP tasks. We calculate the \ourmethod{} score using samples of size $M=50$ to estimate $\hat{\phi}_l(\vx)$ of \texttt{[CLS]} token's representations between layers.
We use eight text classification datasets for ID data: SST-2 \cite[\textbf{\textsc{sst};}][]{socher2013recursive}, Subjectivity \cite[\textbf{\textsc{subj};}][]{Pang+Lee:04a}, AG-News \cite[\textbf{\textsc{agn};}][]{zhang2015character}, and TREC \cite[\textbf{\textsc{trec};}][]{li2002learning}, BigPatent \cite[\textbf{\textsc{bp};}][]{sharma2019bigpatent}, AmazonReviews \cite[\textbf{\textsc{ar};}][]{mcauley2015image}, MovieGenre \cite[\textbf{\textsc{mg};}][]{maas2011learning}, 20NewsGroups \cite[\textbf{\textsc{ng};}][]{lang1995newsweeder}. We use One Billion Word Benchmark (\textbf{\textsc{obw}}) \citep{chelba2013one} for OOD data, similarly to \citet{ovadia2019can}, because of the diversity of the corpus. We subsample OOD datasets to be of the same size as their ID test set counterparts. \Cref{sec:appendix:exp-design} provides more details about the models, datasets, and training procedures.

%\jan{explicitly introduce the minimal/additional(extra?) prerequisites dichotomy, as this seems to be central to our story. Also, introduce the simple/complex dichotomy of datasets (refer to the appendix for empirical evidence), as we will be needing this later.} \fran{Is it still needed to introduce the prerequisite dichotomy since we elaborated on that in the intro. Also, does it makes sense to introduce the simple/complex dichotomy here? Here it seems a bit arbitrary since it is not yet motivated by anything.}

%\martin{Subsection: baselines? Maybe introduce them in the previous section? Perhaps more than one sentence for each method. + Bullets/enumerations take space, rather use inline enumeration.}

We compare \ourmethod{} to several state-of-the-art black-box and white-box OOD detection methods:
\begin{enumerate*}[label=(\arabic*),ref=\arabic*]
    \item \textbf{Maximum softmax probability} (\textbf{MSP}) -- the negative posterior class probability of the most probable class, $- \max_{c}p(y=c|\vx)$, often considered a baseline OOD detection method\citep{hendrycksbaseline};
    \item \textbf{Entropy} (\textbf{ENT}) -- the entropy of the posterior class distribution, $\sH[Y|\vx,\vw]$;
    \item \textbf{Energy} (\textbf{EGY}) -- a density-based method that overcomes the overconfidence issue by calculating energy scores from logits $-\log \sum_{i=0}^{C-1}e^{f_{L}(\vx)_i}$ instead of softmax scores \citep{liu2020energy};
    \item \textbf{Monte-Carlo dropout} (\textbf{MC}) -- the entropy of predictive distribution obtained using Monte-Carlo dropout \citep{gal2016dropout}. We use $M=30$ stochastic forward passes to estimate uncertainty;
    \item \textbf{Gradient norm} (\textbf{GRAD}) -- the L2-norm of the penultimate layer's gradient of the loss function with most likely class considered as a true class \citep{oberdiek2018classification}.
    \item \textbf{Activation shaping} (\textbf{ASH}) -- removing 90\% of the smallest activations and adjusting the rest using ASH-S method in the penultimate layer  \citep{djurisic2023extremely}.
\end{enumerate*}

Additionally, we compare \ourmethod{} to three standard open-box OOD detection methods. 
Given that these methods entail considerably more prerequisites compared to \ourmethod{} and other \whitebox{} methods, this comparison is intended solely as a reference point:
\begin{enumerate*}[label=(\arabic*),ref=\arabic*]
    \item \textbf{Rectified Activations} (\textbf{ReAct}) -- setting the values of the activations in the penultimate layer to be at most the 90th percentile of the activations of the training data \citep{sun2021react}.
    \item \textbf{Ensemble} (\textbf{ENSM}) -- an ensemble of $M=5$ models of the same type, e.g., an ensemble of five RoBERTa or ensemble of five ELECTRA models, \citep{lakshminarayanan2017simple};
    \item \textbf{Temperature scaling} (\textbf{TEMP}) -- introduces a temperature parameter $T$ into the softmax function such that it minimizes the negative log-likelihood on the ID validation set \citep{guo2017calibration};
    %$\frac{\exp(\vz_i/T)}{\sum_{j=0}^{C-1} \exp(\vz_j/T)}$
    \item \textbf{Mahalanobis distance} (\textbf{MD}) -- Mahalanobis distance of a query instance in the representation space with respect to the closest class-conditional Gaussian distribution \citep{lee2018simple}.
\end{enumerate*}

\subsection{OOD detection performance}

As the performance measure for OOD detection, we follow the standard practice and use the area under the receiver operating characteristic curve (AUROC) metric (in \Cref{sec:appendix:extra-results}, we report the results using two other commonly used metrics, AUPR-IN and FPR@95TPR; these gave qualitatively identical results as AUROC). The OOD detection task is essentially a binary classification task, with AUC corresponding to the probability that a randomly chosen OOD instance will have a higher uncertainty score than a randomly chosen ID instance \citep{fawcett2006introduction}.
%
%To quantify the success of the methods for OOD detection, we use the area under the receiver operating curve (AUROC) since, in the OOD detection setting, it can be thought of as the probability that a random OOD example has a higher uncertainty score than random ID example \citep{fawcett2006introduction}. 
The AUROC for random value assignment is 50\%, while a perfect method achieves 100\%. We run each experiment five times with different random seeds  and report the mean AUROC.

\begin{table}[]
\caption{The performance of OOD detection methods measured by AUROC (\%). The best-performing \whitebox{} method is in \textbf{bold}. Open-box methods that outperform the best-performing \whitebox{} method are in \textbf{bold}. Higher is better. We test the performance of \ourmethodmean{} and \ourmethodL{} against the MSP baseline using the one-sided Man-Whitney U test; significant improvements ($p<.05$) are indicated with asterisks ($^*$).
}
\small
\centering
\scalebox{0.76}{
\begin{tabular}{ll|cccccccc|cccc} 
\toprule
\multirow{2}{*}{Model} & \multirow{2}{*}{Dataset} & \multicolumn{8}{c|}{White-box/Black-box} & \multicolumn{4}{c}{Open-box}\\
& & \ourmethodmean{} & \ourmethodL{} & MSP & ENT & EGY & MC & GRAD & ASH & ReAct & ENSM & TEMP & MD\\
\midrule
\multirow{8}{*}{\rotatebox[origin=c]{90}{RoBERTa}} & \textsc{sst} & 50.56 & \textbf{72.83} & 71.69 & 71.69 & 71.61 & 68.28 & 71.76 & 67.22 & 69.55 & 69.03 & 71.64 & \textbf{85.36}\\ 
& \textsc{subj} & 52.02 & 74.66 & 74.55 & 74.55 & 75.79 & 74.21 & 74.93 & \textbf{79.27} & 73.33 & 76.68 & 74.41 & \textbf{93.47}\\
& \textsc{agn} & 77.46 & 61.95 & 73.57 & 73.80 & 76.36 & \textbf{77.55} & 73.58 & 72.54 & 77.10 & \textbf{80.35} & 75.38 & \textbf{82.63}\\
& \textsc{trec} &  69.63 & 95.30 & 96.20 & \textbf{96.40} & 96.28 & 95.68 & 96.14 & 90.36 & 96.05 & \textbf{96.87} & \textbf{96.74} & \textbf{96.74}\\ 
& \textsc{bp} & 87.20\nospacetext{$^*$} & \textbf{89.53}\nospacetext{$^*$} & 70.15 & 72.82 & 85.84 & 74.29 & 73.11 & 82.18 & 86.19 & 79.39 & 86.01 & \textbf{97.35}\\
& \textsc{ar} & 91.41\nospacetext{$^*$} & \textbf{93.20}\nospacetext{$^*$} & 89.06 & 89.96 & 92.39 & 90.59 & 88.65 & 91.42 & 92.65 & 92.44 & 92.25 & \textbf{98.35}\\
& \textsc{mg} & \textbf{88.15}\nospacetext{$^*$} & 85.23\nospacetext{$^*$} & 75.02 & 76.60 & 86.45 & 79.98 & 74.28 & 81.62 & 87.30 & 76.98 & 84.30 & \textbf{95.12}\\
& \textsc{ng} & \textbf{83.53}\nospacetext{$^*$} & 72.02 & 77.49 & 78.76 & 82.65 & 79.32 & 76.93 & 77.73 & 83.17 & 80.77 & 82.87 & \textbf{90.68}\\
\midrule
\multirow{8}{*}{\rotatebox[origin=c]{90}{ELECTRA}} & \textsc{sst} & 74.36 & \textbf{78.11}\nospacetext{$^*$} & 73.84 & 73.84  & 71.97 & 70.81 & 73.82 & 67.92 & 71.18 & 73.81 & 73.58 & \textbf{78.85}\\
& \textsc{subj} & 74.10 & 77.41 & \textbf{78.17} & \textbf{78.17} & 70.46 & 77.71 & 78.11 & 75.11 & 68.33 & \textbf{79.23} & \textbf{78.20} & \textbf{81.59}\\
& \textsc{agn} & 65.67 & \textbf{80.28} & 76.80 & 77.01 & 79.75 & 79.55 & 76.57 & 77.96 & 79.46 & 79.50 & 78.31 & \textbf{86.10}\\
& \textsc{trec} & 97.48 & \textbf{98.90}\nospacetext{$^*$} & 97.26 & 97.56 & 97.48 & 96.21 & 97.07 & 90.18 & 97.50 & 97.55 & 98.20 & 97.54 \\
& \textsc{bp} & 86.06\nospacetext{$^*$} & \textbf{96.72}\nospacetext{$^*$} & 78.56 & 81.75 & 84.63 & 83.04 & 76.77 & 79.81 & 85.26 & 84.20 & 84.69 & \textbf{98.28}\\
& \textsc{ar} & 84.58 & \textbf{91.66}\nospacetext{$^*$} & 87.74 & 88.44 & 90.64 & 88.53 & 87.52 & 83.96 & 91.01 & \textbf{91.98} & 90.35 & \textbf{95.47}\\
& \textsc{mg} & 80.52 & \textbf{90.63}\nospacetext{$^*$} & 73.83 & 74.78 & 80.41 & 76.67 & 73.35 & 71.84 & 81.22 & 76.86 & 78.47 & \textbf{92.96}\\
& \textsc{ng} & 77.61 & \textbf{82.47}\nospacetext{$^*$} & 76.45 & 77.73 & 80.83 & 79.11 & 75.97 & 74.50 & 80.95 & 79.93 & 80.75 & \textbf{89.13}\\
\bottomrule
\end{tabular}
}
\label{tab:agnostic-ood}
\end{table}

OOD detection performance is shown in~\Cref{tab:agnostic-ood}. The first observation is that \ourmethod{} outperforms other \whitebox{} methods. Secondly, \ourmethodL{} outperforms other \whitebox{} methods more often than \ourmethodmean{}, thus in the rest of the experiments we focus on \ourmethodL{}. Lastly, 
while BLOOD demonstrates superior performance on most datasets, the improvements are more consistently observed when applied with ELECTRA compared to RoBERTa. Interestingly, the datasets where BLOOD with RoBERTa outperforms  other \whitebox{} methods (\textsc{sst}, \textsc{bp}, \textsc{ar}, \textsc{mg}, and \textsc{ng}) appear to be more complex, as indicated by the minimum description length  \citep{perez2021rissanen} (cf.~\Cref{sec:appendix:exp-design}).
%Still, the performance of \ourmethod{} on RoBERTa picks up \alert{on more complex datasets, \textsc{bp}, \textsc{ar}, \textsc{mg}, and \textsc{ng}, as characterized by the minimum description length  \citep{perez2021rissanen} (cf.~\Cref{sec:appendix:exp-design}).}
We offer explanations for these observations in  sections \ref{sec:diff-source} and \ref{sec:complexity}. 

Compared to open-box methods, \ourmethod{} is outperformed by MD in all setups except when using ELECTRA on the \textsc{trec} dataset. However, \ourmethod{} remains competitive with ENSM and TEMP.
Unlike the findings in \citep{ovadia2019can}, the dominance of ENSM is reduced. This is likely because we employ a pre-trained language model ensemble, while they use entirely randomly initialized models. In our ensemble, the model parameters exhibit minimal variation since all models are pre-trained. Variability between models arises solely from the random initialization of the classification head and the stochastic nature of the training process. The high performance of MD on transformer-based language models is aligns with prior research \citep{podolskiy2021revisiting}.

\subsection{Source of the differences in transformations of ID and OOD data}
\label{sec:diff-source}

\begin{figure}[t!]
\small
\centering
\begin{subfigure}{.49\linewidth}
  \centering
  \includegraphics[width=\linewidth]{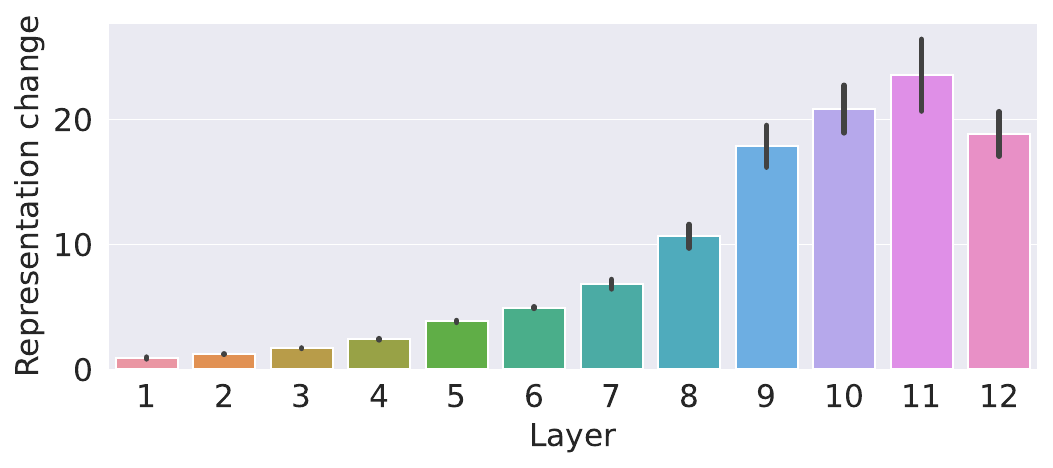}
  \caption{RoBERTa's representation change}
\end{subfigure}
\begin{subfigure}{.49\linewidth}
  \centering
  \includegraphics[width=\linewidth]{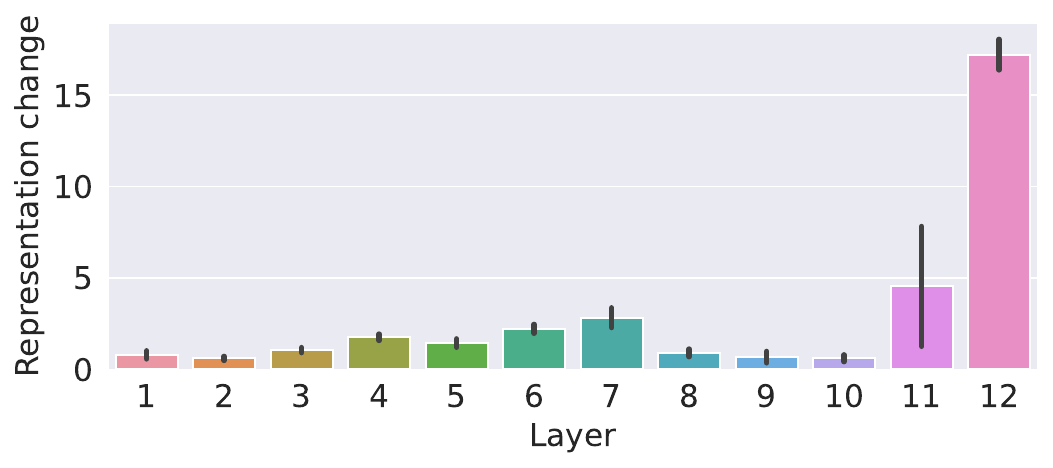}
  \caption{ELECTRA's representation change}
\end{subfigure}
\begin{subfigure}{.49\linewidth}
  \centering
  \includegraphics[width=\linewidth]{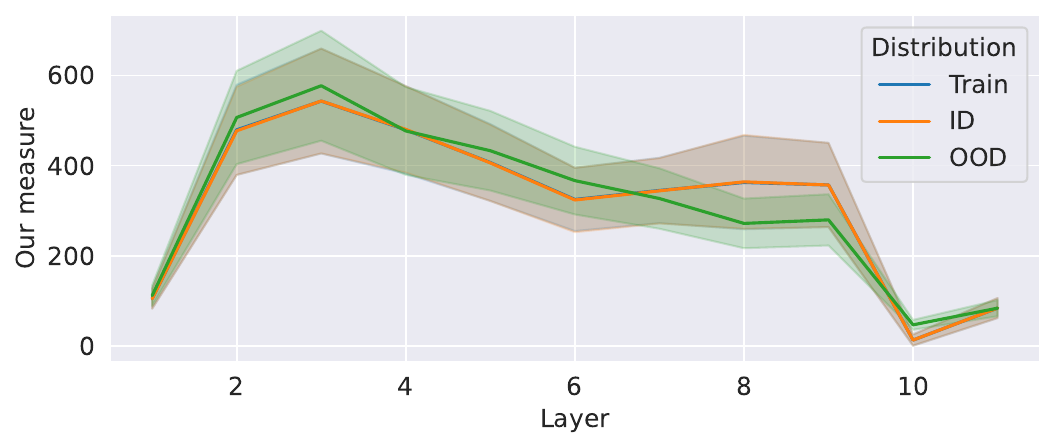}
  \caption{Pre-trained RoBERTa's \ourmethod{} score by layer}
\end{subfigure}
\begin{subfigure}{.49\linewidth}
  \centering
  \includegraphics[width=\linewidth]{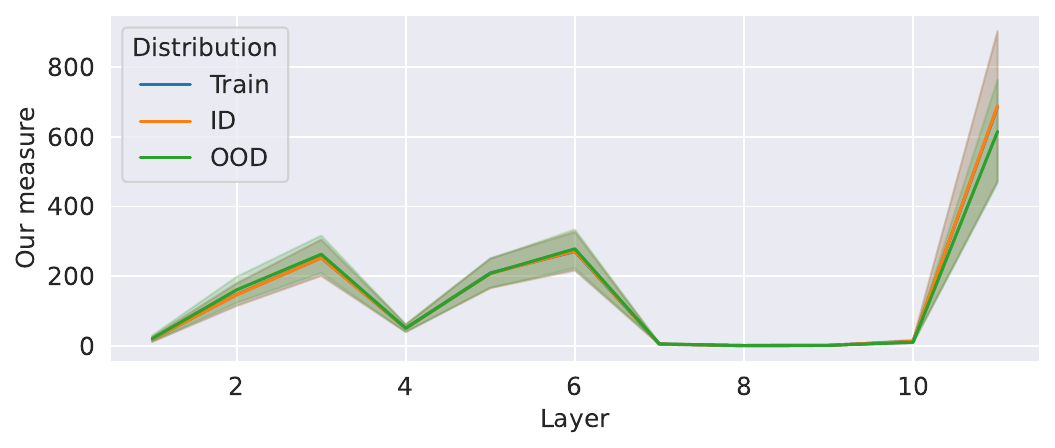}
  \caption{Pre-trained ELECTRA \ourmethod{} score by layer}
\end{subfigure}
\begin{subfigure}{.49\linewidth}
  \centering
  \includegraphics[width=\linewidth]{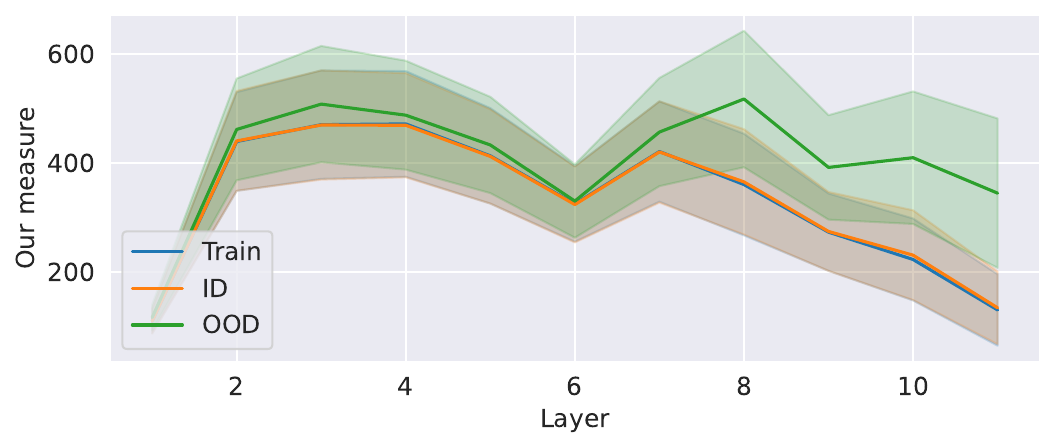}
  \caption{Fine-tuned RoBERTa's \ourmethod{} score by layer}
\end{subfigure}
\begin{subfigure}{.49\linewidth}
  \centering
  \includegraphics[width=\linewidth]{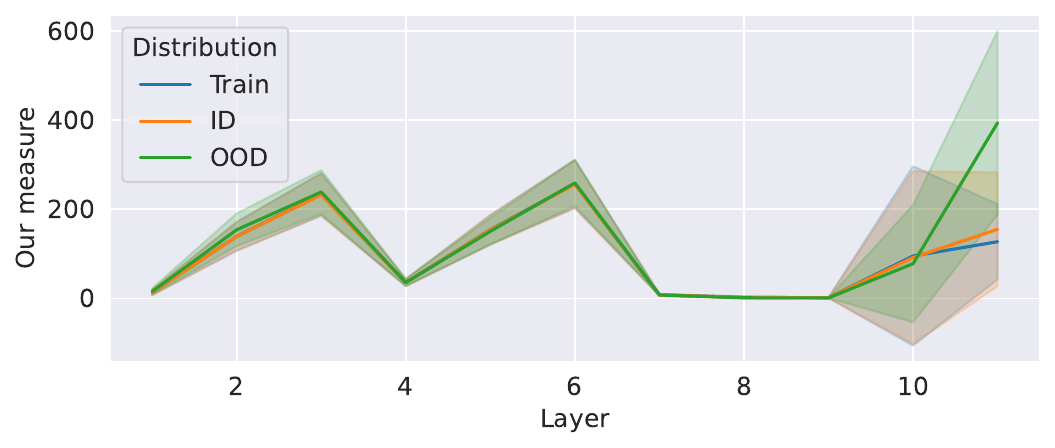}
  \caption{Fine-tuned ELECTRA \ourmethod{} score by layer}
\end{subfigure}
\caption{The impact of change of each layer on \ourmethod{} score across layers. Top row: Change in intermediate representations of training instances by layer for (a) RoBERTa and (b) ELECTRA. The scores are averaged across instances for the \textsc{ar} dataset. The black error bars denote the standard deviation. Middle row: \ourmethod{} score by layer of models for \textsc{ar} before fine-tuning. Bottom row: \ourmethod{} score by layer of models for \textsc{ar} after fine-tuning. 
\vspace*{-12px}}
\label{fig:repr-change}
\end{figure}

Understanding which layers of the model are impacted by the model's training could shed some light on the behavior of our method. 
To find out how much each layer has learned, we examine the changes in intermediate representations of instances after training. 
For simplicity, we use the Euclidean distances $\| \vr_{\text{init}} - \vr_{\text{FT}} \|_2$ between representations of the initialized model ($\vr_{\text{init}}$) and the representations after fine-tuning the model ($\vr_{\text{FT}}$). 
We calculate this distance for all instances in the training set at each of the model's layers and then compute the average for each layer.

\Cref{fig:repr-change} illustrates the extent of representation changes in training data alongside \ourmethod{} scores before and after fine-tuning at each intermediate layer. The representations of the upper layers change significantly more than the representations of the lower layers.
This is expected since transformer-based language models learn morphological- and syntactic-based features in the lower layers, which are similar between tasks and can be mostly reused from the pre-training. In contrast, higher layers learn more task-specific features such as context and coreferences \citep{peters2018dissecting,tenney2019bert,jawahar2019does}. 
%This effect explains why \ourmethodL{} outperforms \ourmethodmean{} on the OOD detection task.
%If our hypothesis about the model learning smooth transitions for ID data holds true, then the degree of change in the intermediate representation within a layer directly correlates with the smoothness of that layer's transition for ID data.
Our hypothesis posits that the smooth transformations of ID data are a by-product of the learning algorithm learning the natural progression between abstractions. Consequently, layers more impacted by training will exhibit smoother transformations, which explains why \ourmethodL{} outperforms \ourmethodmean{} on the OOD detection task. 
This effect becomes apparent when comparing the representation change (upper row of \Cref{fig:repr-change}) with the \ourmethod{} score (lower two rows of \Cref{fig:repr-change}) across layers, with a more significant difference in transition smoothness between ID and OOD data observed in layers where representations have undergone more substantial changes overall.
The effect is particularly emphasized in ELECTRA, where the last layer undergoes the most significant change, resulting in \ourmethodL{} performing exceptionally well due to the radical smoothing of the final transformation.

We also anticipate that the representations of ID data will undergo more significant changes after fine-tuning than those of OOD data, given the model's focus on the ID region of the representation space during training.
This effect would cause a difference in smoothness because the ID region of the space would be smoothed out while the OOD region of the space would keep its original sharpness. Same as above, we calculate the change in representations using Euclidean distance of representations before and after fine-tuning.
We then quantify the difference between changes in representations of ID and OOD data using the common language effect size (CLES) \citep{mcgraw1992common}, corresponding to the probability that representations of ID data exhibited greater changes after training than representations of OOD data.\footnote{The CLES statistics quantifies the effect size of the difference between two samples. It is equivalent to AUC of the corresponding univariate binary classifier, representing the probability that a randomly selected score from the first sample will exceed a randomly selected score from the second sample.} 
%However, now we use AUROC to measure the difference between changes in representations of ID and OOD data. However, now we assume the ID data should have larger changes, i.e., AUROC above 50\% signals that ID representations changed more than the OOD representations. 
We measure this difference for the model's last layer and the mean difference across all layers.

\Cref{tab:repr-diff} shows the effect size quantified using CLES for the changes in representations between ID and OOD data. In most setups, CLES is far above 50\%, which means that representations of ID data underwent more significant changes than those of OOD data. The results imply that the learning algorithm's focus during training is on the ID region of the representation space. In contrast, the rest of the representation space is largely unaltered. Moreover, the difference in transformation smoothness between layers, observed between ID and OOD data, can be attributed to the inherently non-smooth transformations of the initialized models.  These non-smooth transformations gradually become smoother within the ID region.
However, more complex datasets (\textsc{bp}, \textsc{ar}, \textsc{mg}, and \textsc{ng}) in conjunction with the RoBERTa model contradict our initial expectation. In these cases, CLES approaches or even drops below 50\%. 
This indicates that the ID region of the representation space undergoes similar or even lesser changes compared to the rest of the representation space.

\begin{wraptable}{R}{.47\linewidth}
\caption{Effect size of the changes in representations between ID and OOD data. We calculate CLES (\%) averaged across layers (Mean) and for the last layer (Last), showing averages over five random seeds with standard deviation.}
\small
\centering
\scalebox{0.77}{
\begin{tabular}{llccc}
\toprule
Model & Dataset & Mean & Last \\
\midrule
\multirow{8}{*}{RoBERTa} & \textsc{sst} & 66.86 $\pm$ \phantom{0}5.90 & 63.91 $\pm$ \phantom{0}5.64 \\ 
& \textsc{subj} & 78.77 $\pm$ \phantom{0}9.61 & 68.08 $\pm$ 10.53 \\
& \textsc{agn} & 73.28 $\pm$ \phantom{0}3.59 & 60.18 $\pm$ \phantom{0}4.38 \\
& \textsc{trec} & 90.63 $\pm$ \phantom{0}7.19 & 74.02 $\pm$ 21.03 \\ 
& \textsc{bp} & 55.98 $\pm$ 29.52 & 39.65 $\pm$ 16.38\\
& \textsc{ar} & 52.52 $\pm$ 15.53 & 33.83 $\pm$ \phantom{0}8.45 \\
& \textsc{mg} & 34.40 $\pm$ \phantom{0}9.56 & 46.23 $\pm$ 11.90 \\
& \textsc{ng} & 40.93 $\pm$ \phantom{0}8.51 & 49.56 $\pm$ \phantom{0}9.14 \\
\midrule
\multirow{8}{*}{ELECTRA} & \textsc{sst} & 82.09 $\pm$ \phantom{0}1.31 & 78.67 $\pm$ \phantom{0}0.97 \\
& \textsc{subj} & 77.43 $\pm$ 13.52 & 75.61 $\pm$ 14.63 \\
& \textsc{agn} & 81.28 $\pm$ \phantom{0}3.62 & 80.82 $\pm$ \phantom{0}4.23 \\
& \textsc{trec} & 99.86 $\pm$ \phantom{0}0.05 & 99.10 $\pm$ \phantom{0}0.54\\
& \textsc{bp} & 93.35 $\pm$ \phantom{0}2.00 & 82.80 $\pm$ \phantom{0}3.19 \\
& \textsc{ar} & 82.21 $\pm$ \phantom{0}9.59 & 81.95 $\pm$ \phantom{0}7.70 \\
& \textsc{mg} & 83.88 $\pm$ \phantom{0}6.01 & 83.83 $\pm$ \phantom{0}7.70 \\
& \textsc{ng} & 79.08 $\pm$ \phantom{0}8.84 & 80.16 $\pm$ \phantom{0}4.60 \\
\bottomrule
\end{tabular}
}
\label{tab:repr-diff}
\end{wraptable}

Our interpretation of this phenomenon is that the algorithm faces greater difficulty in fitting the data, necessitating more substantial adjustments to the model. These significant alterations not only result in smoothing out transitions for ID data but, as a consequence, also make transformations in the rest of the space less smooth.
This would explain the improved performance of \ourmethod{} in conjunction with RoBERTa on these datasets, as the difference in transformation smoothness is attributed not only to the smoothing of the ID region of the space but also to the reduction in smoothness of the remaining space.
This sharpening effect in the region populated by OOD data is evident when comparing sub-figures (c) and (e) in \Cref{fig:repr-change}. 

\subsection{The effect of dataset complexity}
\label{sec:complexity}

In the previous subsection, we demonstrated that \ourmethod{} performs better on more complex datasets compared to simpler ones.\footnote{We support this finding by calculating the Pearson correlation coefficient between MDL and difference in AUROC of \ourmethodmean{} (to capture the influence on all layers in the model) and the baseline method (MSP) for each dataset. We found a significant ($p<.05$) correlation of  0.79 for RoBERTa and 0.73 for ELECTRA.} To investigate this phenomenon further, we re-evaluate the performance of OOD detection methods on simplified versions of the more complex datasets. Specifically, we use the binary classification datasets \textsc{bp2}, \textsc{ar2}, and \textsc{mg2}, which are derived from \textsc{bp}, \textsc{ar}, and \textsc{mg} datasets, respectively, by retaining only two classes (cf.~\Cref{sec:appendix:exp-design} for additional details).

\begin{table}[]
\caption{The performance of OOD detection methods for the simplified datasets measured by AUROC (\%). The best-performing \whitebox{} method is in \textbf{bold}. Open-box methods that outperform all \whitebox{} methods are in \textbf{bold}. Higher is better. The right side of the table shows a comparison of changes in representations between ID and OOD data using CLES (\%).}
\small
\centering
\scalebox{0.77}{
\addtolength{\tabcolsep}{-0.1em}
\begin{tabular}{@{}ll|ccccccc|cccc|cc@{}} 
\toprule
\multirow{2}{*}{Model} & \multirow{2}{*}{Dataset} & \multicolumn{7}{c|}{White-box/Black box} & \multicolumn{4}{c|}{Open-box} & \multicolumn{2}{c}{CLES}\\
 &  & \ourmethodL{} & MSP & ENT & EGY & MC & GRAD & ASH & ReAct & ENSM & TEMP & MD & Mean & Last\\
\midrule
\multirow{3}{*}{\rotatebox[origin=c]{0}{RoBERTa}} & \textsc{bp2} & 79.66 & 89.74 & 89.74 & 88.23 & 88.92 & \textbf{89.84} & 82.66 & 87.60 & 87.59 & \textbf{89.92} & \textbf{97.66} & 94.57 & 84.27 \\
& \textsc{ar2} & 88.20 & 93.33 & 93.33 & \textbf{94.27} & 93.30 & 93.58 & 92.63 & 93.31 & \textbf{94.55} & 93.34 & \textbf{99.02} & 91.84 & 80.47 \\
& \textsc{mg2}  & 84.78 & 78.13 & 78.13 & \textbf{85.44} & 82.62 & 78.28 & 74.05 & \textbf{85.95} & 83.95 & 78.23 & \textbf{97.48} & 86.80 & 70.25\\
\midrule
\multirow{3}{*}{\rotatebox[origin=c]{0}{ELECTRA}} & \textsc{bp2} & 71.71 & \textbf{93.23} & \textbf{93.23} & 92.51 & 92.61 & 93.20 & 86.93 & 92.24 & 91.25 & \textbf{93.34} & \textbf{98.75} & 97.28 & 94.87 \\
& \textsc{ar2}  & 90.67 & \textbf{96.16} & \textbf{96.16} & 93.80 & 95.47 & 96.14  & 91.95 & 93.40 & 95.20 & \textbf{96.20} & 93.22 & 97.07 & 96.22\\
& \textsc{mg2}  & \textbf{91.41} & 88.02 & 88.02 & 85.08 & 88.55 & 88.10 & 76.48 & 85.11 & 84.12 & 87.95 & \textbf{98.28} & 88.28 & 87.10\\
\bottomrule
\end{tabular}
}
\label{tab:agnostic-ood-augmented}
\end{table}

\Cref{tab:agnostic-ood-augmented} shows AUROC for the OOD detection task on simplified datasets, as well as the CLES of representation changes. We observe a decrease in AUROC for \ourmethod{} in comparison to the AUROC on the original datasets, while the AUROC of other \whitebox{} methods shows an increase. The drop in AUROC for \ourmethod{} can be explained by examining the CLES of representation changes, which exhibits a notable increase compared to the original datasets in the case of RoBERTa, and even a slight increase for ELECTRA. The rise in CLES of the change in representations suggests that the models managed to learn the task without the need to sharpen the transformations of the OOD data, thereby reducing the ability of \ourmethod{} to detect OOD instances. 

We suspect that the increase in AUROC for the other \whitebox{} methods may be attributed to the same factor that led to the AUROC decrease in \ourmethod{} -- namely, the task's simplicity. However, this cause manifests differently. The simplified datasets, having fewer ambiguous instances in their test sets due to the reduced number of classes, allow the other (probabilistic) methods to more accurately attribute the estimated uncertainty to the OOD data. See \Cref{sec:appendix:carto} for a more detailed explanation and visualization using dataset cartography \citep{swayamdipta2020dataset}.

\subsection{Types of distribution shift}

Another important aspect to consider for OOD detection is the type of distribution shift. Up to this point, we have only considered OOD data coming from a distribution entirely different than that of the ID data, which is referred to as Far-OOD by \citet{baran2023classical}. We next examine the performance of OOD detection methods on Near-OOD data, which arises from either a semantic or a background shift. For the semantic shift, in line with  \citet{ovadia2019can}, we designate the even-numbered classes of \textsc{ng} dataset as ID and the odd-numbered classes as Near-OOD data. For the background shift,  following  \citet{baran2023classical}, we use the \textsc{sst} dataset as ID and the Yelp Review sentiment classification dataset \citep{zhang2015character} as Near-OOD data. 

\begin{table}[]
\caption{The performance of OOD detection methods on the task of Near-OOD detection measured by AUROC (\%). The best-performing \whitebox{} method is in \textbf{bold}. Open-box methods that outperform all \whitebox{} methods are in \textbf{bold}. Higher is better.}
\small
\centering
\scalebox{0.77}{
\begin{tabular}{cc|ccccccc|cccc} 
\toprule
\multirow{2}{*}{Model} & \multirow{2}{*}{Shift} & \multicolumn{7}{c|}{White-box/Black-box} & \multicolumn{4}{c}{Open-box}\\
 &  & \ourmethodL{} & MSP & ENT & EGY & MC & GRAD & ASH & ReAct & ENSM & TEMP & MD \\
\midrule
\multirow{2}{*}{RoBERTa} & Semantic & 61.61 & 69.46 & \textbf{69.50} & 69.41 & 68.34 & 69.36 & 66.50 & 69.46 & 68.91 & \textbf{70.56} & \textbf{72.03}\\
& Background & \textbf{62.70} & 54.26 & 54.26 & 50.17 & 48.18 & 54.33 & 50.46 & 49.32 & 49.13 & 54.19 & 59.40 \\
\midrule
\multirow{2}{*}{ELECTRA} & Semantic & 62.49 & 63.17 & 63.12 & 60.92 & 62.14 & \textbf{63.23} & 56.85 & 61.00 & \textbf{65.67} & 62.45 & \textbf{64.22}\\
& Background & \textbf{59.35} & 42.96 & 42.96 & 38.68 & 37.96 & 42.77 & 40.66 & 38.53 & 41.25 & 42.63 & 39.31 \\
\bottomrule
\end{tabular}
}
\label{tab:agnostic-ood-shifted}
\end{table}

\Cref{tab:agnostic-ood-shifted} shows the OOD detection performance on the semantic and background shift detection tasks. 
For the semantic shift, \ourmethod{} exhibits suboptimal performance. However, in the case of the background shift, it notably outperforms all other methods, including the open-box approaches, some of which even perform worse than random.
We suspect the subpar performance of other OOD detection methods in background shift detection may be attributed to models performing better on Yelp data compared to the \textsc{sst} data they were trained on (cf.~\Cref{sec:appendix:exp-design}), because Yelp has longer texts with more semantic cues, making models more confident on OOD data.
We speculate the discrepancy in performance between semantic and background shifts arises because \ourmethod{} is focused on the encoding process of the query instances, while other methods only examine the model's outputs. Consequently, \ourmethod{} demonstrates greater sensitivity to the changes in the data-generating distribution. At the same time, other methods are better at detecting changes in the outputs, such as the introduction of an unknown class. In \Cref{sec:appendix:shifts} we show that \ourmethod{} is sensitive to the degree of  distribution shift.

\section{Conclusion}

We have proposed a novel method for out-of-distribution (OOD) detection for Transformer-based networks called \ourmethod{}. The method analyzes representation transformations across intermediate layers and requires only the access to model's weights. Our evaluation on multiple text classification datasets using Transformer-based large pre-trained language models shows that \ourmethod{} outperforms similar methods. 
Our analysis reveals that ID representations undergo smoother transformations between layers compared to OOD representations, because the model concentrates on ID region of the representation space during training. We demonstrated that the learning algorithm retains the original sharpness of the transformations of OOD intermediate representations for simpler datasets but increases the sharpness for more complex datasets. 
%Future work includes applying \ourmethod{} to other domains and developing a theoretical framework to explain the observed differences in between-layer smoothness between ID and OOD data. \fran{Applying \ourmethod{} to other domains is no longer future work.}
\section*{Acknowledgment}

We thank the anonymous reviewers for their insightful comments. Our heartfelt appreciation goes to the members of TakeLab for their continuous support and valuable input. Special thanks to Nina Drobac and Stjepan Šebek for their feedback and helpful suggestions. This work has been supported by the Croatian Science Foundation under the project IP-2020-02-8671 PSYTXT (``Computational Models for Text-Based Personality Prediction and Analysis'').

\bibliography{references}
\bibliographystyle{iclr2023_conference}

\clearpage
\appendix

\section{Reproducibility}

We conducted our experiments on 4× AMD Ryzen Threadripper 3970X 32-Core Processors and 2x NVIDIA GeForce RTX 3090 GPUs with 24GB of RAM, which took a little bit less than three weeks. We used Python 3.8.5, PyTorch \citep{NEURIPS2019_bdbca288} version 1.12.1, Hugging Face Transformers \citep{wolf2020transformers} version 4.21.3, Hugging Face Datasets \citep{lhoest2021datasets} version 2.11.0, scikit-learn \citep{pedregosa2011scikit} version 1.2.2, and CUDA 11.4.

\section{Proof of the Corollary}
\label{sec:appendix:proof}

\begin{repcorollary}{thr:estimate}
   Let $\mJ(\vx) \in \R^{m \times n}$ be a Jacobian matrix, and let $\rvv \in \R^n$ and $\rvw \in \R^m$ be random vectors whose elements are independent random variables with zero mean and unit variance. Then, $\E[(\rvw^{\intercal}\mJ(\vx)\rvv)^2] = \| \mJ(\vx) \|_F^2$.
\end{repcorollary}
\begin{remark}
	Clearly, the result holds true regardless of whether $\mJ(\vx)$ is a Jacobian matrix of some transformation or, indeed, any constant matrix $\mA$ with real entries. Henceforth, we assume the latter is the case.
\end{remark}
It turns out that the requirements on random vectors $\rvv$ and $\rvw$ can be relaxed, and a more general statement from which \Cref{thr:estimate} trivially follows is given in \Cref{thr:general} below:
\begin{theorem}
	\label{thr:general}
	Let $\mA \in \R^{m \times n}$ be a constant $m \times n$ matrix, and let $\rvv \in \R^n$ and $\rvw \in \R^m$ be independent random vectors with identity autocorrelation matrices $\E[\rvv \rvv^{\intercal}] = \mI_n$ and $\E[\rvw \rvw^{\intercal}] = \mI_m$, then  $\E[(\rvw^{\intercal}\mA\rvv)^2] = || \mA ||_F^2$.
\end{theorem}
Before the proof of this theorem, we will need to show two lemmas. The first one says that if one wishes to find a sum of squared $2$-norms of $n$ vectors $\sum_i \|\rva_i\|_2^2$ (which one may interpret as the square of the Frobenius norm $\|\mA\|_F^2$ of the matrix obtained by putting all those vectors together), one can do this stochastically by taking a random linear combination of those vectors and then squaring the $2$-norm of the resulting vector $\|\sum_i \rv_i\rva_i\|_2^2$. The random weights have to satisfy $\E[\rv_i\rv_j] = \delta_{ij}$, where $\delta_{ij}$ is Kronecker delta which is $1$ if $i=j$ and $0$ otherwise.
\begin{lemma}\label{lem:Squaring_principle_for_vectors}
	Let $\mA \in \R^{m \times n}$ be a constant $m \times n$ matrix, and let $\rvv \in \R^n$ be a random vector with identity autocorrelation matrix $\E[\rvv \rvv^{\intercal}]=\mI_n$. Then, $\E[\|\mA \rvv\|_2^2] = \| \mA \|_F^2$.
\end{lemma}
\begin{proof}
	Denote by $\rva_i$ columns of matrix $\mA$, so that $\mA=[\rva_1|\dots|\rva_n]$. The matrix-vector product
	$$\mA\rvv = \sum_i \rv_i\rva_i$$
	where $\rv_i$ denote entries of the vector $\rvv$. Note
	\begin{align*}
		\|\mA \rvv\|_2^2
		&= (\mA \rvv)^{\intercal}\mA \rvv = \left(\sum_i \rv_i\rva_i\right)^\intercal\left(\sum_j \rv_j\rva_j\right)
		= \left(\sum_i \rv_i\rva_i^\intercal\right)\left(\sum_j \rv_j\rva_j\right)\\
		&=  \sum_{i,j} \rva_i^{\intercal}\rva_j \rv_i\rv_j.
	\end{align*}
	Therefore,
	$$\E[\|\mA \rvv\|_2^2]
	= \E[\sum_{i,j} \rva_i^{\intercal}\rva_j \rv_i\rv_j]
	= \sum_{i,j} \rva_i^{\intercal}\rva_j \E[\rv_i\rv_j]
	= \sum_{i} \rva_i^{\intercal}\rva_i = \sum_i \|\rva_i\|_2^2=\|\mA\|_F^2.$$
\end{proof}
The next lemma deals with the same principle as the previous lemma, but this time for scalars rather than vectors. In short, if one wishes to find a sum of squares of $m$ scalars $\sum_j \ra_j^2$ (which one may interpret as the square of the $2$-norm of the vector $\rva$ with those components), one can do this stochastically by taking a random linear combination of those scalars and then just squaring the sum $(\sum_j \rw_j\rva_j)^2$. Again, the random weights have to satisfy $\E[\rw_i\rw_j] = \delta_{ij}$.
\begin{lemma}\label{cor:Squaring_principle_for_numbers}
	Let $\rva \in \R^{m}$ be a constant row-vector, and let $\rvw \in \R^m$ be a random vector with identity autocorrelation matrix $\E[\rvw \rvw^{\intercal}]=\mI_m$. Then, $\E[(\rva^\intercal \rvw)^2] = \| \rva \|_2^2$.
\end{lemma}
\begin{proof}
	This is a direct consequence of \Cref{lem:Squaring_principle_for_vectors}. Just take $\mA = \rva^\intercal$ to be a row-vector and note that the Frobenius norm of that row-vector is just its Euclidean 2-norm.
\end{proof}
\begin{proof}[Proof of \Cref{thr:general}]
	When conditioning on $\rvv$, $\mA\rvv$ is a constant vector and we can use \Cref{cor:Squaring_principle_for_numbers} to write:
	$$\|\mA\rvv\|_2^2
	= \E[\left((\mA\rvv)^\intercal \rvw\right)^2 \mathop{\big|} \rvv]
	= \E[\left(\rvv^\intercal \mA^\intercal \rvw\right)^2 \mathop{\big|} \rvv]
	= \E[\left(\rvw^\intercal \mA \rvv \right)^2 \mathop{\big|} \rvv].$$
	where in the last step we transposed the $1\times 1$ matrix $\rvv^\intercal \mA^\intercal \rvw$.
	Putting this together with \Cref{lem:Squaring_principle_for_vectors} gives:
	$$|| \mA ||_F^2 = \E[ \|\mA\rvv\|_2^2]
	= \E[\E[\left(\rvw^\intercal \mA \rvv \right)^2 \mathop{\big|} \rvv]]
	= \E[\left(\rvw^\intercal \mA \rvv \right)^2].$$
\end{proof}
\begin{remark}
	The estimates above are closely related to the so-called Hutchinson's trick \citep{hutchinson1989trick}, which gives an unbiased estimate of the trace of a matrix as
	$$\tr(\mW) = \E[\rvv^\intercal \mW \rvv],$$ where $\rvv$ satisfies the same conditions as before. Our estimate in \Cref{lem:Squaring_principle_for_vectors} can be seen as its corollary since
	$$\|\mA\|_F^2 = \tr(\mA^\intercal \mA) = \E[\rvv^\intercal \mA^\intercal \mA \rvv] = \E[\|\mA\rvv\|_2^2].$$
    The proof of \Cref{thr:general} is not new. It has appeared in \citep{bujanovic2021theorem}, although the result was stated there in lesser generality. We decided to include both the proofs of \Cref{lem:Squaring_principle_for_vectors} and \Cref{thr:general} for completeness.
\end{remark}
\begin{remark}
	Note that both \Cref{lem:Squaring_principle_for_vectors} and \Cref{thr:general} estimate $\|\mA\|_F^2$. It is possible to show that the variance of the estimator in \Cref{lem:Squaring_principle_for_vectors} is bounded above by the variance of the estimator in \Cref{thr:general}. This should be intuitively clear as the latter takes the vector $\mA\rvv$ and rather than just taking its exact $2$-norm (like the former), it further projects it onto another random vector $\rvw$ in order to estimate its $2$-norm (cf.~Proof of \Cref{thr:general})
\end{remark}
\begin{remark}
The estimate given in \Cref{thr:general} is most useful when both dimensions of the matrix $\mM$ are large, and if obtaining its entries is computationally expensive, but calculating the vector-matrix-vector product can be performed efficiently. If, in addition, one can perform matrix-vector product efficiently (which is the case when, e.g., $m$ is small) it is beneficial to use the estimator given in \Cref{lem:Squaring_principle_for_vectors}. The same is true (by transposing everything) if $n$ is small and/or one can perform vector-matrix product efficiently.
\end{remark}

%\mate{BRISATI:
%\begin{proof}
%    \begin{align*}
%    \intertext{First we need to expand the square,}\\
%         \E[(\rvw^{\intercal}\mA\rvv)^2] &= \E[(\rvw^{\intercal}\mA\rvv)(\rvw^{\intercal}\mA\rvv)^{\intercal}],\\
%    \intertext{then we introduce the trace of $1 \times 1$ matrix,}\\
%        &= \E[\tr(\rvw^{\intercal}\mA\rvv \rvv^{\intercal}\mA^{\intercal}\rvw)],\\
%    \intertext{using the rule of trace rotation we get}\\
%        &= \E[\tr(\rvv\rvv^{\intercal}\mA^{\intercal} \rvw\rvw^{\intercal}\mA)],\\
%    \intertext{the commutativity of trace and expectation gives us,}\\
%        &= \tr(\E[\rvv\rvv^{\intercal}\mA^{\intercal} \rvw\rvw^{\intercal}\mA]),\\
%    \intertext{independence of $\rvv\rvv^{\intercal}\mA^{\intercal}$ and $\rvw\rvw^{\intercal}\mA$ since $\rvv$ and $\rvw$ are independent,}\\
%        &= \tr(\E[\rvv\rvv^{\intercal}\mA^{\intercal}] \E[\rvw\rvw^{\intercal}\mA]),\\
%    \intertext{matrix $\mA$ is constant,}\\
%        &= \tr(\E[\rvv\rvv^{\intercal}]\mA^{\intercal} \E[\rvw\rvw^{\intercal}]\mA),\\
%    \intertext{by the assumptions of the theorem, $\rvv$ and $\rvw$ have identity autocorrelation matrices,}\\
%        &= \tr(\mI_n\mA^{\intercal} \mI_m\mA),\\
%    \intertext{and by the definition of the Frobenius norm,}\\
%        &= \tr(\mA^{\intercal} \mA) = ||\mA||^2_F.
%    \end{align*}
%\end{proof}
%}
\section{Experimental Design}
\label{sec:appendix:exp-design}

In this section, we present details about the models and datasets used in the experiments along with the description and hyperparamters of the training procedures to ensure the reproducibility of the results.

\subsection{Models}

For our experiments, we choose two popular transformer-based large pre-trained language models that obtain state-of-the-art results on a variety of NLP tasks, e.g., text classification, named entity recognition, machine translation, text summarization. Both models have similar architecture with the main difference being the pre-training data and pre-training objectives. We used the same hyperparamters for both models. For fine-tuning we used Adam optimizer \citep{kingma2015adam} with $\beta_1=0.9$, $\beta_2=0.999$, $\epsilon=10^{-8}$, learning rate of $2\times10^{-5}$. We fine-tuned the models for ten epochs. The batch size depends on the dataset used. We repeated each experiment with five different random seeds that varied the initialization of the classification head and the stochastic nature of the learning procedure. The models we use are:

\begin{itemize}
    \item \textbf{RoBERTa} -- Uses masked language modeling (MLM) pre-training objective. The model has 12 layers, a hidden state size of 768, and 12 attention heads with 125M parameters in total; 
    \item \textbf{ELECTRA} -- Unlike RobERTa which uses generative pre-training objective, ELECTRA uses discriminative pre-training objective. The model has 12 layers, a hidden state size of 768, and 12 attention heads with 110M parameters in total.
\end{itemize}

\subsection{Datasets}

In our experiments, we work with several text classification datasets. Datasets we used as ID data and their preprocessing procedures are:

\begin{itemize}
    \item \textbf{\textsc{sst}} -- The Stanford Sentiment Treebank dataset contains single sentences extracted from movie reviews labeled with the sentiment of the review. The task is an almost balanced binary classification with labels corresponding to positive and negative sentiment;
    \item \textbf{\textsc{subj}} -- The Subjectivity dataset is a collection of movie review documents. The task is to classify the reviews into one of two balanced classes based on the nature of the review: objective or subjective;
    \item \textbf{\textsc{agn}} -- The AG News topic classification dataset consists of news articles from several news sources. The dataset consists of four balanced classes representing the topic of the article: World, Sports, Business, and Sci/Tech. The train split of the dataset was subsampled to 20000 instances for our experiments keeping the balance of the labels;
    \item \textbf{\textsc{trec}} -- The Text REtrieval Conference (TREC) Question Classification dataset gathers questions labeled with their topics: Abbreviation, Entity, Description and abstract concept, Human being, Location, and Numeric value. The dataset contains six imbalanced labels;
    \item \textbf{\textsc{bp}} -- BigPatent consists of records of U.S. patent descriptions along with human written abstractive summaries from nine patent categories: Human Necessities, Performing Operations and Transporting, Chemistry and Metallurgy, Textiles and Paper, Fixed Constructions, Mechanical Engineering and Lightning and Heating and Weapons and Blasting, Physics, Electricity, and General tagging of new or cross-sectional technology. Even though this dataset is usually used for summarization we use the summaries for classification. We remove all duplicates from train and test splits and between train and test splits. We subsample the whole dataset by taking 3\% of the original train set and 20\% of the original test set;  
    \item \textbf{\textsc{ar}} -- Amazon Customer Reviews dataset contains customer reviews of products from Amazon store. The dataset contains the product category along with the review text. We subset all of the categories to 12 with comparable sizes and significant semantic differences: Gift Card, Software, Video Games, Luggage, Video, Grocery, Furniture, Musical Instruments, Watches, Tools, Baby, and Jewelry. We preprocess the data first dropping all of the reviews with less than 15 words and all of the duplicates from the training split. We then subsample 0.25\% of the data and split that subsample into train and test sets with sizes of 80\% and 20\% of the subsample size, respectively;
    \item \textbf{\textsc{mg}} -- IMDb Genre Classification Dataset is used for the classification of movies' genres from their descriptions from IMDb. In our experiment, we use a subset of 15 biggest genres with significant semantic differences: Drama, Documentary, Comedy, Horror, Thriller, Action, Western, Reality TV, Adventure, Family, Music, Romance, Sci-fi, Adult, and Crime. We preprocess the data by first removing all duplicates from train and test splits and between train and test splits, and then subsample the train data to 50\% of the original size and test data to 15\% of the original size;
    \item \textbf{\textsc{ng}} -- 20Newsgroups data is a collection of news documents labeled based on the topic of the news with 20 different labels that are almost uniformly distributed. Following sci-kit-learn, we preprocess both train and test data by removing headers, signature blocks, and quotation blocks to eliminate  simple correlations to which models easily overfit. We also remove any potential duplicate documents between train and test sets to avoid data leaks.
\end{itemize}

Simplified datasets were preprocessed the same as their original counterparts described above. After the preprocesing, datasets were simplified by removing the data from all of the labels but two from both train and test sets. The choice of the retained two labels was made on the basis of the absolute and relative sizes of the data with given labels and semantic differences between labels. 

\begin{itemize}
    \item \textbf{\textsc{bp2}} -- BigPatent2 retains labels Human Necessities and Physics.
    \item \textbf{\textsc{ar2}} -- AmazonReviews2 retains labels Grocery and Baby.
    \item \textbf{\textsc{mg2}} -- MovieGenre2 retains labels Drama and Documentary.
\end{itemize}

Datasets we used as OOD data and their preprocessing procedures are:

\begin{itemize}
    \item \textbf{\textsc{obw}} -- One Billion Word Benchmark is a popular dataset for training and evaluating language models. In this paper, we use it as the OOD data for that reason, i.e., the diversity of the corpus. We use the test set of the corpus and in each experiment, we subsample it to the size of the ID test set.
    \item \textbf{Yelp} -- The Yelp Reviews dataset consists of reviews from Yelp labeled for the sentiment of the review, e.g., positive or negative. The task is balanced binary sentiment classification. When using this dataset as OOD data we subsample the test set of the dataset to the size of the ID test set. Transfer of RoBERTa model fine-tuned on \textsc{sst} data achieves $F_1$ score of 94.13 $\pm$ 0.79, while transfer of ELECTRA model fine-tuned on \textsc{sst} data achieves $F_1$ score of 95.05 $\pm$ 1.24.
\end{itemize}

More details about the used datasets are shown in \Cref{tab:data-stats}.

\begin{table}
\caption{Information about the datasets used in the experiments. The table contains sizes of the train and ID test sets, number of classes in each dataset, micro $F_1$ score on the test set for RoBERTa and ELECTRA, minimum description length (MDL) as an estimate of data complexity (Based on Rissanen Data Analysis. We take an average of five runs with RoBERTa and ELECTRA. MDL is computed with 128 blocks of 32 instances per batch, Adam optimizer and a learning rate of $2\times10^{-5}$. Numbers in the table are normalized with respect to the largest value.), and batch size used to train the model.}
\small
\centering
\begin{tabular}{lcccccccc}
\toprule
Dataset & Train & Test & \# Labels & $F_1^{\text{RoBERTa}}$ & $F_1^{\text{ELECTRA}}$ & MDL & Batch Size\\
\midrule
\textsc{sst} & 8544 & 2210 & 2 & 87.70 $\pm$ 1.17 & 88.60 $\pm$ 0.50 & 0.389 & 32 \\ 
\textsc{subj} & 8000 & 2000 & 2 & 96.30 $\pm$ 0.49 & 96.83 $\pm$ 0.62 & 0.148 & 32\\
\textsc{agn} & 20000 & 7600 & 4 & 92.63 $\pm$ 0.35 & 92.18 $\pm$ 0.80 & 0.279 & 16\\
\textsc{trec} & 5452 & 500 & 6 & 96.92 $\pm$ 0.37 & 96.56 $\pm$ 0.73 & 0.138 & 32\\ 
\textsc{bp} & 20792 & 7678 & 9 & 64.24 $\pm$ 0.67 & 63.87 $\pm$ 0.92 & 1 & 16\\
\textsc{ar} & 20336 & 5085 & 12 & 87.26 $\pm$ 0.40 & 86.65 $\pm$ 0.70 &  0.324 & 16\\
\textsc{mg} & 22832 & 6849 & 15 & 70.99 $\pm$ 0.57 & 71.32 $\pm$ 0.44 &  0.652 & 16\\
\textsc{ng} & 11314 & 7306 & 20 & 72.65 $\pm$ 0.37 & 71.38 $\pm$ 0.44 & 0.751 & 16\\
\midrule
\textsc{bp2} & 7410 & 2735 & 2 & 93.52 $\pm$ 0.22 & 93.61 $\pm$ 0.50 &  0.253 & 16\\
\textsc{ar2} & 6388 & 1597 & 2 & 97.27 $\pm$ 0.59 & 97.85 $\pm$ 0.27 & 0.124 & 16\\
\textsc{mg2} & 13291 & 3986 & 2 & 93.38 $\pm$ 0.36 & 93.30 $\pm$ 0.28 & 0.281 & 16\\
\bottomrule
\end{tabular}
\label{tab:data-stats}
\end{table}

\section{BLOOD for image classification}

To explore the effectiveness of our BLOOD methodology on image data, we utilized CIFAR-10 and CIFAR-100 as our ID datasets \citep{krizhevsky2009learning}. These were chosen for their resemblance to the ImageNet dataset, which was used to pre-train the Vision Transformer (ViT) \citep{dosovitskiy2021image}. We tested OOD detection using two different datasets: Street View House Numbers \citep[SVHN;][]{yuval2011reading}, which comprises real-world imagery of house numbers offering different visual features from CIFAR datasets, and the Beans dataset, which contains images relevant to agricultural disease classification, adding another layer of diversity to our evaluation.

In addition to the Transformer architecture, we also employ convolutional neural networks (CNNs), specifically ResNet34 and ResNet50 \citep{he2016deep}, to gauge how our method performs on other architectures. For CNNs, we employed a learning rate of $0.01$ and trained the models over $50$ epochs. In contrast, for our transformer-based model, the ViT, we adopted a fine-tuning approach, utilizing a learning rate of $2\times10^{-5}$ with an additional weight decay of $0.001$, over a shorter span of $10$ epochs. This distinction in training methodologies is reflective of the different learning dynamics and capacities of CNNs versus Transformers.

We show the results in \Cref{tab:cv}. The ViT model, with its transformer architecture, showed a marked improvement in AUROC with \ourmethodL{}, underscoring the critical role of the last layer in such models. This aligns with our observations from textual data analysis, where the most significant changes also occurred in the last layer, suggesting a pattern that transformers exhibit across different data modalities.

Conversely, the CNN models, ResNet34 and ResNet50, displayed more nuanced changes across their layers. Since these models were trained from scratch, the learning occurred more prominently in the lower layers, and the last layer often had an inverse impact on OOD detection capabilities. This was evidenced by low AUROC scores for \ourmethodL{}, while \ourmethodmean{} remained competitive with other methods. Additionally, \ourmethod{}'s ability to leverage the complexity of datasets was apparent both with ViT and CNNs, particularly with the more granular CIFAR-100 dataset compared to CIFAR-10, which is consistent with the observations on textual data.

\begin{table}[]
\caption{The performance of OOD detection methods on the task of OOD detection on image classification measured by AUROC (\%). The best-performing method is in \textbf{bold}. Higher is better.}
\small
\centering
\scalebox{0.95}{
\begin{tabular}{llcccccccc} 
\toprule
Model & ID Dataset & \ourmethodmean{} & \ourmethodL{} & MSP & ENT & EGY & MC & GRAD & ASH \\
\midrule
& & \multicolumn{8}{c}{SVHN}\\
\cmidrule(lr){3-10}
\multirow{2}{*}{ResNet34} & CIFAR-10 & 84.00 & 49.99 & 88.92 & \textbf{89.71} & 88.23 & 81.39 & 82.29 & 83.15 \\
& CIFAR-100 & \textbf{87.05} & 49.51 & 73.64 & 76.20 & 77.90 & 78.31 & 85.97 & 80.54 \\
\midrule
\multirow{2}{*}{ResNet50} & CIFAR-10 & 84.13 & 49.91 & 89.55 & \textbf{90.29} & 84.36 & 90.07 & 88.01 & 88.20 \\
& CIFAR-100 & 79.77 & 50.50 & 79.86 & 83.32 & 82.70 & \textbf{83.42} & 81.80 & 80.89 \\
\midrule
\multirow{2}{*}{ViT} & CIFAR-10 & 99.37 & 92.45 & 99.12 & 98.19 & \textbf{99.55} & 98.91 & 96.94 & 95.01 \\
& CIFAR-100 & 80.97 & \textbf{89.07} & 84.06 & 84.58 & 88.21 & 87.47 & 85.53 & 81.29 \\
\midrule
& & \multicolumn{8}{c}{Beans}\\
\cmidrule(lr){3-10}
\multirow{2}{*}{ResNet34} & CIFAR-10 & 85.68 & 14.16 & 89.23 & 91.55 & \textbf{92.41} & 91.09 & 81.04 & 90.67 \\
& CIFAR-100 & \textbf{86.40} & 23.54 & 73.32 & 79.90 & 81.22 & 83.97 & 76.55 & 81.02 \\
\midrule
\multirow{2}{*}{ResNet50} & CIFAR-10 & 78.37 & 50.28 & 79.18 & \textbf{80.45} & 79.02 & 79.62 & 69.33 & 77.26 \\
& CIFAR-100 & \textbf{86.62} & 49.76 & 79.13 & 80.79 & 82.31 & 83.07 & 85.29 & 75.16 \\
\midrule
\multirow{2}{*}{ViT} & CIFAR-10 & 95.41 & 99.58 & 99.31 & 99.31 & \textbf{99.98} & 98.79 & 96.92 & 94.31 \\
& CIFAR-100 & 91.68 & \textbf{96.53} & 96.02 & 95.71 & 95.81 & 95.83 & 89.33 & 88.14 \\
\toprule
\end{tabular}
}
\label{tab:cv}
\end{table}

\section{BLOOD as an Open-Box Method}

In the scenarios where more resources are available, beyond just the trained model, it is possible to utilize those resources to improve the OOD detection performance of \ourmethod{}. Similarly to a lot of popular OOD detection methods \citep{liang2018enhancing,sun2021react,sun2022dice}, \ourmethod{} can be improved by using a small validation set to learn the optimal weights for the weighted average of the \ourmethod{} score in each layer.

To obtain the weights for the weighted sum, first we create a validation set from the 5\% of our ID and OOD test sets and then fit the logistic regression. In \Cref{tab:blood-open-box} we showcase the results of using \ourmethod{} in the open-box scenarios. From the results it can be seen that it is useful to extend \ourmethod{} to open-box setting if the validation set is obtainable. In rare cases \ourmethodL{} outperform the open-box \ourmethod{} likely due to noise introduced by including the lower layers. But still, the differences in performance when open-box \ourmethod{} is outperformed by the \ourmethodL{} are minuscule compared to the gains in other cases.

\begin{table}[]
\caption{Performance of \ourmethod{} in an open-box setting. Cases in which the open-box \ourmethod{} ouperforms both \ourmethodmean{} and \ourmethodL{} are in \textbf{bold}.}
\small
\centering
\scalebox{1}{
\begin{tabular}{l|cccccccc}
\toprule
Model   & \textsc{SST}   & \textsc{SUBJ}  & \textsc{AGN}   & \textsc{TREC}  & \textsc{BP}    &  \textsc{AR}    &  \textsc{MG}    & \textsc{NG}    \\ \midrule
RoBERTa & \textbf{73.94} & \textbf{82.98} & \textbf{81.39} & 91.73 & \textbf{93.47} & \textbf{96.25} & \textbf{92.46} & \textbf{89.97} \\ 
ELECTRA & 77.67 & 77.30 & \textbf{82.76} & 98.73 & 96.54 & \textbf{91.97} & \textbf{90.92} & \textbf{85.19} \\ \bottomrule
\end{tabular}
}
\label{tab:blood-open-box}
\end{table}
\section{Dataset Complexity and Uncertainty}
\label{sec:appendix:carto}

The identification of OOD instances often relies on estimating the underlying uncertainty of the model \citep{ovadia2019can}. The intuition is that a well-performing model should exhibit higher confidence when dealing with data resembling the training data, i.e., ID data.
However, ML models are susceptible to two sources of uncertainty.
\textit{Aleatoric} uncertainty stems from the inherent ambiguity and noise in the data, and it is thus irreducible by the acquisition of more data and characteristic of ambiguous ID data;
In contrast, \textit{epistemic} uncertainty stems from the lack of relevant information in the training data, and it is thus reducible by acquiring more relevant data and characteristic of the OOD data \citep{der2009aleatory,kendall2017uncertainties}.
The total amount of uncertainty about a given prediction, i.e., aleatoric and epistemic uncertainty, is called \textit{predictive uncertainty} \citep{gal2016uncertainty}.
%Aleatoric uncertainty stems from the inherent ambiguity and noise in the data, and it is thus irreducible by the acquisition of more data and characteristic of ambiguous ID data; In contrast, epistemic uncertainty stems from the lack of relevant information in the training data, and it is thus reducible by acquiring more relevant data and characteristic of the OOD data \citep{der2009aleatory,kendall2017uncertainties}.

In our experiments, we use probabilistic baselines such as MSP, ENT, and MC. However, these approaches cannot reliably disentangle the epistemic uncertainty from the model's predictive uncertainty \citep{kirsch2021pitfalls}. 

In \Cref{sec:complexity} we show our method, \ourmethod{}, performs better on complex datasets than on simpler datasets. The results of our experiment also show the opposite of that for probabilistic methods (MSP, ENT, and MC), i.e., they work better on simpler datasets than on more complex datasets. We hypothesize that the effect of probabilistic methods working better on simpler datasets comes from the amount of aleatoric uncertainty in the dataset. Because of their simplicity, simpler datasets have less ambiguous instances, and thus with lowered aleatoric uncertainty, epistemic uncertainty of OOD instances dominates the model's predictive uncertainty. 

\begin{figure}[t!]
\small
\centering
\begin{subfigure}{.49\linewidth}
  \centering
  \includegraphics[width=\linewidth]{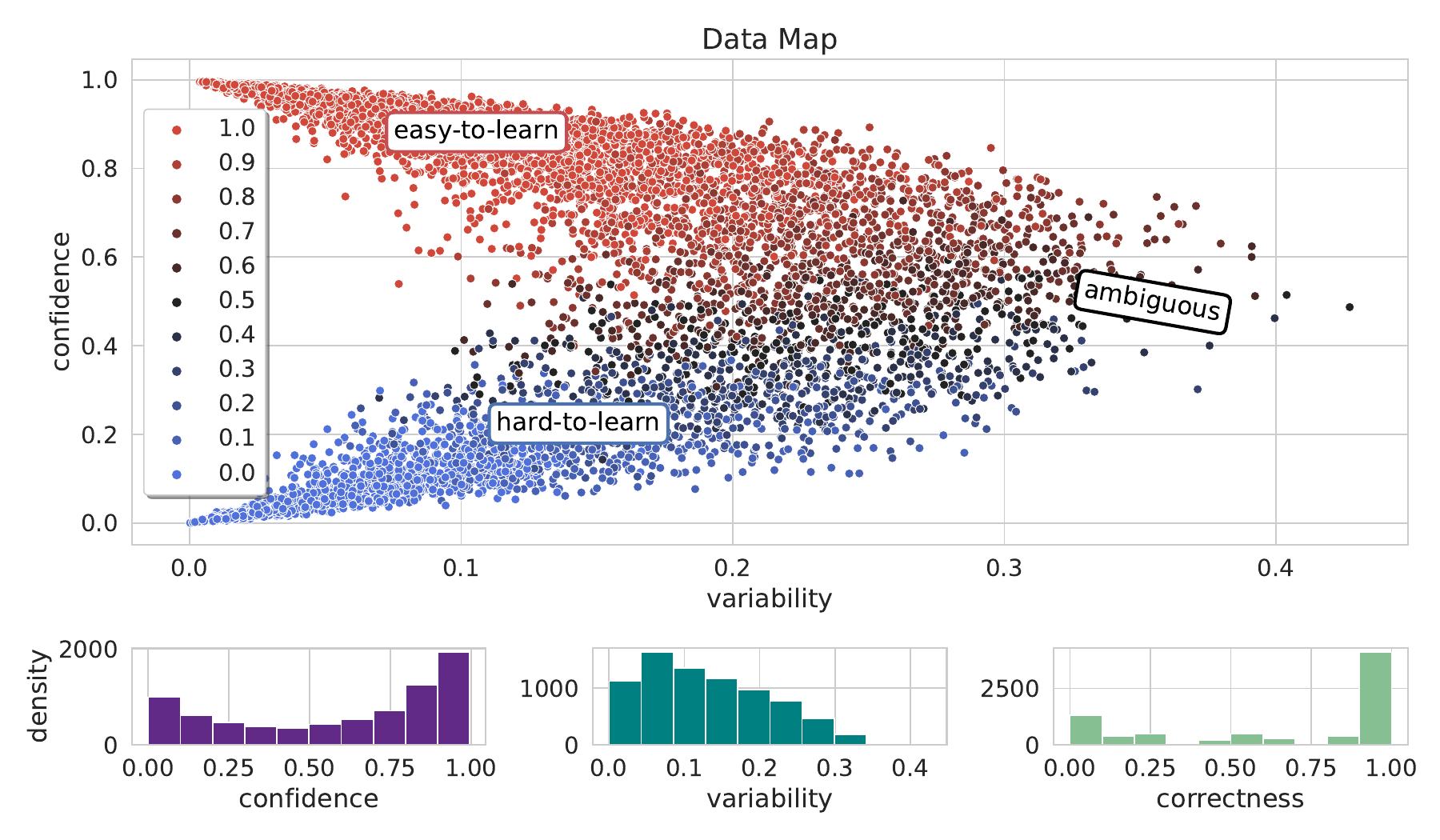}
  \caption{\textsc{bp}}
\end{subfigure}
\begin{subfigure}{.49\linewidth}
  \centering
  \includegraphics[width=\linewidth]{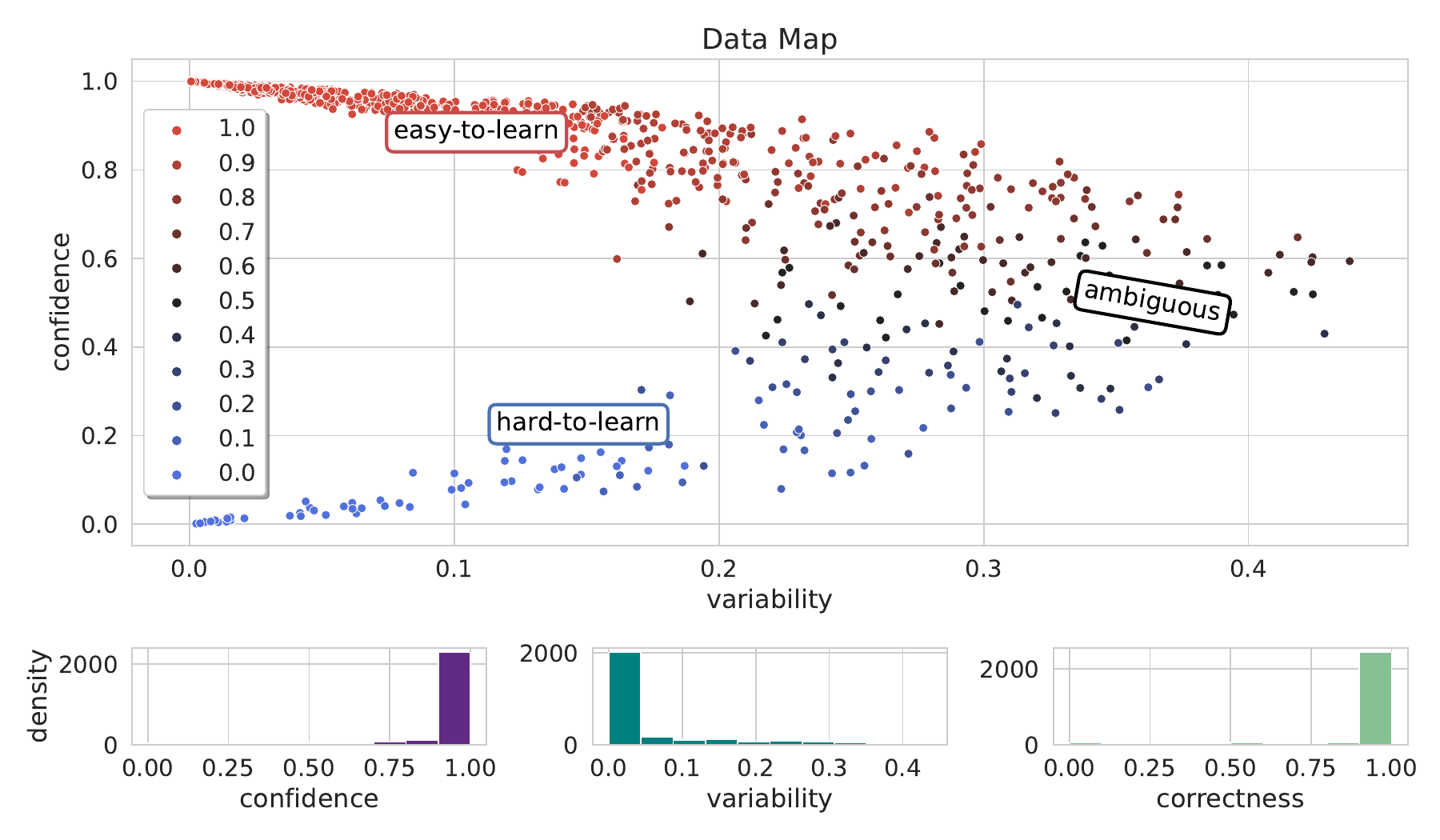}
  \caption{\textsc{bp2}}
\end{subfigure}
\begin{subfigure}{.49\linewidth}
  \centering
  \includegraphics[width=\linewidth]{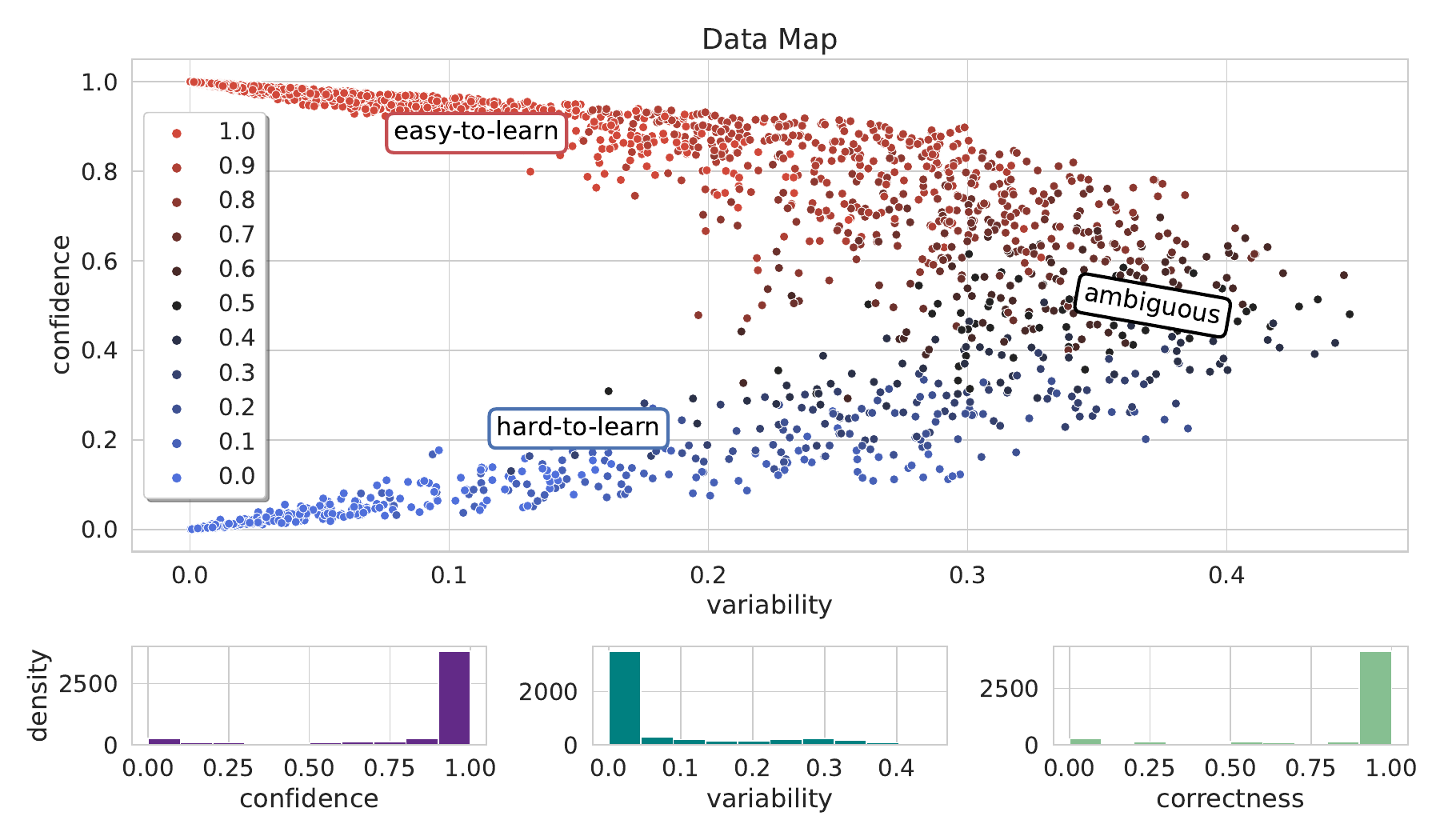}
  \caption{\textsc{ar}}
\end{subfigure}
\begin{subfigure}{.49\linewidth}
  \centering
  \includegraphics[width=\linewidth]{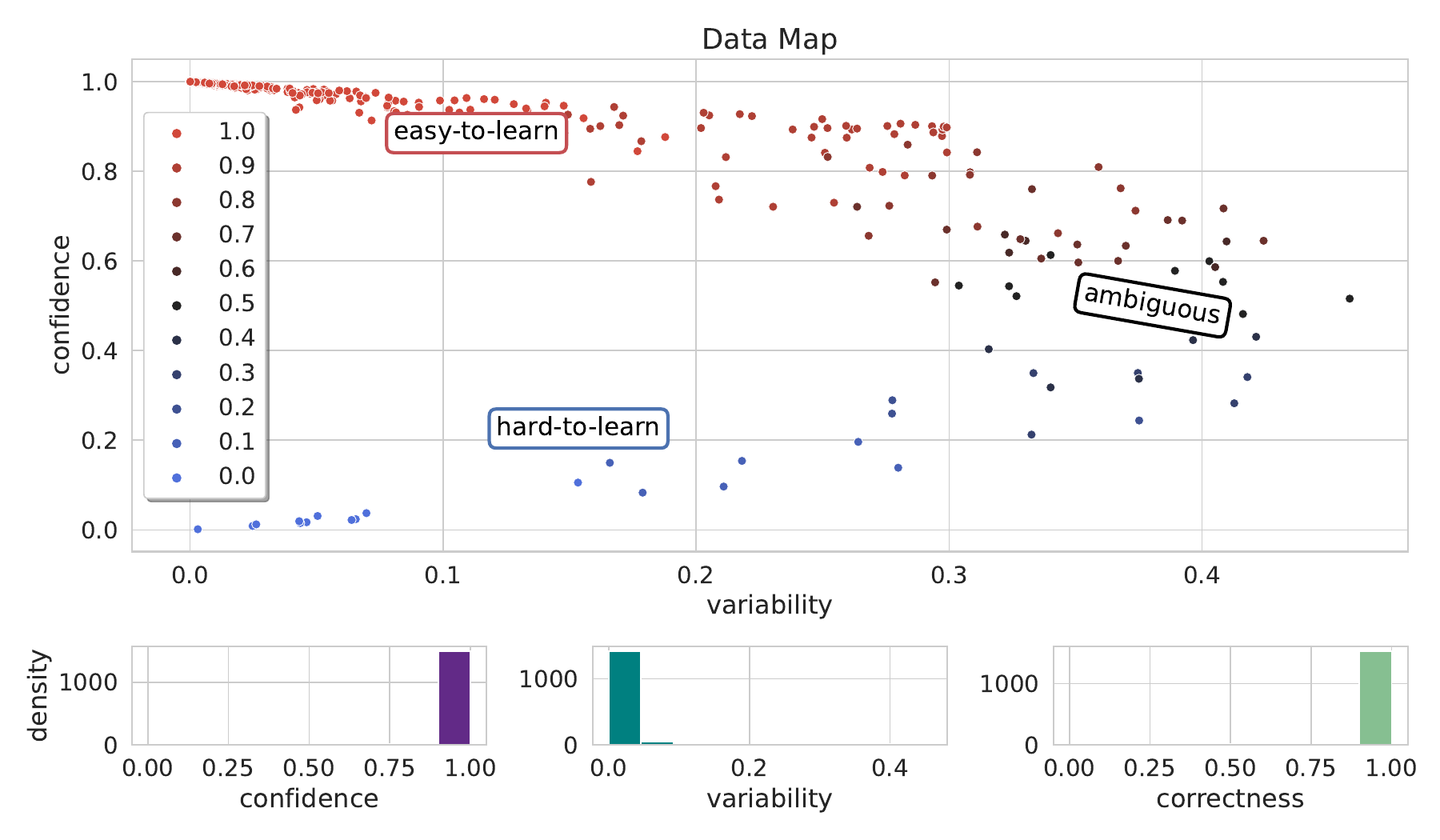}
  \caption{\textsc{ar2}}
\end{subfigure}
\begin{subfigure}{.49\linewidth}
  \centering
  \includegraphics[width=\linewidth]{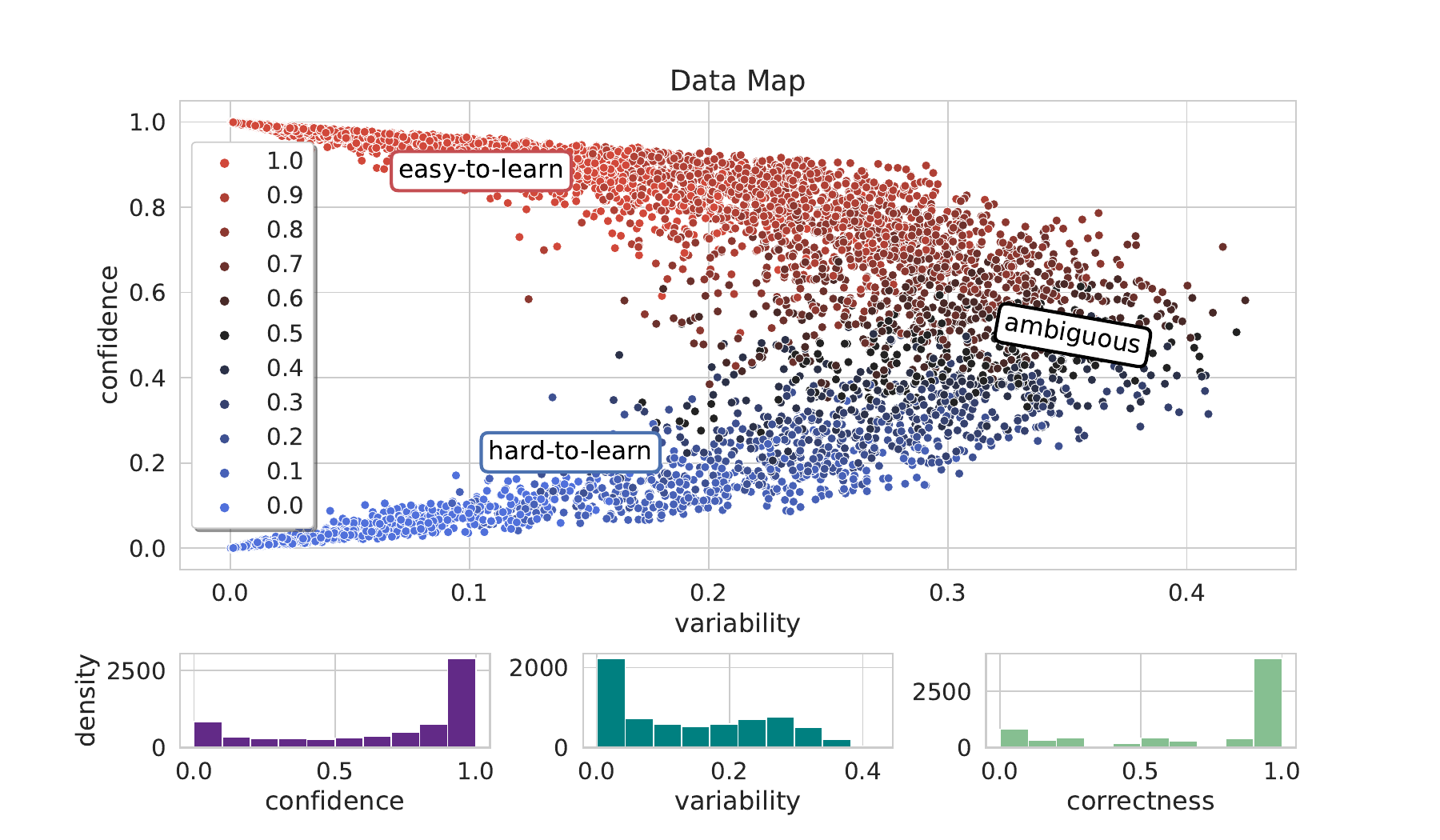}
  \caption{\textsc{mg}}
\end{subfigure}
\begin{subfigure}{.49\linewidth}
  \centering
  \includegraphics[width=\linewidth]{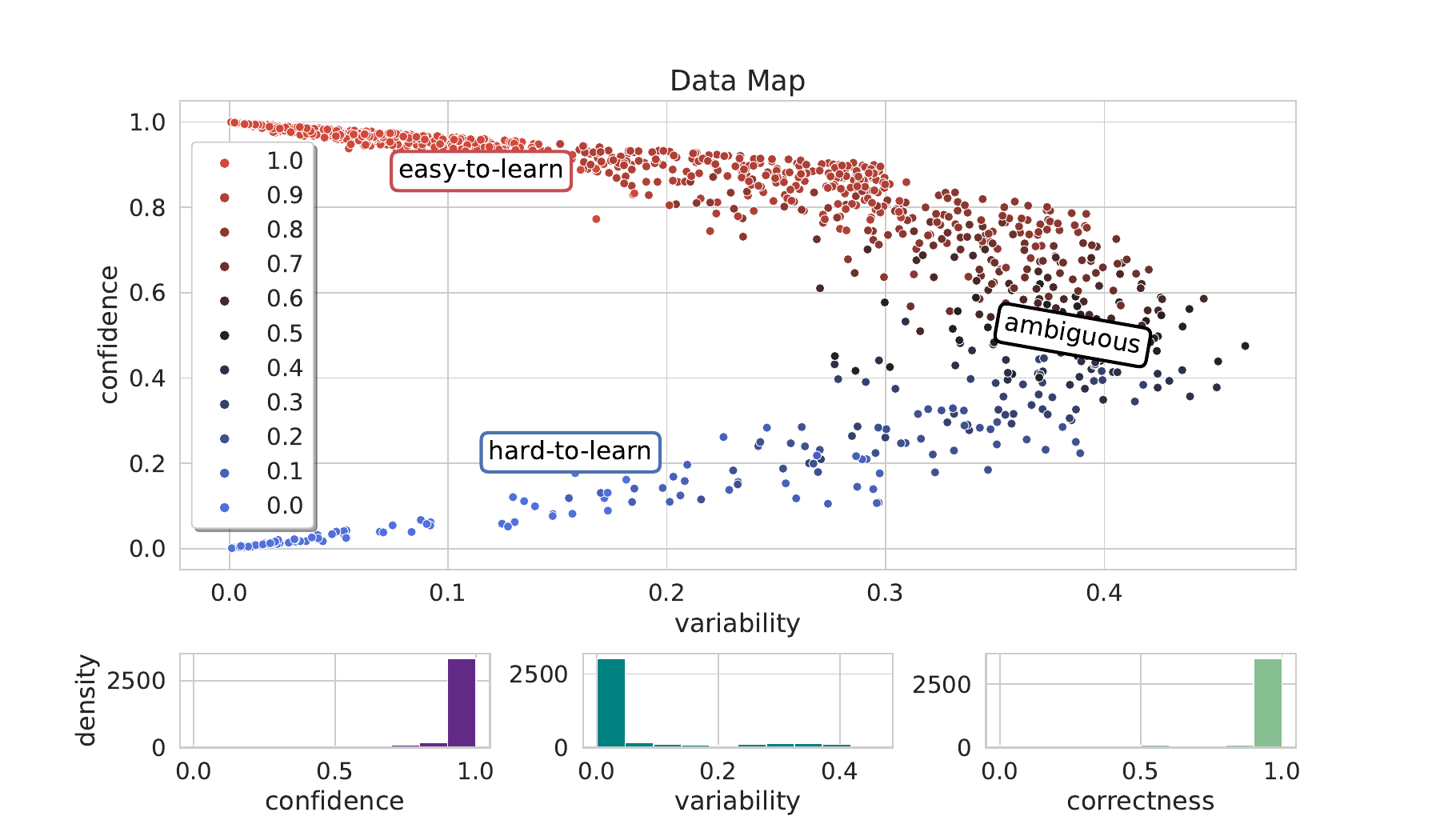}
  \caption{\textsc{mg2}}
\end{subfigure}
\caption{Data maps with RoBERTa for test sets of (a) \textsc{bp}, (b) \textsc{bp2}, (c) \textsc{ar}, \textsc{ar2}, \textsc{mg}, and \textsc{mg2}. Each subfigure shows data map and histograms of confidence, variability, and correctness of instances.  Data maps for ELECTRA are qualitatively the same.}
\label{fig:cartography}
\end{figure}

To visualize the drop in ambiguity, we use data cartography \citep{swayamdipta2020dataset}. The idea is to use the training dynamics of the examples to discover easy-to-learn, hard-to-learn, and ambiguous examples in the dataset. Training dynamics used to create data maps are confidence (how confident the model is in the true label across epochs), variability (the spread of the posterior probability of the true label across epochs), and correctness (the fraction of times the model correctly labels the example across epochs). Ambiguous examples are characterized by high variability, and it can be seen from \Cref{fig:cartography} that simplified datasets, by removing classes, lowered the density of ambiguous examples in the ID dataset. Since the other \whitebox{} methods can not disentangle aleatoric from epistemic uncertainty, lowering the density of ambiguous examples helps them capture the epistemic uncertainty in the OOD data needed for their detection.
\section{Degree of Distribution Shift}
\label{sec:appendix:shifts}

\begin{figure}[t!]
\small
\centering
\begin{subfigure}{.49\linewidth}
  \centering
  \includegraphics[width=\linewidth]{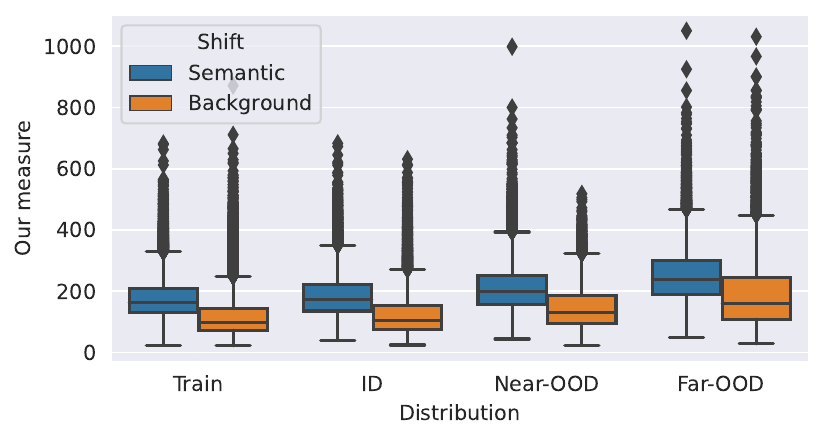}
  \caption{RoBERTa}
\end{subfigure}
\begin{subfigure}{.49\linewidth}
  \centering
  \includegraphics[width=\linewidth]{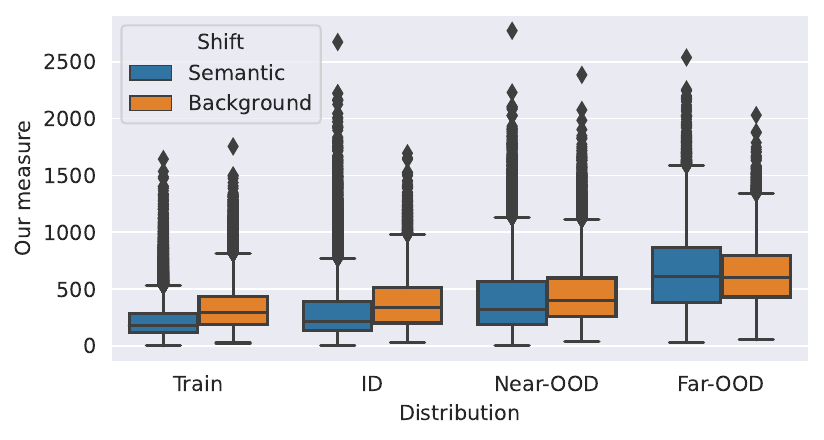}
  \caption{ELECTRA}
\end{subfigure}
\caption{Box plots of change in \ourmethodL{} scores with an increase in the degree of distribution shift for the tasks of semantic and background shift detection for (a) RoBERTa and (b) ELECTRA. The amount of distribution shift increases from left to right: training distribution, test ID data distribution, Near-OOD distribution, and Far-OOD distribution.}
\label{fig:shifted}
\end{figure}

Another important feature of an OOD detection method is the proportional sensitivity to the degree of distribution shift involved in the data \citep{ovadia2019can}. In \Cref{fig:shifted}, we show that the uncertainty scores produced by \ourmethodL{} increase in proportion to the degree of distribution shift. Training data exhibits the lowest uncertainty score, ID test data shows only a slight increase in uncertainty, Near-OOD data exhibits a jump in the uncertainty score, while for Far-OOD data, the uncertainty score is the highest. 
\section{Additional Results}
\label{sec:appendix:extra-results}

In \Cref{sec:experiments} we present the results of our experiments measured with AUROC averaged across five different seeds. In this section, we present averaged results alongside their standard deviations for AUROC as well as two other commonly used metrics in the OOD detection literature, AUPR-IN and FPR@95TPR:
\begin{itemize}
    \item \textbf{AUPR-IN} -- area under the Precision-Recall curve illustrates how precision and recall vary with different thresholds of OOD detection method's uncertainty score. The ID data is considered a positive class. A higher score indicates better performance.
    \item \textbf{FPR@95TPR} -- the false positive rate of an OOD classifier when the true positive rate is 95\%, The ID data is considered a positive class. A lower score indicates better performance.
\end{itemize}
Results of OOD detection measured with AUROC are given in \Cref{tab:agnostic-ood-a} for white-box/black-box methods and in \Cref{tab:non-agnostic-ood-a} for open-box methods. Results for simplified datasets measured with AUROC are given in \Cref{tab:agnostic-ood-augmented-a} for white-box/black-box methods and in \Cref{tab:non-agnostic-ood-augmented-a} for open-box methods.
Results of distribution Nera-OOD detection measured with AUROC are given in \Cref{tab:agnostic-ood-shifted-a} for white-box/black-box methods and in \Cref{tab:non-agnostic-ood-shifted-a} for open-box methods.

Results of OOD detection measured with AUPR are given in \Cref{tab:agnostic-ood-aupr} for white-box/black-box methods and in \Cref{tab:non-agnostic-ood-aupr} for open-box methods. Results for simplified datasets measured with AUPR are given in \Cref{tab:agnostic-ood-augmented-aupr} for white-box/black-box methods and in \Cref{tab:non-agnostic-ood-augmented-aupr} for open-box methods.
Results of distribution Near-OOD detection measured with AUPR are given in \Cref{tab:agnostic-ood-shifted-aupr} for white-box/black-box methods and in \Cref{tab:non-agnostic-ood-shifted-aupr} for open-box methods.

Results of OOD detection measured with FPR@95TPR are given in \Cref{tab:agnostic-ood-fpr} for white-box/black-box methods and in \Cref{tab:non-agnostic-ood-fpr} for open-box methods. Results for simplified datasets measured with FPR@95TPR are given in \Cref{tab:agnostic-ood-augmented-fpr} for white-box/black-box methods and in \Cref{tab:non-agnostic-ood-augmented-fpr} for open-box methods.
Results of distribution Near-OOD detection measured with FPR@95TPR are given in \Cref{tab:agnostic-ood-shifted-fpr} for white-box/black-box methods and in \Cref{tab:non-agnostic-ood-shifted-fpr} for open-box methods.

We also provide visualization of the assessment of how changes in individual layers influence the BLOOD score throughout intermediate layers. Figures \ref{fig:repr-change-sst}, \ref{fig:repr-change-subj}, \ref{fig:repr-change-agn}, \ref{fig:repr-change-trec}, \ref{fig:repr-change-bp}, \ref{fig:repr-change-mg}, and \ref{fig:repr-change-ng} show vizualizations for \textsc{sst}, \textsc{subj}, \textsc{agn}, \textsc{trec} \textsc{bp}, \textsc{mg}, and \textsc{ng} respectively akin to \Cref{fig:repr-change} for \textsc{ar}.

\begin{table}[]
\caption{OOD detection performance measured by AUROC (\%) with standard deviation of the white-box/black-box methods. The best-performing measure is in \textbf{bold}. Higher is better.}
\small
\centering
\scalebox{0.8}{
\begin{tabular}{llccccccc} 
\toprule
Model & Dataset & \ourmethodmean{} & \ourmethodL{} & MSP & ENT & EGY & MC & GRAD \\
\midrule
\multirow{8}{*}{RoBERTa} & \textsc{sst} & 50.56 $\pm$ 7.03 & \textbf{72.83 $\pm$ 9.72} & 71.61 $\pm$ 0.84 & 71.69 $\pm$ 1.38 & 71.69 $\pm$ 1.38 & 68.28 $\pm$ 1.16 & 71.76 $\pm$ 1.34 \\ 
& \textsc{subj} & 52.02 $\pm$ 14.01 & 74.66 $\pm$ 7.32 & \textbf{75.79 $\pm$ 8.95} & 74.55 $\pm$ 8.48 & 74.55 $\pm$ 8.48 & 74.21 $\pm$ 6.82 & 74.93 $\pm$ 8.57 \\
& \textsc{agn} & 77.46 $\pm$ 5.73 & 61.95 $\pm$ 8.87 & 76.36 $\pm$ 3.34 & 73.57 $\pm$ 2.96 & 73.8 $\pm$ 2.99 & \textbf{77.55 $\pm$ 3.04} & 73.58 $\pm$ 2.94 \\
& \textsc{trec} &  69.63 $\pm$ 10.11 & 95.3 $\pm$ 2.96 & 96.28 $\pm$ 0.74 & 96.2 $\pm$ 0.83 & \textbf{96.40 $\pm$ 0.79} & 95.68 $\pm$ 0.78 & 96.14 $\pm$ 0.84 \\ 
& \textsc{bp} & 87.20 $\pm$ 2.95 & \textbf{89.53 $\pm$ 3.37} & 85.84 $\pm$ 1.34 & 70.15 $\pm$ 0.84 & 72.82 $\pm$ 0.98 & 74.29 $\pm$ 1.06 & 73.11 $\pm$ 1.49 \\
& \textsc{ar} & 91.41 $\pm$ 1.78 & \textbf{93.20 $\pm$ 1.24} & 92.39 $\pm$ 0.77 & 89.06 $\pm$ 1.11 & 89.96 $\pm$ 1.05 & 90.59 $\pm$ 0.92 & 88.65 $\pm$ 1.06 \\
& \textsc{mg} & \textbf{88.15 $\pm$ 1.87} & 85.23 $\pm$ 3.99 & 86.45 $\pm$ 2.87 & 75.02 $\pm$ 2.33 & 76.6 $\pm$ 2.47 & 79.98 $\pm$ 2.34 & 74.28 $\pm$ 2.27 \\
& \textsc{ng} & \textbf{83.53 $\pm$ 1.13} & 72.02 $\pm$ 4.76 & 82.65 $\pm$ 0.73 & 77.49 $\pm$ 0.73 & 78.76 $\pm$ 0.82 & 79.32 $\pm$ 1.10 & 76.93 $\pm$ 0.75 \\
\midrule
\multirow{8}{*}{ELECTRA} & \textsc{sst} & 74.36 $\pm$ 2.77 & \textbf{78.11 $\pm$ 2.09} & 71.97 $\pm$ 1.72 & 73.84 $\pm$ 1.89 & 73.84 $\pm$ 1.89 & 70.81 $\pm$ 2.01 & 73.82 $\pm$ 1.92 \\
& \textsc{subj} & 74.1 $\pm$ 11.02 & 77.41 $\pm$ 11.41 & 70.46 $\pm$ 7.16 & \textbf{78.17 $\pm$ 7.78} & \textbf{78.17 $\pm$ 7.78} & 77.71 $\pm$ 8.40 & 78.11 $\pm$ 7.86 \\
& \textsc{agn} & 65.67 $\pm$ 6.90 & \textbf{80.28 $\pm$ 3.41} & 79.75 $\pm$ 3.96 & 76.8 $\pm$ 3.18 & 77.01 $\pm$ 3.28 & 79.55 $\pm$ 2.74 & 76.57 $\pm$ 3.15 \\
& \textsc{trec} & 97.48 $\pm$ 0.84 & \textbf{98.90 $\pm$ 0.37} & 97.48 $\pm$ 0.65 & 97.26 $\pm$ 0.93 & 97.56 $\pm$ 0.82 & 96.21 $\pm$ 0.79 & 97.07 $\pm$ 0.96 \\
& \textsc{bp} & 86.06 $\pm$ 1.85 & \textbf{96.72 $\pm$ 1.4} & 84.63 $\pm$ 1.80 & 78.56 $\pm$ 2.57 & 81.75 $\pm$ 2.58 & 83.04 $\pm$ 2.33 & 76.77 $\pm$ 3.04 \\
& \textsc{ar} & 84.58 $\pm$ 3.49 & \textbf{91.66 $\pm$ 2.59} & 90.64 $\pm$ 1.64 & 87.74 $\pm$ 1.68 & 88.44 $\pm$ 1.76 & 88.53 $\pm$ 2.16 & 87.52 $\pm$ 1.62 \\
& \textsc{mg} & 80.52 $\pm$ 7.30 & \textbf{90.63 $\pm$ 3.56} & 80.41 $\pm$ 4.15 & 73.83 $\pm$ 4.12 & 74.78 $\pm$ 4.25 & 76.67 $\pm$ 3.46 & 73.35 $\pm$ 4.04 \\
& \textsc{ng} & 77.61 $\pm$ 2.12 & \textbf{82.47 $\pm$ 2.85} & 80.83 $\pm$ 2.88 & 76.45 $\pm$ 2.66 & 77.73 $\pm$ 2.74 & 79.11 $\pm$ 2.16 & 75.97 $\pm$ 2.63  \\
\bottomrule
\end{tabular}
}
\label{tab:agnostic-ood-a}
\end{table}

\begin{table}[]
\caption{OOD detection performance measured by AUROC (\%) with standard deviation of the open-box measures. Measures that outperform all of the white-box/black-box methods are in \textbf{bold}. Higher is better.}
\small
\centering
\scalebox{0.8}{
\begin{tabular}{llccc}
\toprule
Model & Dataset & ENSM & TEMP & MD \\
\midrule
\multirow{8}{*}{RoBERTa} & \textsc{sst} & 69.03 $\pm$ 1.07 & 71.64 $\pm$ 1.60 & \textbf{85.36 $\pm$ 3.42}\\ 
& \textsc{subj} & \textbf{76.68 $\pm$ 1.91} & 74.41 $\pm$ 8.60 & \textbf{93.47 $\pm$ 0.66}\\
& \textsc{agn} & \textbf{80.35 $\pm$ 1.15} & 75.38 $\pm$ 2.93 & \textbf{82.63 $\pm$ 3.19} \\
& \textsc{trec} & \textbf{96.87 $\pm$ 0.54} & \textbf{96.74 $\pm$ 0.84} & \textbf{96.74 $\pm$ 2.01} \\ 
& \textsc{bp} & 79.39 $\pm$ 1.07 & 86.01 $\pm$ 1.02 & \textbf{97.35 $\pm$ 1.12} \\
& \textsc{ar} & 92.44 $\pm$ 0.26 & 92.25 $\pm$ 0.85 & \textbf{98.35 $\pm$ 0.26} \\
& \textsc{mg} & 76.98 $\pm$ 1.54 & 84.3 $\pm$ 3.2 & \textbf{95.12 $\pm$ 1.68} \\
& \textsc{ng} & 80.77 $\pm$ 1.2 & 82.87 $\pm$ 0.49 & \textbf{90.68 $\pm$ 0.95} \\
\midrule
\multirow{8}{*}{ELECTRA} & \textsc{sst} & 73.81 $\pm$ 1.18 & 73.58 $\pm$ 1.98 & \textbf{78.85 $\pm$ 1.48} \\
& \textsc{subj} & \textbf{79.23 $\pm$ 3.17} & \textbf{78.20 $\pm$ 7.87} & \textbf{81.59 $\pm$ 8.16} \\
& \textsc{agn} & 79.5 $\pm$ 2.03 & 78.31 $\pm$ 3.59 & \textbf{86.1 $\pm$ 1.85} \\
& \textsc{trec} & 97.55 $\pm$ 0.37 & 98.2 $\pm$ 0.57 & 97.54 $\pm$ 1.18\\
& \textsc{bp} & 84.2 $\pm$ 1.68 & 84.69 $\pm$ 2.12 & \textbf{98.28 $\pm$ 0.46} \\
& \textsc{ar} & \textbf{91.98 $\pm$ 1.08} & 90.35 $\pm$ 1.7 & \textbf{95.47 $\pm$ 0.83} \\
& \textsc{mg} & 76.86 $\pm$ 1.08 & 78.47 $\pm$ 4.66 & \textbf{92.96 $\pm$ 3.67} \\
& \textsc{ng} & 79.93 $\pm$ 0.83 & 80.75 $\pm$ 2.69 & \textbf{89.13 $\pm$ 0.86} \\
\bottomrule
\end{tabular}
}
\label{tab:non-agnostic-ood-a}
\end{table}

\begin{table}[]
\caption{OOD detection performance measured by AUROC (\%) with standard deviation of the simplified datasets. The best-performing measure is in \textbf{bold}. Higher is better.}
\small
\centering
\scalebox{0.8}{
\begin{tabular}{llcccccc} 
\toprule
Model & Dataset & \ourmethodL{} & MSP & ENT & EGY & MC & GRAD \\
\midrule
\multirow{3}{*}{RoBERTa} & \textsc{bp2} & 79.66 $\pm$ 10.03 & 89.74 $\pm$ 1.20 & 89.74 $\pm$ 1.20 & 88.23 $\pm$ 0.72 & 88.92 $\pm$ 1.33 & \textbf{89.84 $\pm$ 1.13} \\
& \textsc{ar2} & 88.20 $\pm$ 5.04 & 93.33 $\pm$ 3.98 & 93.33 $\pm$ 3.98 & \textbf{94.27 $\pm$ 1.00} & 93.30 $\pm$ 2.72 & 93.58 $\pm$ 3.75 \\
& \textsc{mg2}  & 84.78 $\pm$ 14.37 & 78.13 $\pm$ 6.36 & 78.13 $\pm$ 6.36 & \textbf{85.44 $\pm$ 3.25} & 82.62 $\pm$ 4.25 & 78.28 $\pm$ 6.37 \\
\midrule
\multirow{3}{*}{ELECTRA} & \textsc{bp2} & 71.71 $\pm$ 16.41 & \textbf{93.23 $\pm$ 2.00} & \textbf{93.23 $\pm$ 2.00} & 92.51 $\pm$ 1.45 & 92.61 $\pm$ 1.44 & 93.20 $\pm$ 2.02 \\
& \textsc{ar2}  & 90.67 $\pm$ 5.55 & \textbf{96.16 $\pm$ 0.89} & \textbf{96.16 $\pm$ 0.89} & 93.80 $\pm$ 3.80 & 95.47 $\pm$ 1.01 & 96.14 $\pm$ 0.90 \\
& \textsc{mg2}  & \textbf{91.41 $\pm$ 1.45} & 88.02 $\pm$ 2.15 & 88.02 $\pm$ 2.15 & 85.08 $\pm$ 6.62 & 88.55 $\pm$ 1.52 & 88.10 $\pm$ 2.11 \\
\bottomrule
\end{tabular}
}
\label{tab:agnostic-ood-augmented-a}
\end{table}

\begin{table}[]
\caption{OOD detection performance measured by AUROC (\%) with standard deviation of the open-box measures for the simplified dataset is shown left of the vertical line. Measures that outperform all white-box/black-box methods are in \textbf{bold}. Higher is better. Effect size of changes in representations between ID and OOD data for a simplified datasets using CLES (\%) is shown right of the vertical line.}
\small
\centering
\scalebox{0.8}{
\begin{tabular}{llccc|cc}
\toprule
Model & Dataset & ENSM & TEMP & MD & Mean & Last \\
\midrule
\multirow{3}{*}{RoBERTa} & \textsc{bp2} & 87.59 $\pm$ 1.30 & \textbf{89.92 $\pm$ 1.18} & \textbf{97.66 $\pm$ 0.61} & 94.57 $\pm$ 6.33 & 84.27 $\pm$ 3.85 \\
& \textsc{ar2} & \textbf{94.55 $\pm$ 0.69} & 93.34 $\pm$ 4.03 & \textbf{99.02 $\pm$ 0.29} & 91.84 $\pm$ 4.70 & 80.47 $\pm$ 6.56 \\
& \textsc{mg2} & 83.95 $\pm$ 1.79 & 78.23 $\pm$ 6.29 & \textbf{97.48 $\pm$ 0.83} & 86.8 $\pm$ 11.25 & 70.25 $\pm$ 10.69 \\
\midrule
\multirow{3}{*}{ELECTRA} & 91.25 $\pm$ 1.10 & \textbf{93.34 $\pm$ 2.01} & \textbf{98.75 $\pm$ 0.3}3 & 97.28 $\pm$ 0.92 & 94.87 $\pm$ 1.21 \\
& \textsc{ar2} & 95.20 $\pm$ 1.04 & \textbf{96.20 $\pm$ 0.91} & 93.22 $\pm$ 6.26 & 97.07 $\pm$ 1.01 & 96.22 $\pm$ 2.63 \\
& \textsc{mg2} & 84.12 $\pm$ 2.71 & 87.95 $\pm$ 2.09 & \textbf{98.28 $\pm$ 0.34} & 88.28 $\pm$ 1.09 & 87.1 $\pm$ 2.65 \\
\bottomrule
\end{tabular}
}
\label{tab:non-agnostic-ood-augmented-a}
\end{table}

\begin{table}[]
\caption{Near-OOD detection performance measured by AUROC (\%) with standard deviation of the white-box/black-box methods. The best-performing measure is in \textbf{bold}. Higher is better.}
\small
\centering
\scalebox{0.8}{
\begin{tabular}{cccccccc} 
\toprule
Model & Shift & \ourmethodL{} & MSP & ENT & EGY & MC & GRAD \\
\midrule
\multirow{2}{*}{RoBERTa} & semantic & 61.61 $\pm$ 2.61 & 69.46 $\pm$ 0.83 & 69.41 $\pm$ 0.99 & \textbf{69.50 $\pm$ 0.86} & 68.34 $\pm$ 0.56 & 69.36 $\pm$ 0.91 \\
& background & \textbf{62.7 $\pm$ 5.75} & 54.26 $\pm$ 2.49 & 50.17 $\pm$ 4.82 & 54.26 $\pm$ 2.49 & 48.18 $\pm$ 2.44 & 54.33 $\pm$ 2.5  \\
\midrule
\multirow{2}{*}{ELECTRA} & semantic & 62.49 $\pm$ 3.81 & 63.17 $\pm$ 2.65 & 60.92 $\pm$ 3.70 & 63.12 $\pm$ 2.69 & 62.14 $\pm$ 2.26 & \textbf{63.23 $\pm$ 2.63} \\
& background & \textbf{59.35 $\pm$ 3.19} & 42.96 $\pm$ 2.92 & 38.68 $\pm$ 2.25 & 42.96 $\pm$ 2.92 & 37.96 $\pm$ 3.47 & 42.77 $\pm$ 2.92 \\
\bottomrule
\end{tabular}
}
\label{tab:agnostic-ood-shifted-a}
\end{table}

\begin{table}[]
\caption{Near-OOD detection performance measured by AUROC (\%) with standard deviation of the open-box measures for the augmented dataset is shown left of the vertical line. Measures that outperform all white-box/black-box methods are in \textbf{bold}. Higher is better.}
\small
\centering
\scalebox{0.8}{
\begin{tabular}{ccccc}
\toprule
Model & Shift & ENSM & TEMP & MD \\
\midrule
\multirow{2}{*}{RoBERTa} & semantic & 68.91 $\pm$ 1.12 & \textbf{70.56 $\pm$ 1.25} & \textbf{72.03 $\pm$ 0.89} \\
& background & 49.13 $\pm$ 2.5 & 54.19 $\pm$ 2.57  & 59.4 $\pm$ 10.18 \\
\midrule
\multirow{2}{*}{ELECTRA} & semantic & \textbf{65.67 $\pm$ 0.45} & 62.45 $\pm$ 3.17 & \textbf{64.22 $\pm$ 2.75} \\
& background & 41.25 $\pm$ 2.47 & 42.63 $\pm$ 2.84 & 39.31 $\pm$ 4.93 \\
\bottomrule
\end{tabular}
}
\label{tab:non-agnostic-ood-shifted-a}
\end{table}

\begin{table}[]
\caption{OOD detection performance measured by AUPR-IN (\%) with standard deviation of the white-box/black-box methods. The best-performing measure is in \textbf{bold}. Higher is better.}
\small
\centering
\scalebox{0.8}{
\begin{tabular}{llccccccc} 
\toprule
Model & Dataset & \ourmethodmean{} & \ourmethodL{} & MSP & ENT & EGY & MC & GRAD \\
\midrule
\multirow{8}{*}{RoBERTa} & \textsc{sst} & 50.72 $\pm$ 5.49 & \textbf{72.68 $\pm$ 9.50} & 71.13 $\pm$ 1.35 & 71.49 $\pm$ 2.49 & 71.49 $\pm$ 2.49 & 70.35 $\pm$ 2.04 & 71.59 $\pm$ 2.44 \\ 
& \textsc{subj} & 54.60 $\pm$ 11.38 & 73.45 $\pm$ 8.69 & 73.68 $\pm$ 11.08 & 76.18 $\pm$ 7.48 & 76.18 $\pm$ 7.47 & 75.85 $\pm$ 6.76 & \textbf{76.66 $\pm$ 7.54} \\
& \textsc{agn} & \textbf{76.48 $\pm$ 5.49} & 59.48 $\pm$ 8.20 & 72.73 $\pm$ 3.70 & 71.42 $\pm$ 3.78 & 71.47 $\pm$ 3.78 & 74.64 $\pm$ 3.93 & 71.26 $\pm$ 3.64 \\
& \textsc{trec} & 67.70 $\pm$ 9.83 & 94.98 $\pm$ 3.18 & 96.31 $\pm$ 0.76 & 96.67 $\pm$ 0.88 & \textbf{96.78 $\pm$ 0.85} & 96.29 $\pm$ 0.83 & 96.64 $\pm$ 0.87 \\ 
& \textsc{bp} & 86.80 $\pm$ 3.01 & \textbf{89.98 $\pm$ 2.83} & 84.50 $\pm$ 0.76 & 69.72 $\pm$ 0.88 & 71.06 $\pm$ 0.98 & 72.58 $\pm$ 0.85 & 73.29 $\pm$ 1.15 \\
& \textsc{ar} & 91.30 $\pm$ 1.92 & \textbf{93.18 $\pm$ 1.23} & 92.32 $\pm$ 0.84 & 89.95 $\pm$ 1.17 & 90.35 $\pm$ 1.19 & 91.25 $\pm$ 1.08 & 89.75 $\pm$ 1.16\\
& \textsc{mg} & \textbf{87.86 $\pm$ 1.94} & 84.18 $\pm$ 4.69 & 86.31 $\pm$ 2.93 & 77.69 $\pm$ 2.57 & 78.36 $\pm$ 2.65 & 81.04 $\pm$ 2.26 & 77.27 $\pm$ 2.55 \\
& \textsc{ng} & \textbf{84.70 $\pm$ 0.99} & 72.97 $\pm$ 4.44 & 82.59 $\pm$ 0.99 & 78.74 $\pm$ 0.80 & 79.29 $\pm$ 0.84 & 80.11 $\pm$ 0.92 & 78.43 $\pm$ 0.81\\
\midrule
\multirow{8}{*}{ELECTRA} & \textsc{sst} & 73.50 $\pm$ 3.25 & \textbf{79.28 $\pm$ 1.66} & 69.72 $\pm$ 1.92 & 73.75 $\pm$ 1.97 & 73.75 $\pm$ 1.97 & 72.15 $\pm$ 2.11 & 73.69 $\pm$ 2.03\\
& \textsc{subj} & 73.62 $\pm$ 11.47 & 78.02 $\pm$ 11.63 & 70.99 $\pm$ 7.04 & 79.86 $\pm$ 6.71 & 79.86 $\pm$ 6.71 & \textbf{79.93 $\pm$ 7.30} & 79.78 $\pm$ 6.84 \\
& \textsc{agn} & 64.00 $\pm$ 5.74 & \textbf{78.77 $\pm$ 3.04} & 77.13 $\pm$ 4.37 & 75.63 $\pm$ 3.46 & 75.71 $\pm$ 3.48 & 78.00 $\pm$ 2.98 & 75.37 $\pm$ 3.46 \\
& \textsc{trec} & 97.39 $\pm$ 0.84 & \textbf{98.89 $\pm$ 0.40} & 97.34 $\pm$ 0.62 & 97.78 $\pm$ 0.75 & 97.94 $\pm$ 0.70 & 96.73 $\pm$ 0.74 & 97.66 $\pm$ 0.78 \\
& \textsc{bp} & 87.87 $\pm$ 1.89 & \textbf{96.54 $\pm$ 1.41} & 82.67 $\pm$ 3.05 & 79.34 $\pm$ 3.26 & 80.97 $\pm$ 3.31 & 82.81 $\pm$ 2.66 & 77.16 $\pm$ 3.53 \\
& \textsc{ar} & 86.27 $\pm$ 2.79 & \textbf{91.95 $\pm$ 3.04} & 90.81 $\pm$ 1.38 & 88.69 $\pm$ 1.38 & 89.05 $\pm$ 1.43 & 89.73 $\pm$ 1.57 & 88.53 $\pm$ 1.31 \\
& \textsc{mg} & 81.65 $\pm$ 6.40 & \textbf{91.14 $\pm$ 3.30} & 79.82 $\pm$ 4.52 & 75.49 $\pm$ 4.64 & 75.91 $\pm$ 4.68 & 77.99 $\pm$ 3.63 & 75.21 $\pm$ 4.58 \\
& \textsc{ng} & 78.90 $\pm$ 3.08 & \textbf{82.37 $\pm$ 5.11} & 79.20 $\pm$ 4.28 & 76.85 $\pm$ 3.50 & 77.42 $\pm$ 3.63 & 79.08 $\pm$ 3.10 & 76.60 $\pm$ 3.53 \\
\bottomrule
\end{tabular}
}
\label{tab:agnostic-ood-aupr}
\end{table}

\begin{table}[]
\caption{OOD detection performance measured by AUPR-IN (\%) with standard deviation of the open-box measures. Measures that outperform all of the white-box/black-box methods are in \textbf{bold}. Higher is better.}
\small
\centering
\scalebox{0.8}{
\begin{tabular}{llccc}
\toprule
Model & Dataset & ENSM & TEMP & MD \\
\midrule
\multirow{8}{*}{RoBERTa} & \textsc{sst} & 69.91 $\pm$ 1.79 & 66.91 $\pm$ 3.13 & \textbf{85.77 $\pm$ 3.49} \\ 
& \textsc{subj} & \textbf{78.51 $\pm$ 3.20} & 72.32 $\pm$ 8.40 & \textbf{93.53 $\pm$ 0.75} \\
& \textsc{agn} & \textbf{77.45 $\pm$ 1.62} & 67.88 $\pm$ 3.91 & \textbf{79.93 $\pm$ 3.90} \\
& \textsc{trec} & \textbf{97.28 $\pm$ 0.46} & 96.39 $\pm$ 1.12 & \textbf{96.91 $\pm$ 1.85} \\ 
& \textsc{bp} & 77.44 $\pm$ 1.01 & 80.61 $\pm$ 1.35 & \textbf{96.93 $\pm$ 1.28} \\
& \textsc{ar} & 92.73 $\pm$ 0.34 & 90.37 $\pm$ 1.22 & \textbf{98.24 $\pm$ 0.27} \\
& \textsc{mg} & 77.62 $\pm$ 1.81 & 81.20 $\pm$ 3.85 & \textbf{95.14 $\pm$ 1.51} \\
& \textsc{ng} & 80.83 $\pm$ 1.24 & 79.15 $\pm$ 0.94 & \textbf{91.54 $\pm$ 0.84} \\
\midrule
\multirow{8}{*}{ELECTRA} & \textsc{sst} & 75.15 $\pm$ 1.40 & 69.21 $\pm$ 2.20 & \textbf{81.24 $\pm$ 0.94} \\
& \textsc{subj} & \textbf{81.71 $\pm$ 3.41} & 76.71 $\pm$ 7.49 & \textbf{82.03 $\pm$ 6.97} \\
& \textsc{agn} & 77.76 $\pm$ 3.20 & 72.91 $\pm$ 3.95 & \textbf{83.90 $\pm$ 3.48} \\
& \textsc{trec} & 97.93 $\pm$ 0.24 & 98.02 $\pm$ 0.66 & 97.60 $\pm$ 1.12 \\
& \textsc{bp} & 82.91 $\pm$ 2.00 & 79.52 $\pm$ 3.31 & \textbf{98.27 $\pm$ 0.43} \\
& \textsc{ar} & \textbf{92.42 $\pm$ 0.98} & 88.61 $\pm$ 1.66 & \textbf{96.09 $\pm$ 0.72} \\
& \textsc{mg} & 78.34 $\pm$ 1.32 & 74.60 $\pm$ 5.69 & \textbf{93.67 $\pm$ 3.27} \\
& \textsc{ng} & 80.01 $\pm$ 1.20 & 75.91 $\pm$ 4.15 & \textbf{90.97 $\pm$ 0.80} \\
\bottomrule
\end{tabular}
}
\label{tab:non-agnostic-ood-aupr}
\end{table}

\begin{table}[]
\caption{OOD detection performance measured by AUPR-IN (\%) with standard deviation of the simplified datasets. The best-performing measure is in \textbf{bold}. Higher is better.}
\small
\centering
\scalebox{0.8}{
\begin{tabular}{llcccccc} 
\toprule
Model & Dataset & \ourmethodL{} & MSP & ENT & EGY & MC & GRAD \\
\midrule
\multirow{3}{*}{RoBERTa} & \textsc{bp2} & 78.96 $\pm$ 10.76 & 90.37 $\pm$ 0.95 & 90.37 $\pm$ 0.95 & 89.03 $\pm$ 1.35 & \textbf{90.55 $\pm$ 0.96} & 90.50 $\pm$ 0.88 \\
& \textsc{ar2} & 86.43 $\pm$ 6.29 & 92.81 $\pm$ 5.97 & 92.81 $\pm$ 5.97 & \textbf{94.68 $\pm$ 0.71} & 93.82 $\pm$ 3.79 & 93.24 $\pm$ 5.32 \\
& \textsc{mg2} & 82.94 $\pm$ 15.31 & 77.55 $\pm$ 7.63 & 77.55 $\pm$ 7.63 & \textbf{85.45 $\pm$ 3.94} & 83.11 $\pm$ 4.93 & 77.84 $\pm$ 7.45 \\
\midrule
\multirow{3}{*}{ELECTRA} & \textsc{bp2} & 72.95 $\pm$ 15.63 & 93.22 $\pm$ 3.21 & 93.22 $\pm$ 3.21 & 93.05 $\pm$ 1.83 & \textbf{93.42 $\pm$ 2.26} & 93.15 $\pm$ 3.27 \\
& \textsc{ar2} & 88.97 $\pm$ 6.73 & \textbf{96.34 $\pm$ 1.20} & \textbf{96.34 $\pm$ 1.20} & 94.32 $\pm$ 3.57 & 96.20 $\pm$ 1.06 & 96.30 $\pm$ 1.25 \\
& \textsc{mg2} & 91.12 $\pm$ 1.00 & 90.34 $\pm$ 1.33 & 90.34 $\pm$ 1.33 & 86.75 $\pm$ 6.33 & \textbf{91.22 $\pm$ 0.80} & 90.40 $\pm$ 1.30 \\
\bottomrule
\end{tabular}
}
\label{tab:agnostic-ood-augmented-aupr}
\end{table}

\begin{table}[]
\caption{OOD detection performance measured by AUPR-IN (\%) with standard deviation of the open-box measures for the simplified dataset is shown left of the vertical line. Measures that outperform all white-box/black-box methods are in \textbf{bold}. Higher is better.}
\small
\centering
\scalebox{0.8}{
\begin{tabular}{llccc}
\toprule
Model & Dataset & ENSM & TEMP & MD \\
\midrule
\multirow{3}{*}{RoBERTa} & \textsc{bp2} & 89.20 $\pm$ 1.39 & 88.70 $\pm$ 1.20 & \textbf{97.94 $\pm$ 0.58} \\
& \textsc{ar2} & \textbf{95.79 $\pm$ 0.68} & 91.42 $\pm$ 7.25 & \textbf{99.10 $\pm$ 0.23} \\
& \textsc{mg2} & \textbf{84.69 $\pm$ 3.85} & 74.03 $\pm$ 8.73 & \textbf{97.73 $\pm$ 0.78} \\
\midrule
\multirow{3}{*}{ELECTRA} & \textsc{bp2} & 93.03 $\pm$ 2.09 & 92.01 $\pm$ 3.83 & \textbf{98.81 $\pm$ 0.29} \\
& \textsc{ar2} & \textbf{96.45 $\pm$ 0.79} & 95.64 $\pm$ 1.41 & 87.32 $\pm$ 12.53 \\
& \textsc{mg2} & 87.43 $\pm$ 2.78 & 88.64 $\pm$ 1.47 & \textbf{98.49 $\pm$ 0.30}  \\
\bottomrule
\end{tabular}
}
\label{tab:non-agnostic-ood-augmented-aupr}
\end{table}

\begin{table}[]
\caption{Near-OOD detection performance measured by AUPR-IN (\%) with standard deviation of the white-box/black-box methods. The best-performing measure is in \textbf{bold}. Higher is better.}
\small
\centering
\scalebox{0.8}{
\begin{tabular}{cccccccc} 
\toprule
Model & Shift & \ourmethodL{} & MSP & ENT & EGY & MC & GRAD \\
\midrule
\multirow{2}{*}{RoBERTa} & semantic & 61.15 $\pm$ 3.20 & \textbf{70.56 $\pm$ 2.81} & 70.36 $\pm$ 2.87 & 70.00 $\pm$ 1.29 & 69.99 $\pm$ 2.58 & 70.42 $\pm$ 2.88 \\
& background & \textbf{62.62 $\pm$ 5.66} & 56.61 $\pm$ 4.48 & 56.61 $\pm$ 4.48 & 51.12 $\pm$ 5.68 & 52.43 $\pm$ 4.43 & 56.68 $\pm$ 4.46 \\
\midrule
\multirow{2}{*}{ELECTRA} & semantic & 61.28 $\pm$ 5.53 & 61.85 $\pm$ 2.25 & 61.76 $\pm$ 2.32 & 58.57 $\pm$ 3.85 & 60.62 $\pm$ 1.85 & \textbf{61.97 $\pm$ 2.20} \\
& background & \textbf{59.40 $\pm$ 3.31} & 45.36 $\pm$ 2.09 & 45.36 $\pm$ 2.09 & 42.00 $\pm$ 1.11 & 41.66 $\pm$ 2.21 & 45.26 $\pm$ 2.09 \\
\bottomrule
\end{tabular}
}
\label{tab:agnostic-ood-shifted-aupr}
\end{table}

\begin{table}[]
\caption{Near-OOD detection performance measured by AUPR-IN (\%) with standard deviation of the open-box measures for the augmented dataset is shown left of the vertical line. Measures that outperform all white-box/black-box methods are in \textbf{bold}. Higher is better.}
\small
\centering
\scalebox{0.8}{
\begin{tabular}{ccccc}
\toprule
Model & Shift & ENSM & TEMP & MD \\
\midrule
\multirow{2}{*}{RoBERTa} & semantic & 69.50 $\pm$ 2.84 & 67.13 $\pm$ 3.45 & \textbf{73.76 $\pm$ 1.27} \\
& background & 52.33 $\pm$ 3.19 & 51.30 $\pm$ 5.11 & \textbf{65.79 $\pm$ 8.60} \\
\midrule
\multirow{2}{*}{ELECTRA} & semantic & \textbf{63.49 $\pm$ 1.69} & 55.54 $\pm$ 3.31 & \textbf{64.94 $\pm$ 4.82} \\
& background & 43.50 $\pm$ 1.38 & 39.75 $\pm$ 2.00 & 46.59 $\pm$ 4.32 \\
\bottomrule
\end{tabular}
}
\label{tab:non-agnostic-ood-shifted-aupr}
\end{table}

\begin{table}[]
\caption{OOD detection performance measured by FPR@95TPR (\%) with standard deviation of the white-box/black-box methods. The best-performing measure is in \textbf{bold}. Lower is better.}
\small
\centering
\scalebox{0.8}{
\begin{tabular}{llccccccc} 
\toprule
Model & Dataset & \ourmethodmean{} & \ourmethodL{} & MSP & ENT & EGY & MC & GRAD \\
\midrule
\multirow{8}{*}{RoBERTa} & \textsc{sst} & 94.11 $\pm$ 3.57 & \textbf{79.31 $\pm$ 10.96} & 86.92 $\pm$ 0.51 & 86.92 $\pm$ 0.51 & 86.03 $\pm$ 2.29 & 90.90 $\pm$ 1.01 & 87.35 $\pm$ 0.76 \\ 
& \textsc{subj} & 92.65 $\pm$ 5.85 & 75.88 $\pm$ 10.04 & 80.65 $\pm$ 6.14 & 80.65 $\pm$ 6.14 & \textbf{75.74 $\pm$ 5.77} & 85.40 $\pm$ 3.52 & 80.81 $\pm$ 6.08 \\
& \textsc{agn} & 70.76 $\pm$ 10.35 & 83.21 $\pm$ 8.49 & 77.45 $\pm$ 3.08 & 76.16 $\pm$ 3.09 & \textbf{63.96 $\pm$ 5.94} & 68.23 $\pm$ 3.90 & 77.52 $\pm$ 2.94 \\
& \textsc{trec} & 77.16 $\pm$ 13.86 & 21.96 $\pm$ 11.91 & 19.72 $\pm$ 5.79 & 18.84 $\pm$ 5.82 & \textbf{18.52 $\pm$ 3.56} & 28.20 $\pm$ 5.40 & 20.08 $\pm$ 5.38 \\ 
& \textsc{bp} & 49.75 $\pm$ 8.85 & \textbf{46.77 $\pm$ 14.68} & 80.81 $\pm$ 1.77 & 71.10 $\pm$ 2.79 & 51.05 $\pm$ 5.75 & 68.82 $\pm$ 2.80 & 83.60 $\pm$ 1.85 \\
& \textsc{ar} & 38.04 $\pm$ 7.21 & \textbf{31.80 $\pm$ 5.90} & 59.68 $\pm$ 3.95 & 49.18 $\pm$ 3.13 & 36.88 $\pm$ 3.59 & 47.31 $\pm$ 3.19 & 62.43 $\pm$ 3.52 \\
& \textsc{mg} & \textbf{47.51 $\pm$ 5.60} & 52.45 $\pm$ 8.99 & 81.20 $\pm$ 2.20 & 71.80 $\pm$ 3.80 & 50.76 $\pm$ 7.53 & 65.02 $\pm$ 4.57 & 84.63 $\pm$ 1.69 \\
& \textsc{ng} & \textbf{67.91 $\pm$ 3.33} & 83.79 $\pm$ 5.22 & 78.91 $\pm$ 1.31 & 72.84 $\pm$ 0.99 & 70.04 $\pm$ 2.40 & 73.91 $\pm$ 1.99 & 81.92 $\pm$ 1.42 \\
\midrule
\multirow{8}{*}{ELECTRA} & \textsc{sst} & 79.75 $\pm$ 3.94 & \textbf{77.66 $\pm$ 4.07} & 84.10 $\pm$ 0.79 & 84.10 $\pm$ 0.79 & 83.97 $\pm$ 2.57 & 89.30 $\pm$ 0.98 & 84.09 $\pm$ 0.93 \\
& \textsc{subj} & 77.48 $\pm$ 10.00 & \textbf{75.27 $\pm$ 11.62} & 79.67 $\pm$ 9.49 & 79.67 $\pm$ 9.49 & 82.89 $\pm$ 4.86 & 82.80 $\pm$ 6.09 & 79.70 $\pm$ 9.40 \\
& \textsc{agn} & 85.07 $\pm$ 11.95 & \textbf{63.68 $\pm$ 15.00} & 77.97 $\pm$ 4.64 & 76.93 $\pm$ 5.28 & 62.94 $\pm$ 9.51 & 71.32 $\pm$ 5.52 & 78.48 $\pm$ 4.46 \\
& \textsc{trec} & 11.20 $\pm$ 4.06 & \textbf{3.84 $\pm$ 1.65} & 12.28 $\pm$ 5.22 & 11.40 $\pm$ 4.81 & 12.28 $\pm$ 3.75 & 24.64 $\pm$ 7.74 & 13.24 $\pm$ 5.70 \\
& \textsc{bp} & 73.77 $\pm$ 12.84 & \textbf{14.86 $\pm$ 6.85} & 76.73 $\pm$ 3.31 & 62.47 $\pm$ 4.02 & 53.56 $\pm$ 3.40 & 59.49 $\pm$ 6.27 & 79.19 $\pm$ 4.12 \\
& \textsc{ar} & 73.06 $\pm$ 14.25 & \textbf{42.24 $\pm$ 7.35} & 64.30 $\pm$ 4.51 & 57.35 $\pm$ 5.90 & 45.68 $\pm$ 7.24 & 56.56 $\pm$ 7.00 & 85.25 $\pm$ 2.35 \\
& \textsc{mg} & 75.77 $\pm$ 13.56 & \textbf{43.06 $\pm$ 13.99} & 82.90 $\pm$ 2.70 & 75.95 $\pm$ 4.32 & 62.25 $\pm$ 6.70 & 71.83 $\pm$ 4.26 & 65.19 $\pm$ 4.16 \\
& \textsc{ng} & 86.26 $\pm$ 4.66 & 76.89 $\pm$ 4.50 & 78.84 $\pm$ 3.79 & 73.33 $\pm$ 4.38 & 70.89 $\pm$ 4.64 & \textbf{72.52 $\pm$ 3.5} & 81.76 $\pm$ 3.16 \\
\bottomrule
\end{tabular}
}
\label{tab:agnostic-ood-fpr}
\end{table}

\begin{table}[]
\caption{OOD detection performance measured by FPR@95TPR (\%) with standard deviation of the open-box measures. Measures that outperform all of the white-box/black-box methods are in \textbf{bold}. Lower is better.}
\small
\centering
\scalebox{0.8}{
\begin{tabular}{llccc}
\toprule
Model & Dataset & ENSM & TEMP & MD \\
\midrule
\multirow{8}{*}{RoBERTa} & \textsc{sst} & 90.97 $\pm$ 0.45 & 86.95 $\pm$ 0.37 & \textbf{59.63 $\pm$ 10.21} \\ 
& \textsc{subj} & 87.23 $\pm$ 1.52 & 80.49 $\pm$ 5.95 & \textbf{32.60 $\pm$ 3.68} \\
& \textsc{agn} & \textbf{62.09 $\pm$ 1.84} & \textbf{67.71 $\pm$ 4.12} & \textbf{55.12 $\pm$ 6.12} \\
& \textsc{trec} & 20.40 $\pm$ 5.49 & \textbf{18.00 $\pm$ 4.53} & \textbf{16.04 $\pm$ 10.83} \\ 
& \textsc{bp} & 62.59 $\pm$ 2.78 & \textbf{46.33 $\pm$ 3.58} & \textbf{11.33 $\pm$ 4.96} \\
& \textsc{ar} & 37.70 $\pm$ 1.82 & 36.06 $\pm$ 3.52 & \textbf{7.74 $\pm$ 1.67} \\
& \textsc{mg} & 68.95 $\pm$ 1.85 & 53.77 $\pm$ 7.40 & \textbf{23.20 $\pm$ 7.71} \\
& \textsc{ng} & 70.41 $\pm$ 1.44 & \textbf{67.54 $\pm$ 1.43} & \textbf{49.95 $\pm$ 5.09} \\
\midrule
\multirow{8}{*}{ELECTRA} & \textsc{sst} & 88.62 $\pm$ 1.95 & 84.21 $\pm$ 0.59 & 84.45 $\pm$ 4.86 \\
& \textsc{subj} & 83.87 $\pm$ 2.13 & 79.17 $\pm$ 10.15 & \textbf{68.15 $\pm$ 13.19} \\
& \textsc{agn} & 68.27 $\pm$ 2.85 & 70.61 $\pm$ 8.14 & \textbf{46.47 $\pm$ 4.98} \\
& \textsc{trec} & 18.48 $\pm$ 4.99 & 8.60 $\pm$ 2.74 & 9.12 $\pm$ 5.98 \\
& \textsc{bp} & 53.33 $\pm$ 3.32 & 52.09 $\pm$ 4.81 & \textbf{8.31 $\pm$ 2.46} \\
& \textsc{ar} & \textbf{41.31 $\pm$ 5.67} & 46.08 $\pm$ 6.96 & \textbf{27.52 $\pm$ 8.00} \\
& \textsc{mg} & 71.18 $\pm$ 1.23 & 66.34 $\pm$ 6.76 & \textbf{34.45 $\pm$ 14.27} \\
& \textsc{ng} & \textbf{72.27 $\pm$ 1.35} & \textbf{69.58 $\pm$ 5.00} & \textbf{62.98 $\pm$ 4.75} \\
\bottomrule
\end{tabular}
}
\label{tab:non-agnostic-ood-fpr}
\end{table}

\begin{table}[]
\caption{OOD detection performance measured by FPR@95TPR (\%) with standard deviation of the simplified datasets. The best-performing measure is in \textbf{bold}. Lower is better.}
\small
\centering
\scalebox{0.8}{
\begin{tabular}{llcccccc} 
\toprule
Model & Dataset & \ourmethodL{} & MSP & ENT & EGY & MC & GRAD \\
\midrule
\multirow{3}{*}{RoBERTa} & \textsc{bp2} & 62.28 $\pm$ 17.17 & 55.45 $\pm$ 8.77 & 55.45 $\pm$ 8.77 & 64.56 $\pm$ 5.95 & 66.65 $\pm$ 5.39 & \textbf{55.01 $\pm$ 8.48} \\
& \textsc{ar2} & 39.75 $\pm$ 15.50 & 29.69 $\pm$ 10.78 & 29.69 $\pm$ 10.78 & 30.78 $\pm$ 8.40 & 39.32 $\pm$ 12.32 & \textbf{29.12 $\pm$ 10.95} \\
& \textsc{mg2} & \textbf{41.38 $\pm$ 23.22} & 73.73 $\pm$ 8.16 & 73.73 $\pm$ 8.16 & 67.98 $\pm$ 6.38 & 72.90 $\pm$ 5.66 & 73.30 $\pm$ 8.07 \\
\midrule
\multirow{3}{*}{ELECTRA} & \textsc{bp2} & 65.40 $\pm$ 27.16 & 34.58 $\pm$ 9.20 & 34.58 $\pm$ 9.20 & 44.50 $\pm$ 7.54 & 49.86 $\pm$ 9.84 & \textbf{34.54 $\pm$ 9.16} \\
& \textsc{ar2} & 30.29 $\pm$ 15.97 & \textbf{15.99 $\pm$ 4.57} & \textbf{15.99 $\pm$ 4.57} & 32.97 $\pm$ 26.91 & 26.73 $\pm$ 10.82 & 16.02 $\pm$ 4.55 \\
& \textsc{mg2} & \textbf{36.59 $\pm$ 7.84} & 68.92 $\pm$ 8.81 & 68.92 $\pm$ 8.81 & 66.01 $\pm$ 9.33 & 69.22 $\pm$ 7.05 & 68.40 $\pm$ 8.61 \\
\bottomrule
\end{tabular}
}
\label{tab:agnostic-ood-augmented-fpr}
\end{table}

\begin{table}[]
\caption{OOD detection performance measured by FPR@95TPR (\%) with standard deviation of the open-box measures for the simplified dataset is shown left of the vertical line. Measures that outperform all white-box/black-box methods are in \textbf{bold}. Lower is better.}
\small
\centering
\scalebox{0.8}{
\begin{tabular}{llccc}
\toprule
Model & Dataset & ENSM & TEMP & MD \\
\midrule
\multirow{3}{*}{RoBERTa} & \textsc{bp2} & 70.41 $\pm$ 3.73 & 54.22 $\pm$ 8.01 & \textbf{11.53 $\pm$ 3.33} \\
& \textsc{ar2} & 41.87 $\pm$ 5.09 & \textbf{28.90 $\pm$ 11.07} & \textbf{4.68 $\pm$ 1.62} \\
& \textsc{mg2} & 75.55 $\pm$ 1.77 & 73.57 $\pm$ 8.19 & \textbf{12.83 $\pm$ 5.26} \\
\midrule
\multirow{3}{*}{ELECTRA} & \textsc{bp2} & 63.95 $\pm$ 4.01 & \textbf{33.65 $\pm$ 8.64} & \textbf{6.60 $\pm$ 2.34} \\
& \textsc{ar2} & 34.50 $\pm$ 14.15 & \textbf{15.84 $\pm$ 5.08} & \textbf{13.32 $\pm$ 7.92} \\
& \textsc{mg2} & 76.87 $\pm$ 4.14 & 68.90 $\pm$ 8.82 & \textbf{8.53 $\pm$ 2.86} \\
\bottomrule
\end{tabular}
}
\label{tab:non-agnostic-ood-augmented-fpr}
\end{table}

\begin{table}[]
\caption{Near-OOD detection performance measured by FPR@95TPR (\%) with standard deviation of the white-box/black-box methods. The best-performing measure is in \textbf{bold}. Lower is better.}
\small
\centering
\scalebox{0.8}{
\begin{tabular}{cccccccc} 
\toprule
Model & Shift & \ourmethodL{} & MSP & ENT & EGY & MC & GRAD \\
\midrule
\multirow{2}{*}{RoBERTa} & semantic & 89.79 $\pm$ 1.75 & 90.21 $\pm$ 0.76 & 89.89 $\pm$ 0.88 & \textbf{89.59 $\pm$ 0.90} & 91.50 $\pm$ 0.83 & 90.21 $\pm$ 0.93 \\
& background & \textbf{91.38 $\pm$ 5.83} & 95.32 $\pm$ 1.07 & 95.32 $\pm$ 1.07 & 95.79 $\pm$ 1.38 & 96.79 $\pm$ 0.72 & 95.25 $\pm$ 1.03 \\
\midrule
\multirow{2}{*}{ELECTRA} & semantic & 91.14 $\pm$ 0.92 & 90.99 $\pm$ 0.74 & \textbf{90.98 $\pm$ 0.98} & 91.72 $\pm$ 0.94 & 91.64 $\pm$ 0.90 & \textbf{90.98 $\pm$ 0.47} \\
& background & \textbf{89.38 $\pm$ 2.40} & 96.30 $\pm$ 0.88 & 96.30 $\pm$ 0.88 & 97.12 $\pm$ 0.32 & 97.64 $\pm$ 0.73 & 96.31 $\pm$ 0.78 \\
\bottomrule
\end{tabular}
}
\label{tab:agnostic-ood-shifted-fpr}
\end{table}

\begin{table}[]
\caption{Near-OOD detection performance measured by FPR@95TPR (\%) with standard deviation of the open-box measures for the augmented dataset is shown left of the vertical line. Measures that outperform all white-box/black-box methods are in \textbf{bold}. Lower is better.}
\small
\centering
\scalebox{0.8}{
\begin{tabular}{ccccc}
\toprule
Model & Shift & ENSM & TEMP & MD \\
\midrule
\multirow{2}{*}{RoBERTa} & semantic & 90.39 $\pm$ 0.49 & \textbf{88.77 $\pm$ 1.03} & \textbf{84.49 $\pm$ 1.83} \\
& background & 97.05 $\pm$ 0.57 & 95.34 $\pm$ 1.06 & 95.57 $\pm$ 3.69 \\
\midrule
\multirow{2}{*}{ELECTRA} & semantic & \textbf{90.26 $\pm$ 0.15} & 91.23 $\pm$ 1.26 & 93.37 $\pm$ 0.29 \\
& background & 97.71 $\pm$ 0.33 & 96.27 $\pm$ 0.95 & 98.83 $\pm$ 0.30 \\
\bottomrule
\end{tabular}
}
\label{tab:non-agnostic-ood-shifted-fpr}
\end{table}

% SLIKE!!!!!!!!!!!!!!!!!!!!!!

\begin{figure}[t!]
\small
\centering
\begin{subfigure}{.49\linewidth}
  \centering
  \includegraphics[width=\linewidth]{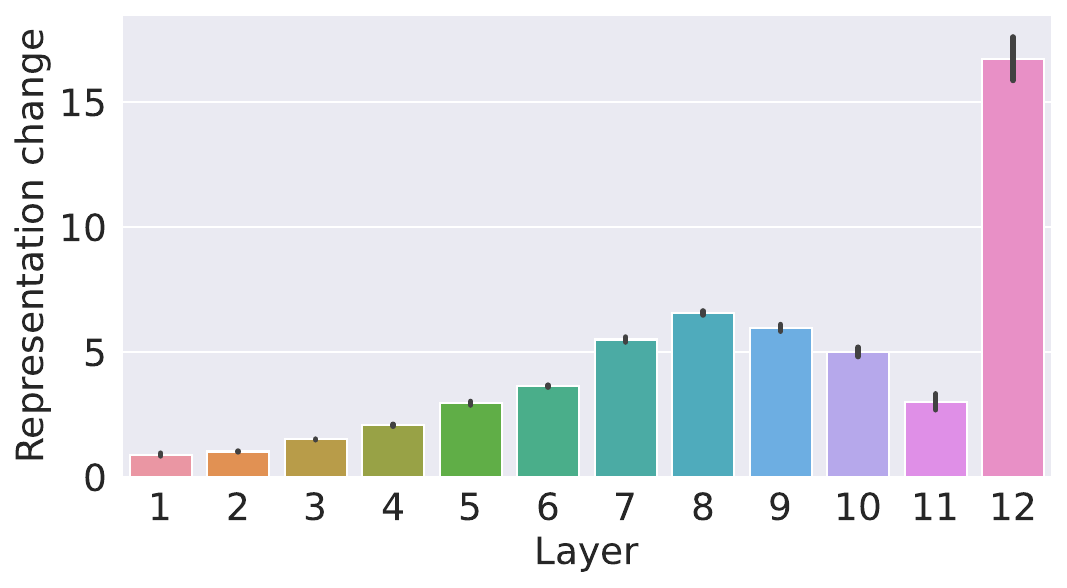}
  \caption{RoBERTa's representation change}
\end{subfigure}
\begin{subfigure}{.49\linewidth}
  \centering
  \includegraphics[width=\linewidth]{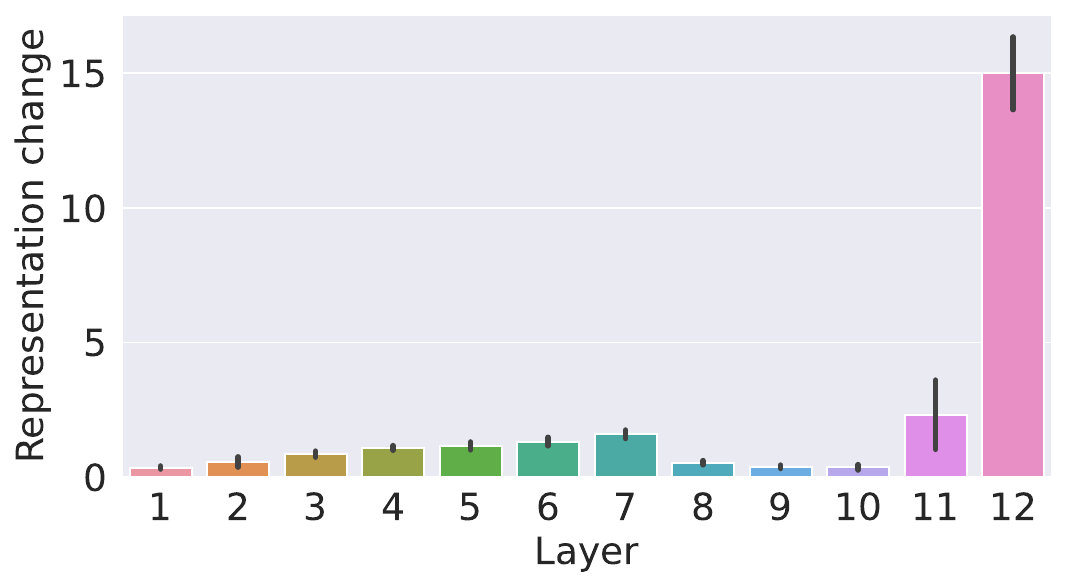}
  \caption{ELECTRA's representation change}
\end{subfigure}
\begin{subfigure}{.49\linewidth}
  \centering
  \includegraphics[width=\linewidth]{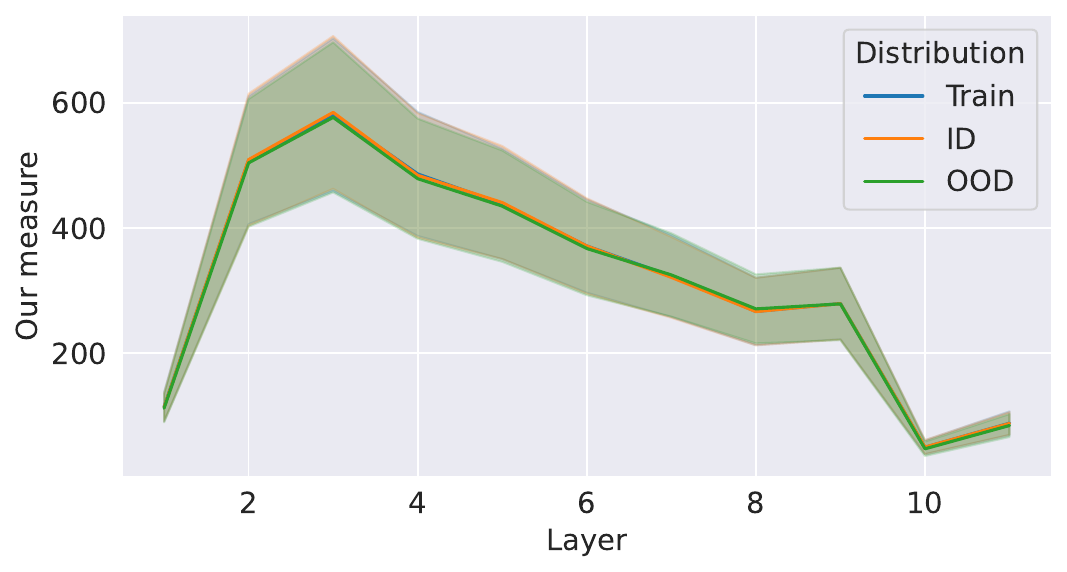}
  \caption{Pre-trained RoBERTa's \ourmethod{} score by layer}
\end{subfigure}
\begin{subfigure}{.49\linewidth}
  \centering
  \includegraphics[width=\linewidth]{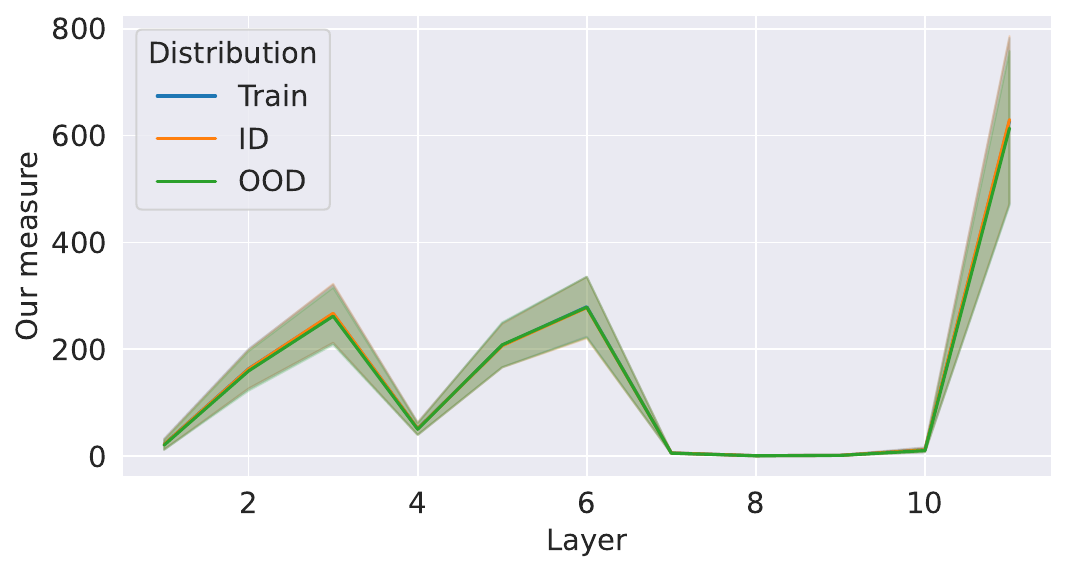}
  \caption{Pre-trained ELECTRA \ourmethod{} score by layer}
\end{subfigure}
\begin{subfigure}{.49\linewidth}
  \centering
  \includegraphics[width=\linewidth]{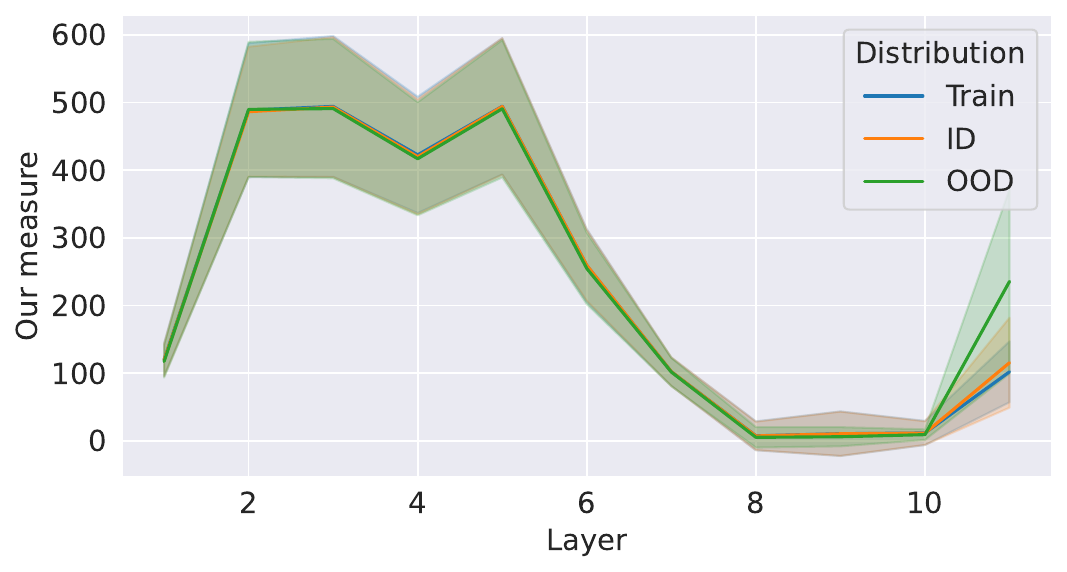}
  \caption{Fine-tuned RoBERTa's \ourmethod{} score by layer}
\end{subfigure}
\begin{subfigure}{.49\linewidth}
  \centering
  \includegraphics[width=\linewidth]{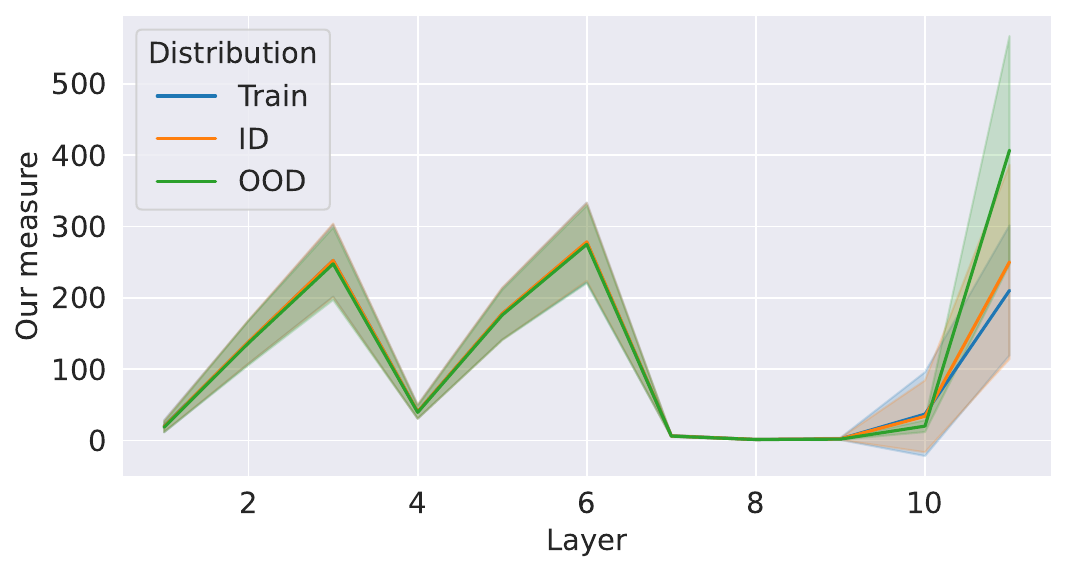}
  \caption{Fine-tuned ELECTRA \ourmethod{} score by layer}
\end{subfigure}
\caption{The impact of change of each layer on \ourmethod{} score across layers. Top row: Change in intermediate representations of training instances by layer for (a) RoBERTa and (b) ELECTRA. The scores are averaged across instances for the \textsc{sst} dataset. The black error bars denote the standard deviation. Middle row: \ourmethod{} score by layer of models for \textsc{sst} before fine-tuning. Bottom row: \ourmethod{} score by layer of models for \textsc{sst} after fine-tuning.}
\label{fig:repr-change-sst}
\end{figure}

\begin{figure}[t!]
\small
\centering
\begin{subfigure}{.49\linewidth}
  \centering
  \includegraphics[width=\linewidth]{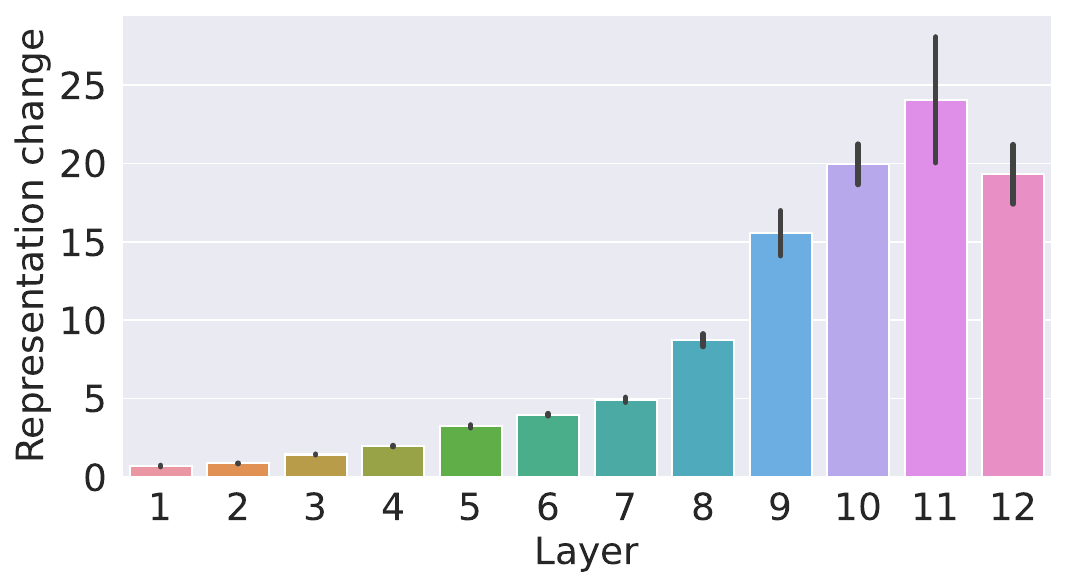}
  \caption{RoBERTa's representation change}
\end{subfigure}
\begin{subfigure}{.49\linewidth}
  \centering
  \includegraphics[width=\linewidth]{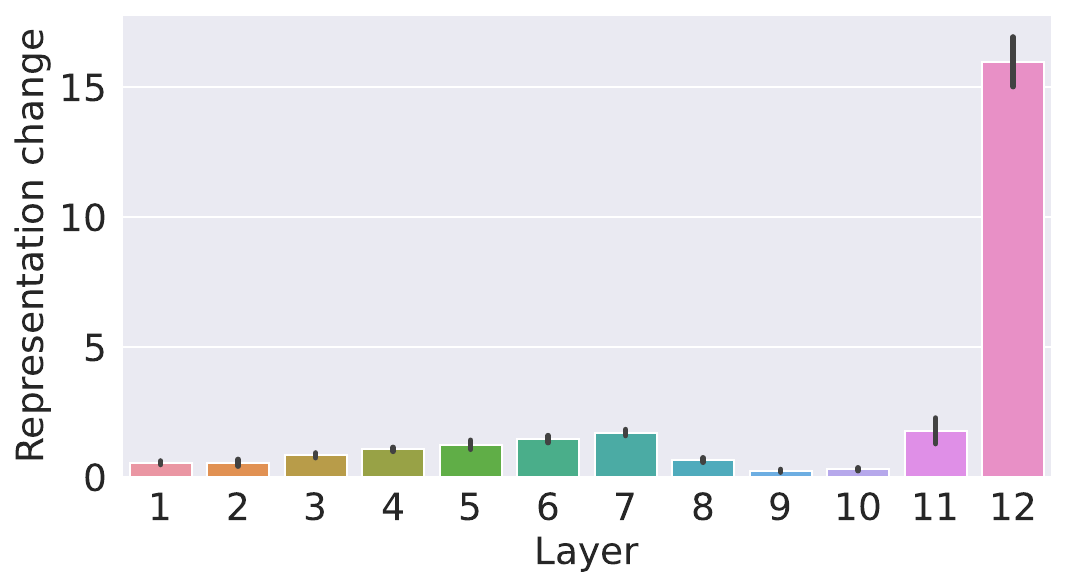}
  \caption{ELECTRA's representation change}
\end{subfigure}
\begin{subfigure}{.49\linewidth}
  \centering
  \includegraphics[width=\linewidth]{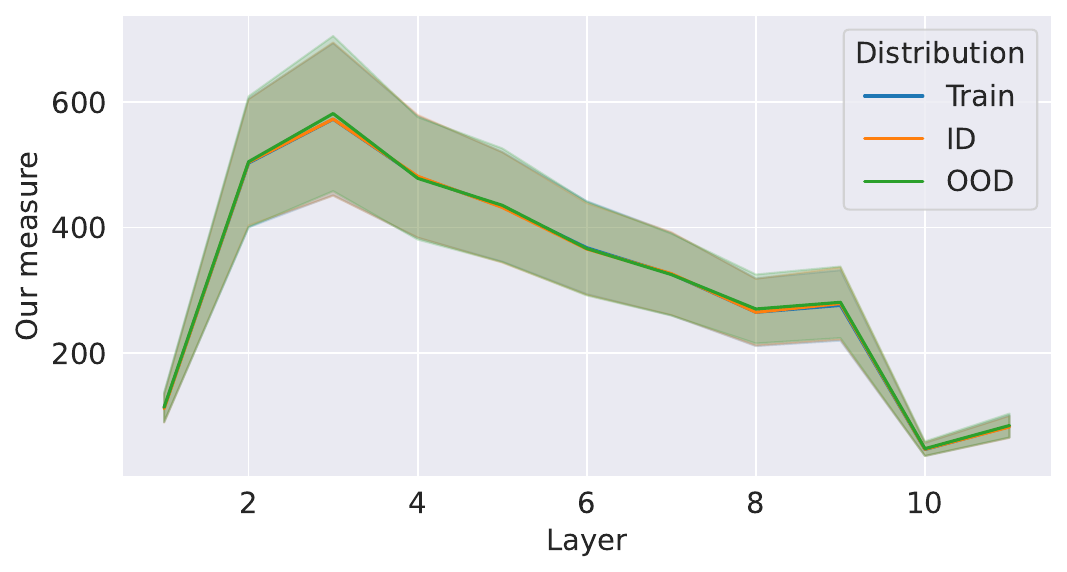}
  \caption{Pre-trained RoBERTa's \ourmethod{} score by layer}
\end{subfigure}
\begin{subfigure}{.49\linewidth}
  \centering
  \includegraphics[width=\linewidth]{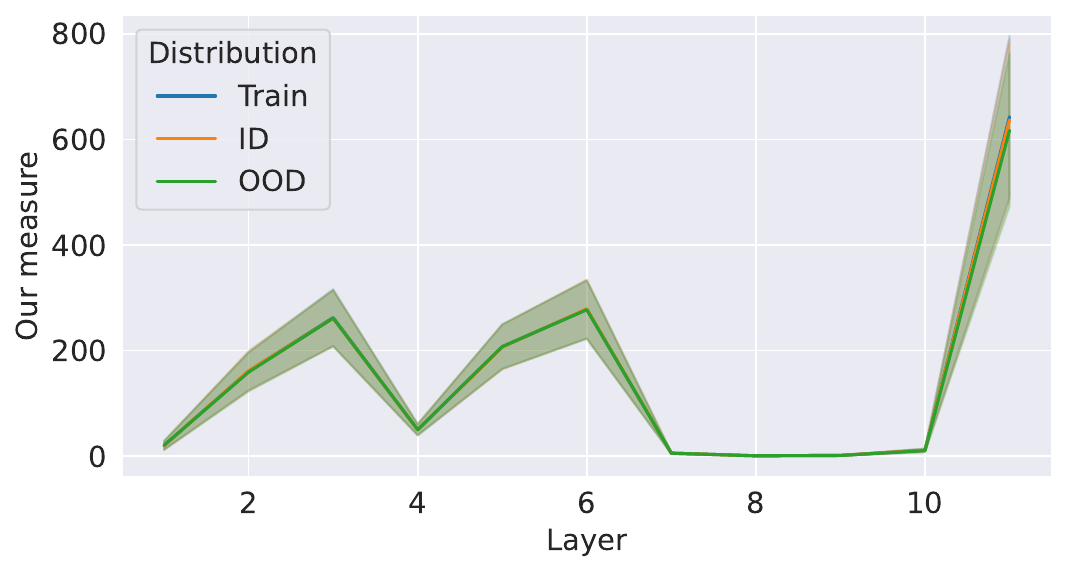}
  \caption{Pre-trained ELECTRA \ourmethod{} score by layer}
\end{subfigure}
\begin{subfigure}{.49\linewidth}
  \centering
  \includegraphics[width=\linewidth]{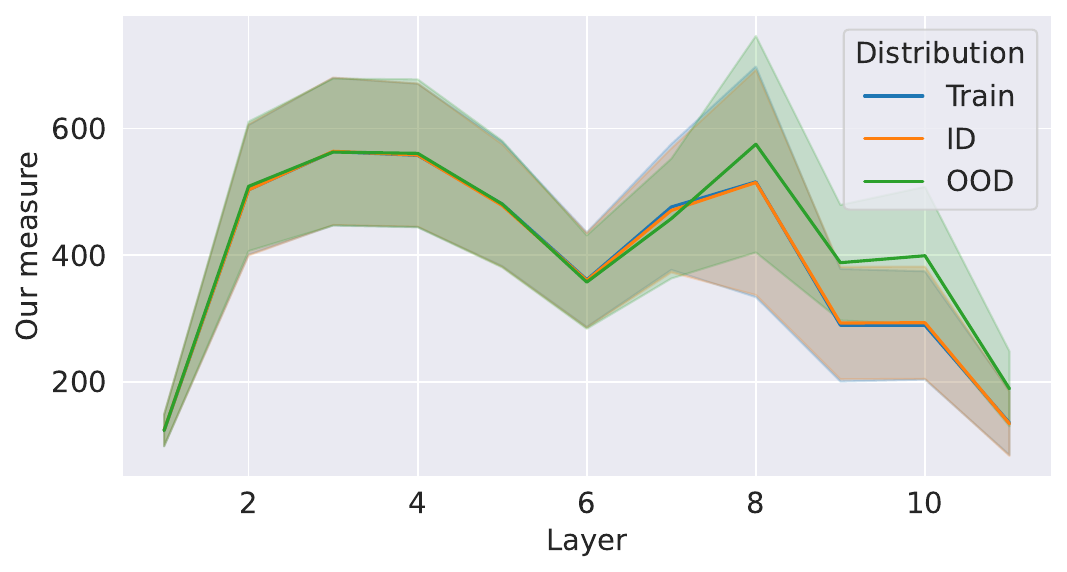}
  \caption{Fine-tuned RoBERTa's \ourmethod{} score by layer}
\end{subfigure}
\begin{subfigure}{.49\linewidth}
  \centering
  \includegraphics[width=\linewidth]{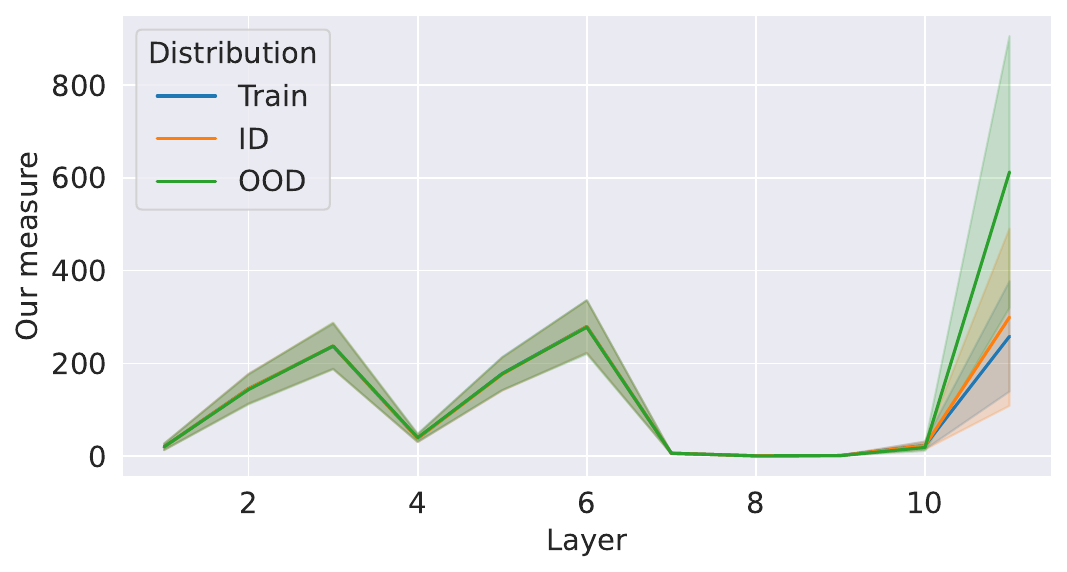}
  \caption{Fine-tuned ELECTRA \ourmethod{} score by layer}
\end{subfigure}
\caption{The impact of change of each layer on \ourmethod{} score across layers. Top row: Change in intermediate representations of training instances by layer for (a) RoBERTa and (b) ELECTRA. The scores are averaged across instances for the \textsc{subj} dataset. The black error bars denote the standard deviation. Middle row: \ourmethod{} score by layer of models for \textsc{subj} before fine-tuning. Bottom row: \ourmethod{} score by layer of models for \textsc{subj} after fine-tuning.}
\label{fig:repr-change-subj}
\end{figure}

\begin{figure}[t!]
\small
\centering
\begin{subfigure}{.49\linewidth}
  \centering
  \includegraphics[width=\linewidth]{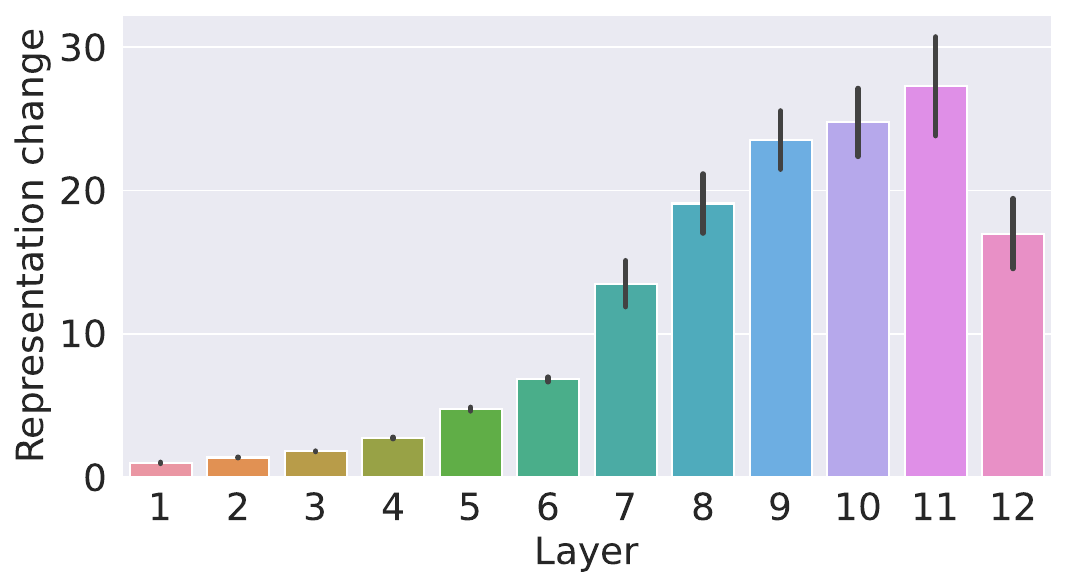}
  \caption{RoBERTa's representation change}
\end{subfigure}
\begin{subfigure}{.49\linewidth}
  \centering
  \includegraphics[width=\linewidth]{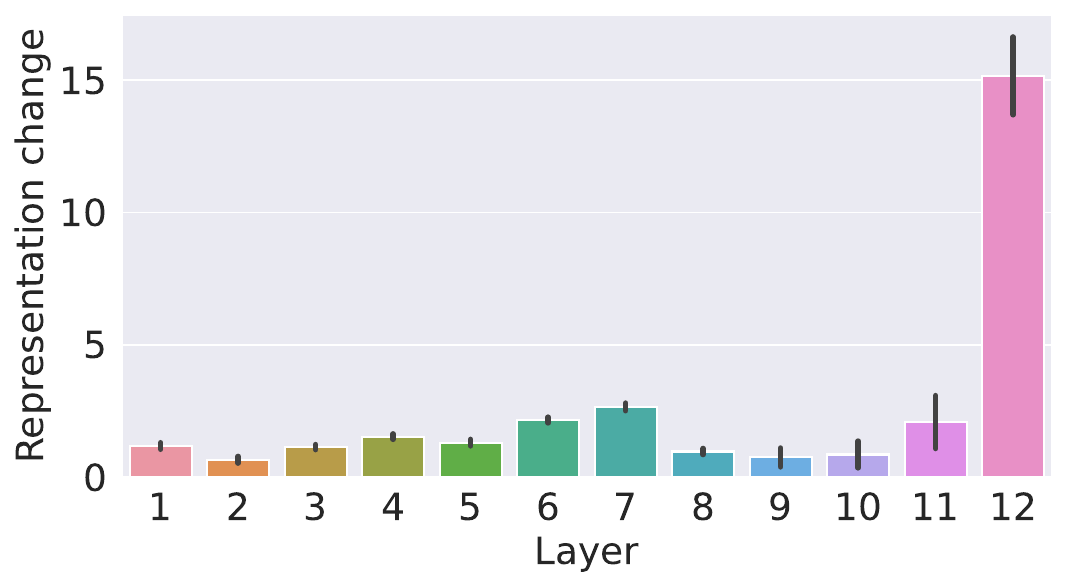}
  \caption{ELECTRA's representation change}
\end{subfigure}
\begin{subfigure}{.49\linewidth}
  \centering
  \includegraphics[width=\linewidth]{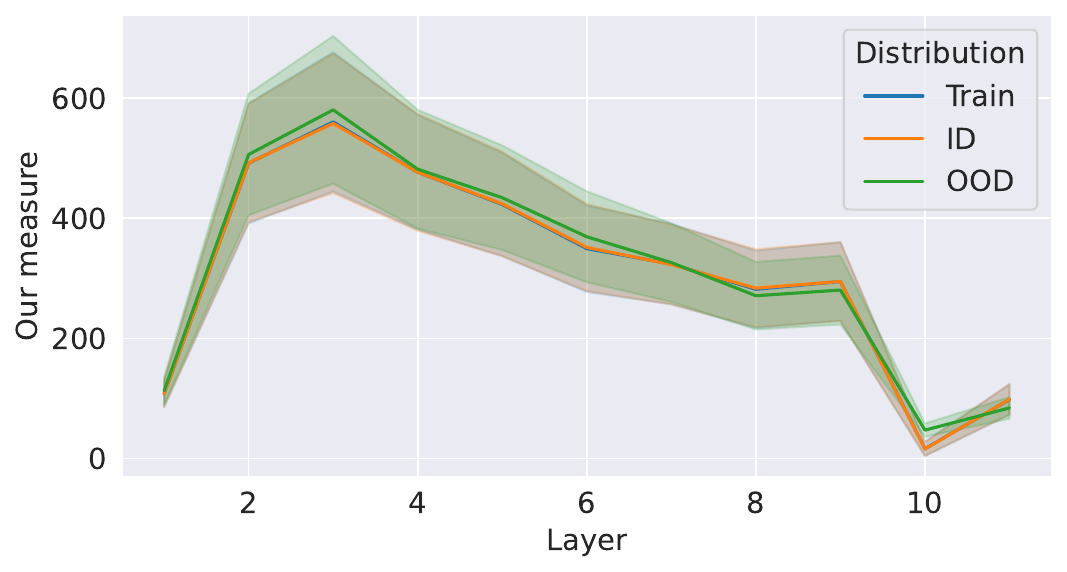}
  \caption{Pre-trained RoBERTa's \ourmethod{} score by layer}
\end{subfigure}
\begin{subfigure}{.49\linewidth}
  \centering
  \includegraphics[width=\linewidth]{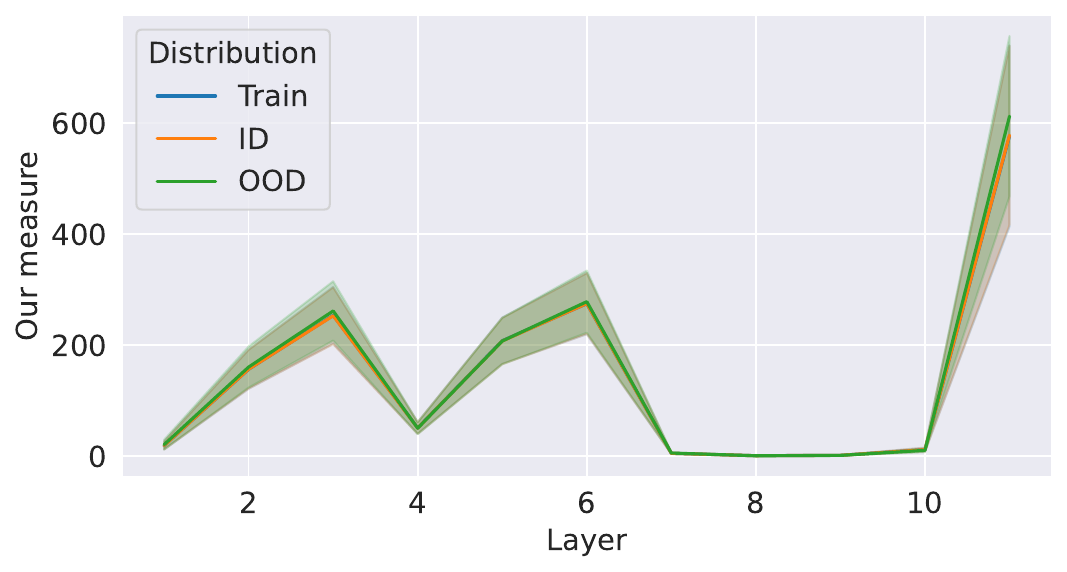}
  \caption{Pre-trained ELECTRA \ourmethod{} score by layer}
\end{subfigure}
\begin{subfigure}{.49\linewidth}
  \centering
  \includegraphics[width=\linewidth]{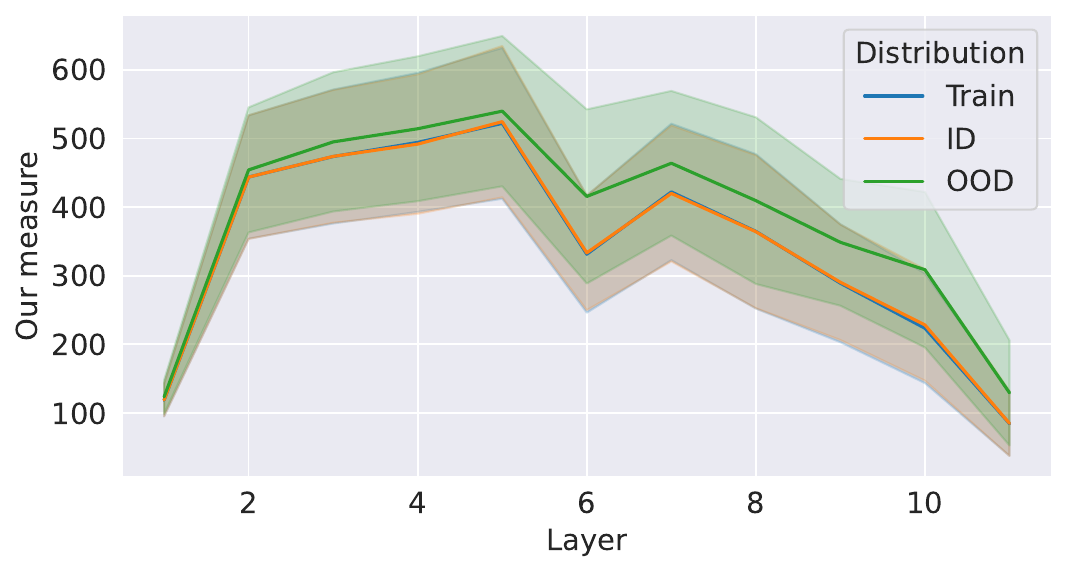}
  \caption{Fine-tuned RoBERTa's \ourmethod{} score by layer}
\end{subfigure}
\begin{subfigure}{.49\linewidth}
  \centering
  \includegraphics[width=\linewidth]{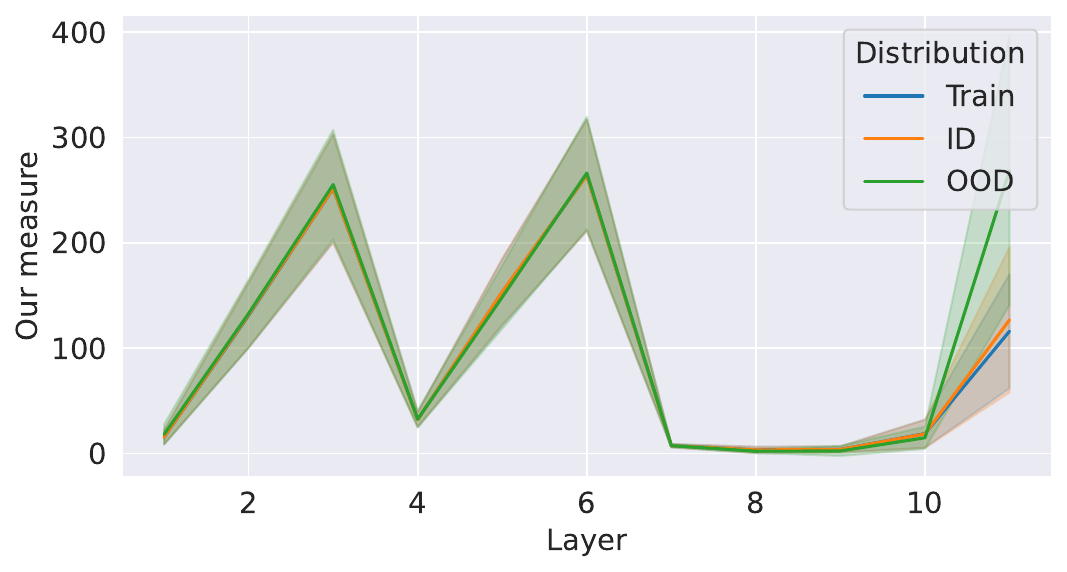}
  \caption{Fine-tuned ELECTRA \ourmethod{} score by layer}
\end{subfigure}
\caption{The impact of change of each layer on \ourmethod{} score across layers. Top row: Change in intermediate representations of training instances by layer for (a) RoBERTa and (b) ELECTRA. The scores are averaged across instances for the \textsc{agn} dataset. The black error bars denote the standard deviation. Middle row: \ourmethod{} score by layer of models for \textsc{agn} before fine-tuning. Bottom row: \ourmethod{} score by layer of models for \textsc{agn} after fine-tuning.}
\label{fig:repr-change-agn}
\end{figure}

\begin{figure}[t!]
\small
\centering
\begin{subfigure}{.49\linewidth}
  \centering
  \includegraphics[width=\linewidth]{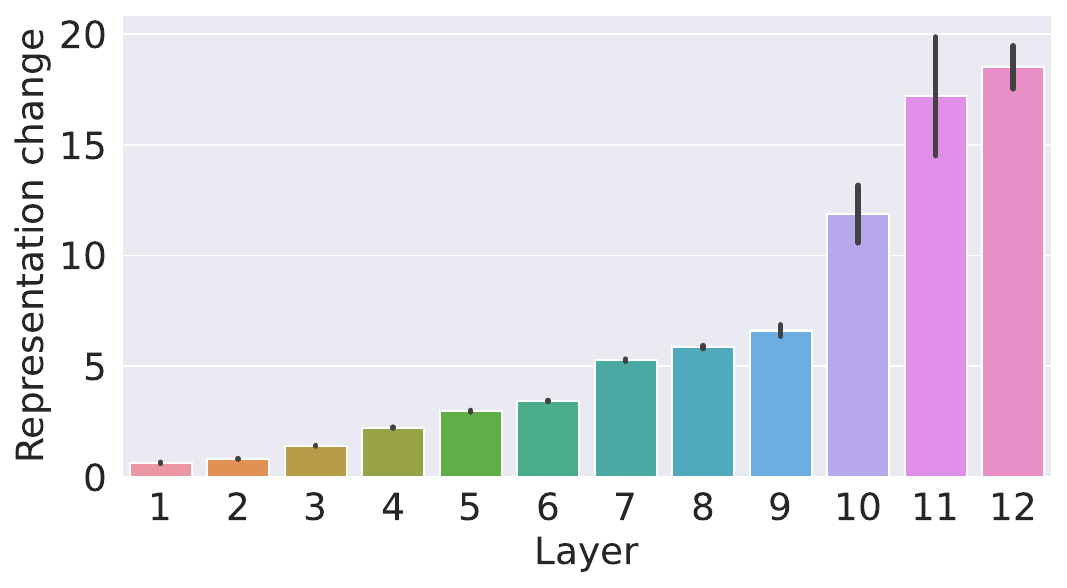}
  \caption{RoBERTa's representation change}
\end{subfigure}
\begin{subfigure}{.49\linewidth}
  \centering
  \includegraphics[width=\linewidth]{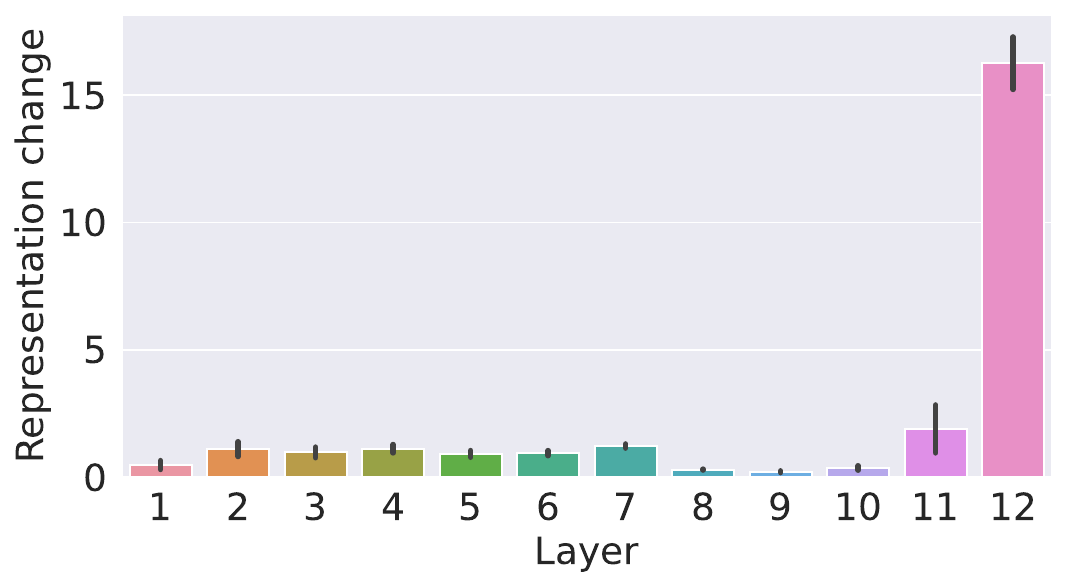}
  \caption{ELECTRA's representation change}
\end{subfigure}
\begin{subfigure}{.49\linewidth}
  \centering
  \includegraphics[width=\linewidth]{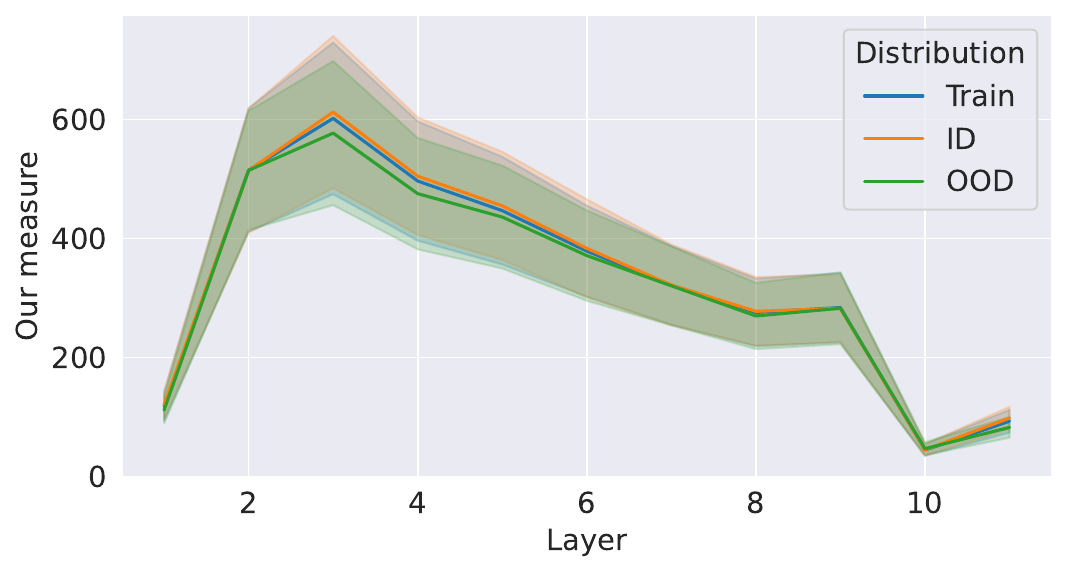}
  \caption{Pre-trained RoBERTa's \ourmethod{} score by layer}
\end{subfigure}
\begin{subfigure}{.49\linewidth}
  \centering
  \includegraphics[width=\linewidth]{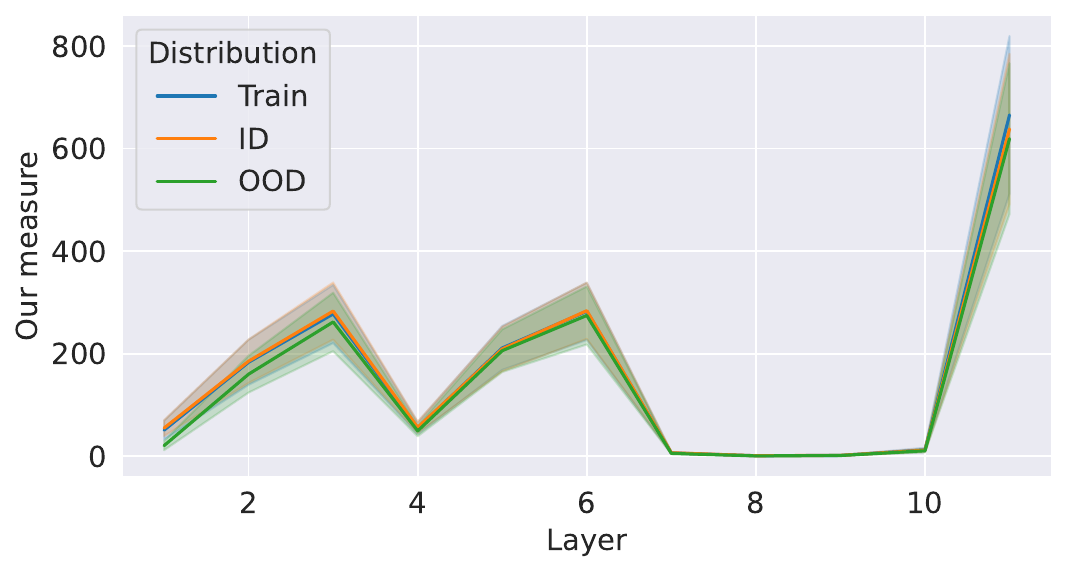}
  \caption{Pre-trained ELECTRA \ourmethod{} score by layer}
\end{subfigure}
\begin{subfigure}{.49\linewidth}
  \centering
  \includegraphics[width=\linewidth]{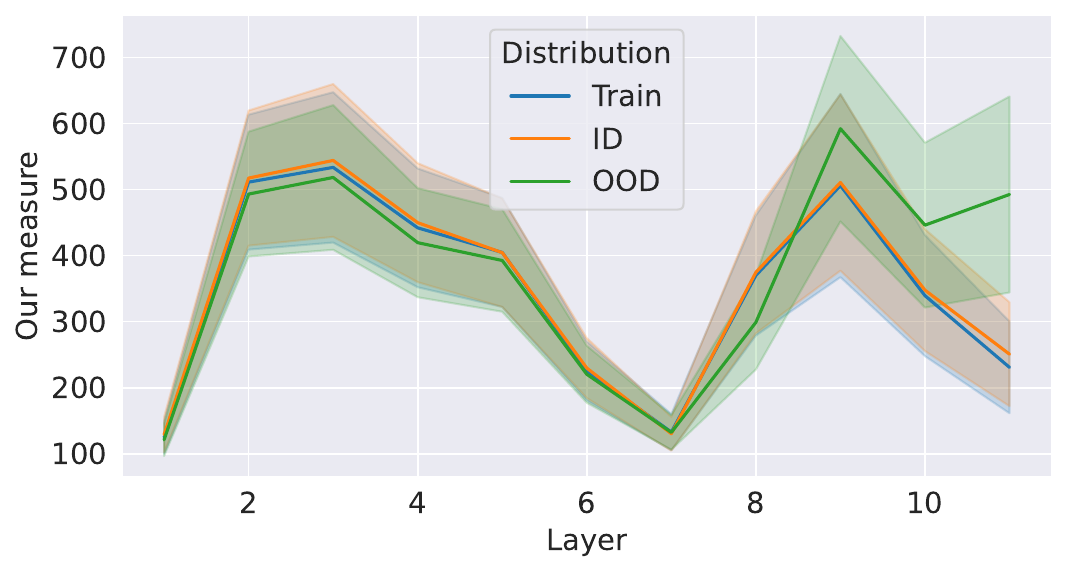}
  \caption{Fine-tuned RoBERTa's \ourmethod{} score by layer}
\end{subfigure}
\begin{subfigure}{.49\linewidth}
  \centering
  \includegraphics[width=\linewidth]{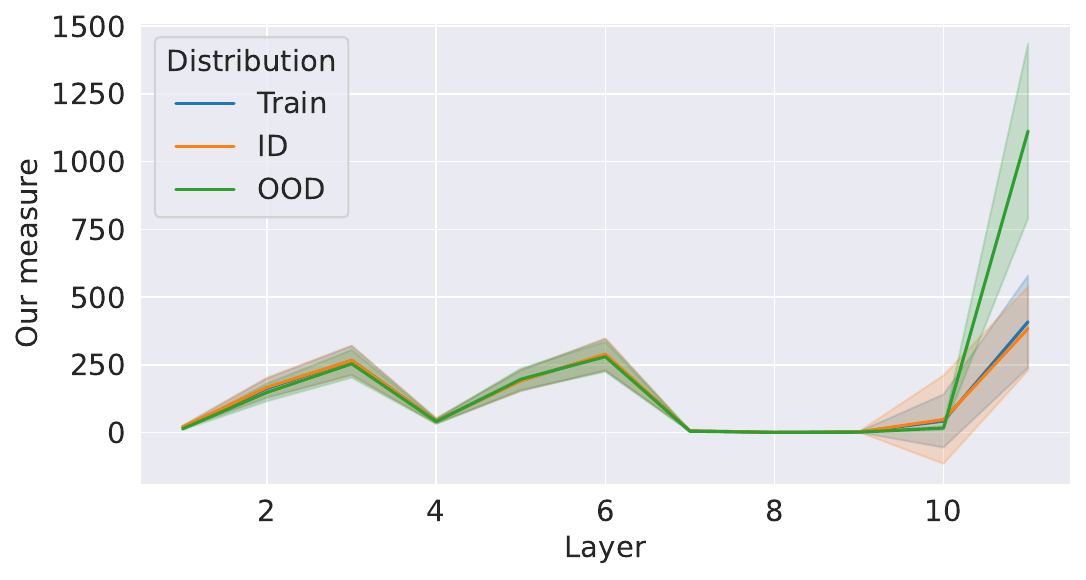}
  \caption{Fine-tuned ELECTRA \ourmethod{} score by layer}
\end{subfigure}
\caption{The impact of change of each layer on \ourmethod{} score across layers. Top row: Change in intermediate representations of training instances by layer for (a) RoBERTa and (b) ELECTRA. The scores are averaged across instances for the \textsc{trec} dataset. The black error bars denote the standard deviation. Middle row: \ourmethod{} score by layer of models for \textsc{trec} before fine-tuning. Bottom row: \ourmethod{} score by layer of models for \textsc{trec} after fine-tuning.}
\label{fig:repr-change-trec}
\end{figure}

\begin{figure}[t!]
\small
\centering
\begin{subfigure}{.49\linewidth}
  \centering
  \includegraphics[width=\linewidth]{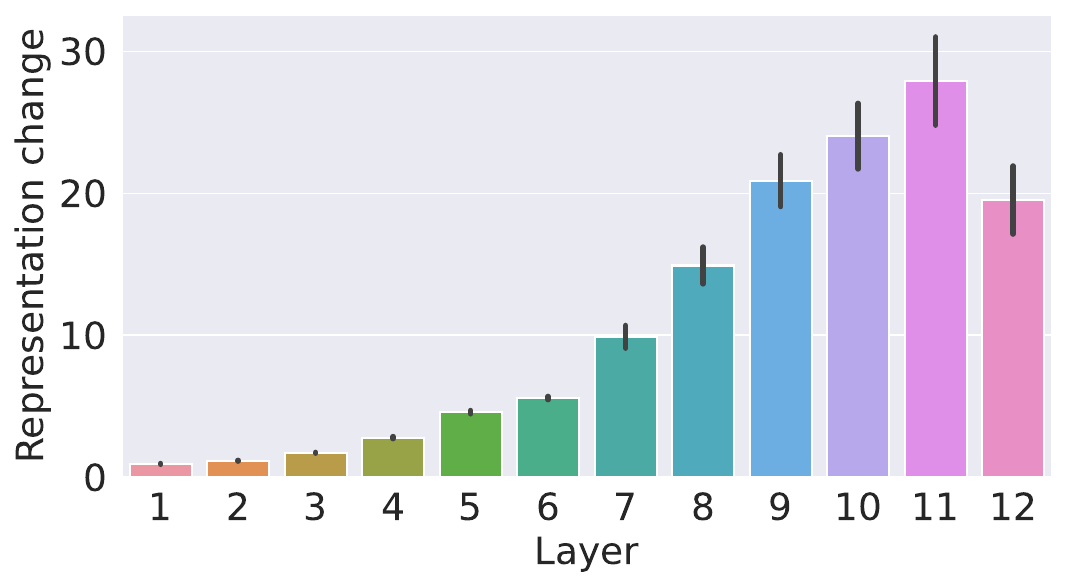}
  \caption{RoBERTa's representation change}
\end{subfigure}
\begin{subfigure}{.49\linewidth}
  \centering
  \includegraphics[width=\linewidth]{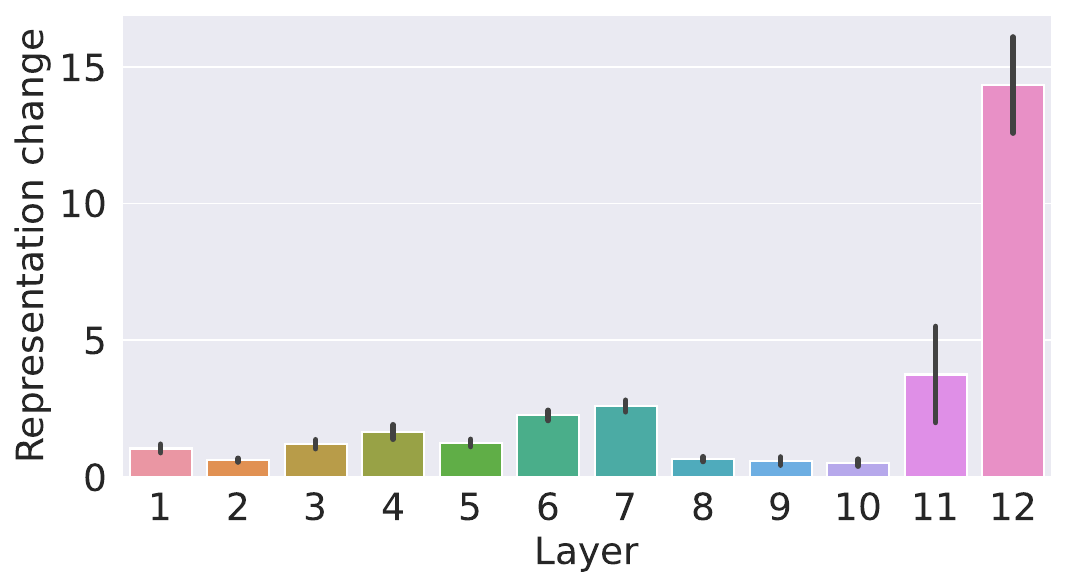}
  \caption{ELECTRA's representation change}
\end{subfigure}
\begin{subfigure}{.49\linewidth}
  \centering
  \includegraphics[width=\linewidth]{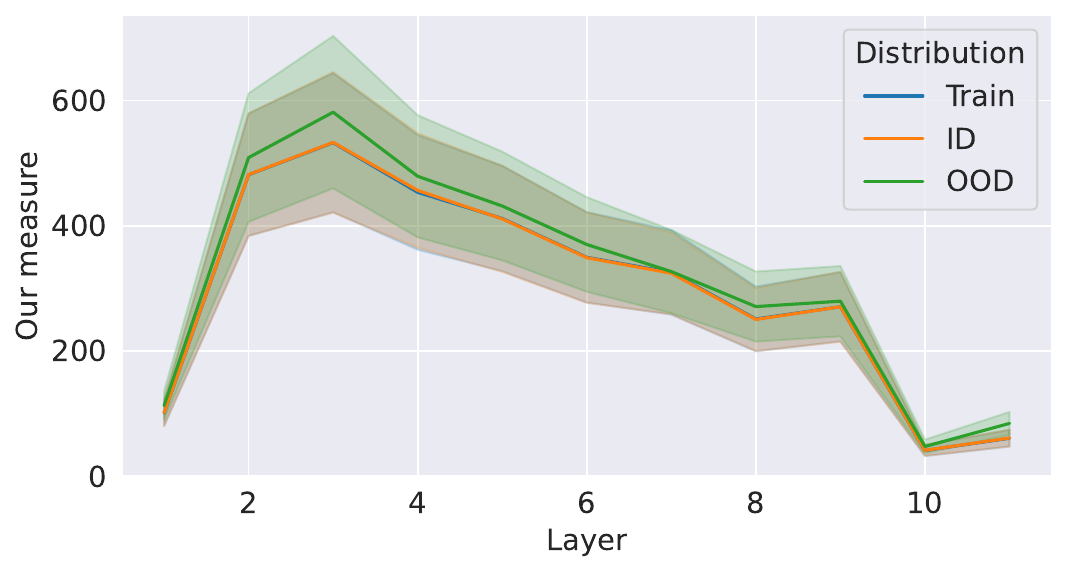}
  \caption{Pre-trained RoBERTa's \ourmethod{} score by layer}
\end{subfigure}
\begin{subfigure}{.49\linewidth}
  \centering
  \includegraphics[width=\linewidth]{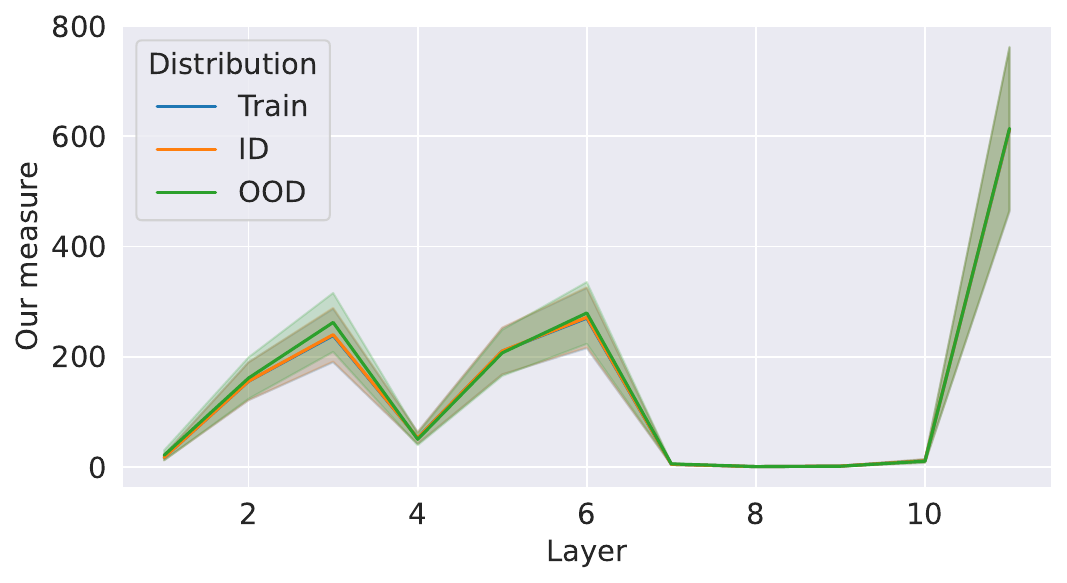}
  \caption{Pre-trained ELECTRA \ourmethod{} score by layer}
\end{subfigure}
\begin{subfigure}{.49\linewidth}
  \centering
  \includegraphics[width=\linewidth]{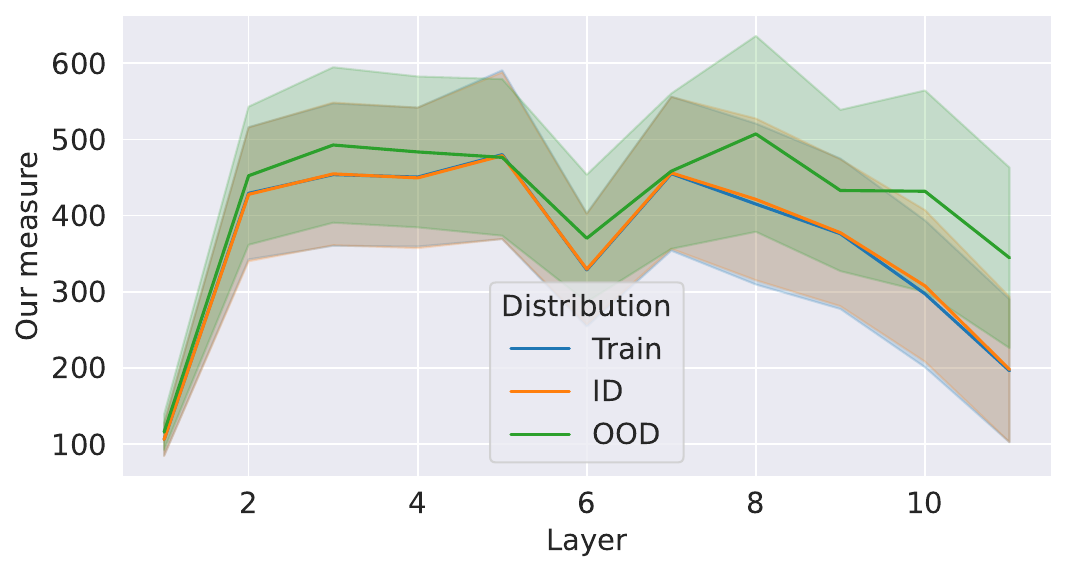}
  \caption{Fine-tuned RoBERTa's \ourmethod{} score by layer}
\end{subfigure}
\begin{subfigure}{.49\linewidth}
  \centering
  \includegraphics[width=\linewidth]{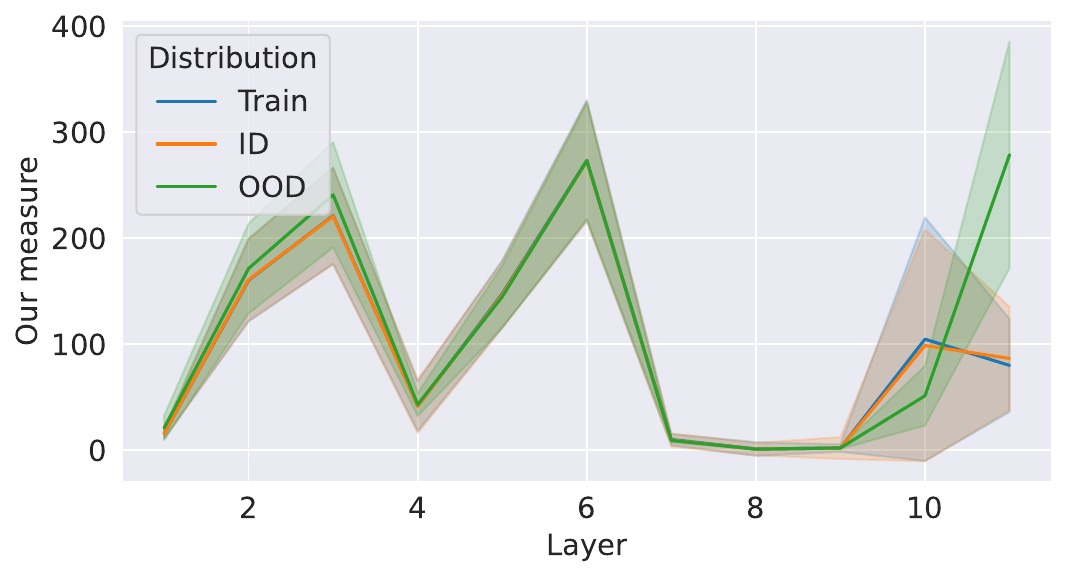}
  \caption{Fine-tuned ELECTRA \ourmethod{} score by layer}
\end{subfigure}
\caption{The impact of change of each layer on \ourmethod{} score across layers. Top row: Change in intermediate representations of training instances by layer for (a) RoBERTa and (b) ELECTRA. The scores are averaged across instances for the \textsc{bp} dataset. The black error bars denote the standard deviation. Middle row: \ourmethod{} score by layer of models for \textsc{bp} before fine-tuning. Bottom row: \ourmethod{} score by layer of models for \textsc{bp} after fine-tuning.}
\label{fig:repr-change-bp}
\end{figure}

\begin{figure}[t!]
\small
\centering
\begin{subfigure}{.49\linewidth}
  \centering
  \includegraphics[width=\linewidth]{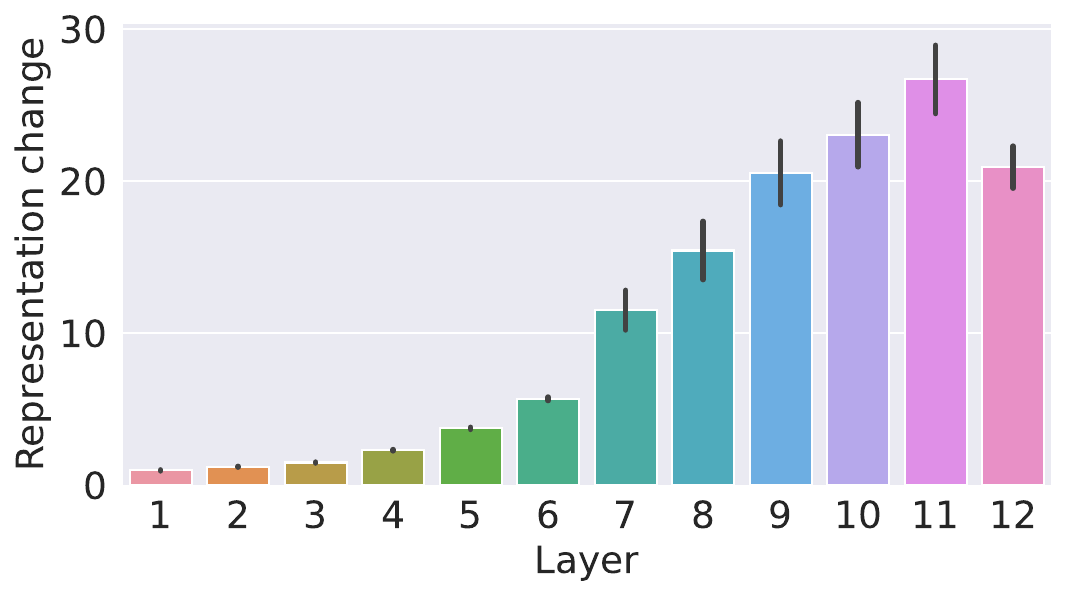}
  \caption{RoBERTa's representation change}
\end{subfigure}
\begin{subfigure}{.49\linewidth}
  \centering
  \includegraphics[width=\linewidth]{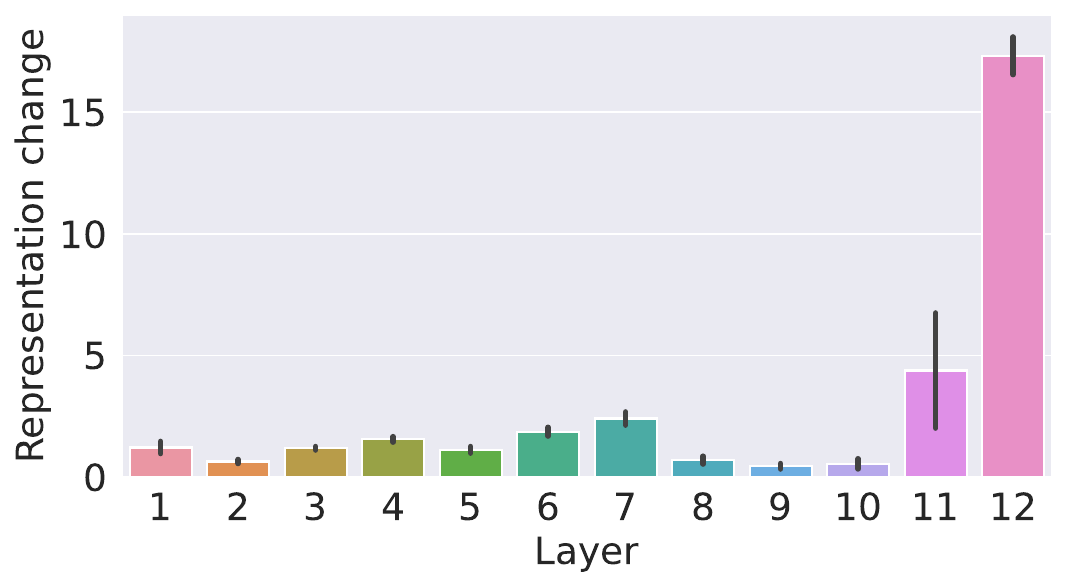}
  \caption{ELECTRA's representation change}
\end{subfigure}
\begin{subfigure}{.49\linewidth}
  \centering
  \includegraphics[width=\linewidth]{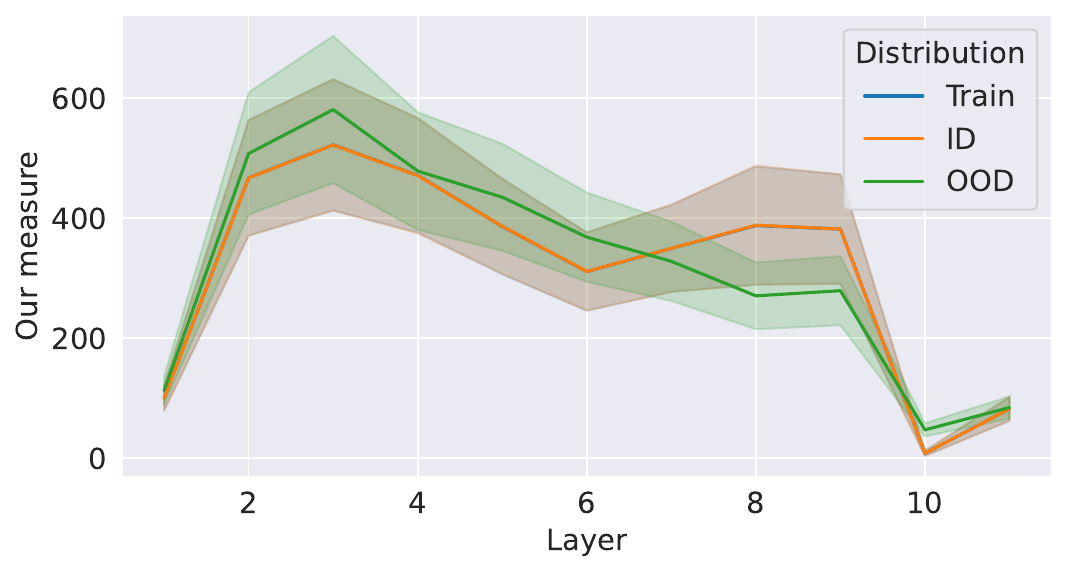}
  \caption{Pre-trained RoBERTa's \ourmethod{} score by layer}
\end{subfigure}
\begin{subfigure}{.49\linewidth}
  \centering
  \includegraphics[width=\linewidth]{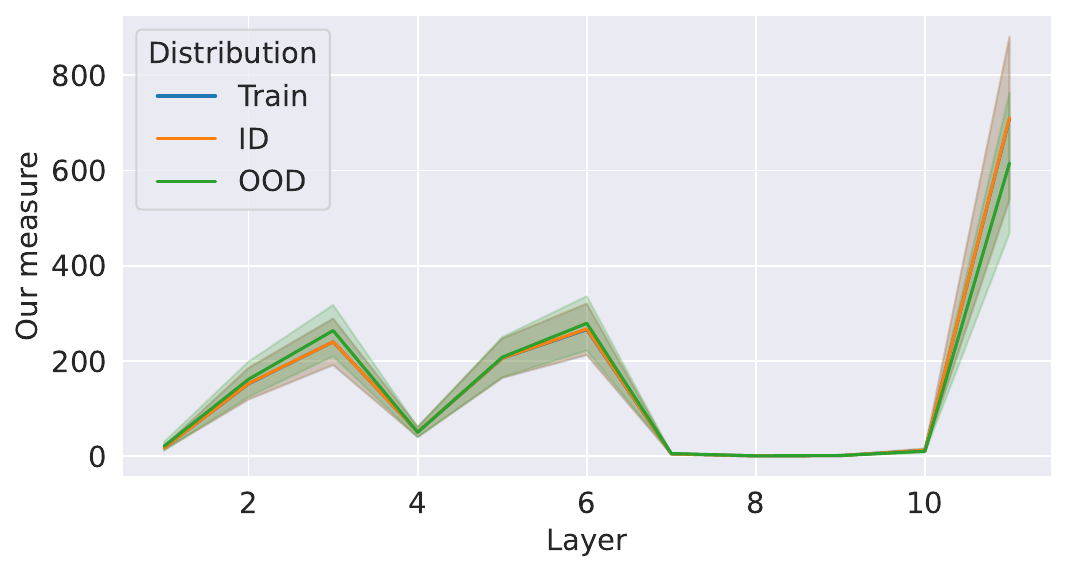}
  \caption{Pre-trained ELECTRA \ourmethod{} score by layer}
\end{subfigure}
\begin{subfigure}{.49\linewidth}
  \centering
  \includegraphics[width=\linewidth]{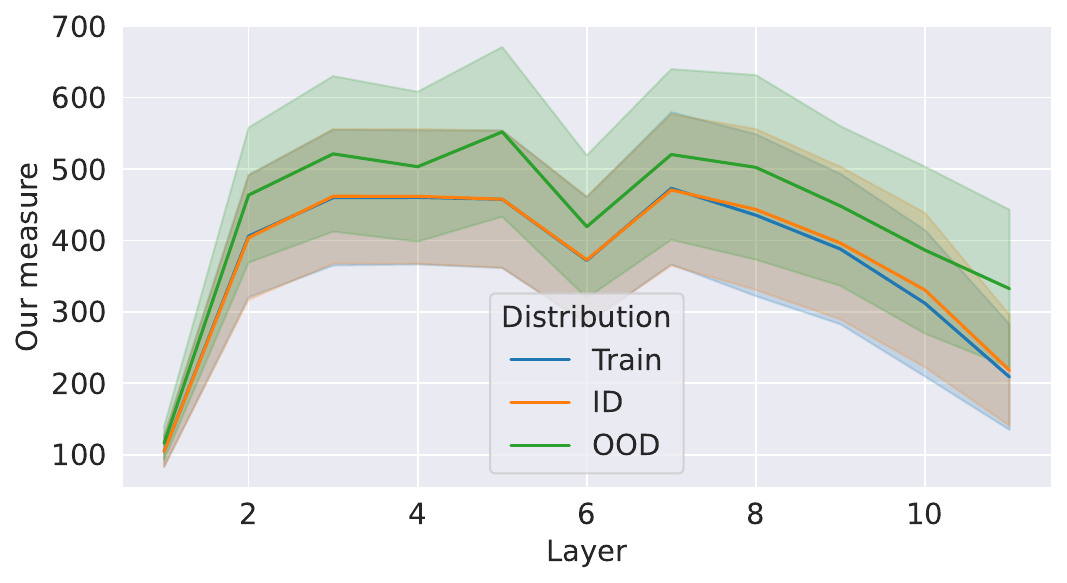}
  \caption{Fine-tuned RoBERTa's \ourmethod{} score by layer}
\end{subfigure}
\begin{subfigure}{.49\linewidth}
  \centering
  \includegraphics[width=\linewidth]{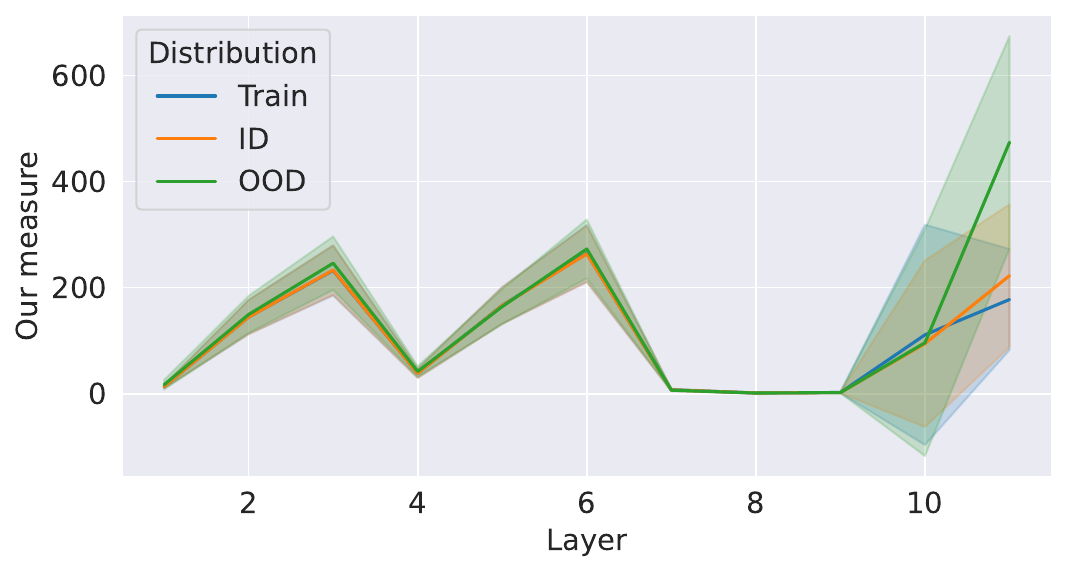}
  \caption{Fine-tuned ELECTRA \ourmethod{} score by layer}
\end{subfigure}
\caption{The impact of change of each layer on \ourmethod{} score across layers. Top row: Change in intermediate representations of training instances by layer for (a) RoBERTa and (b) ELECTRA. The scores are averaged across instances for the \textsc{mg} dataset. The black error bars denote the standard deviation. Middle row: \ourmethod{} score by layer of models for \textsc{mg} before fine-tuning. Bottom row: \ourmethod{} score by layer of models for \textsc{mg} after fine-tuning.}
\label{fig:repr-change-mg}
\end{figure}

\begin{figure}[t!]
\small
\centering
\begin{subfigure}{.49\linewidth}
  \centering
  \includegraphics[width=\linewidth]{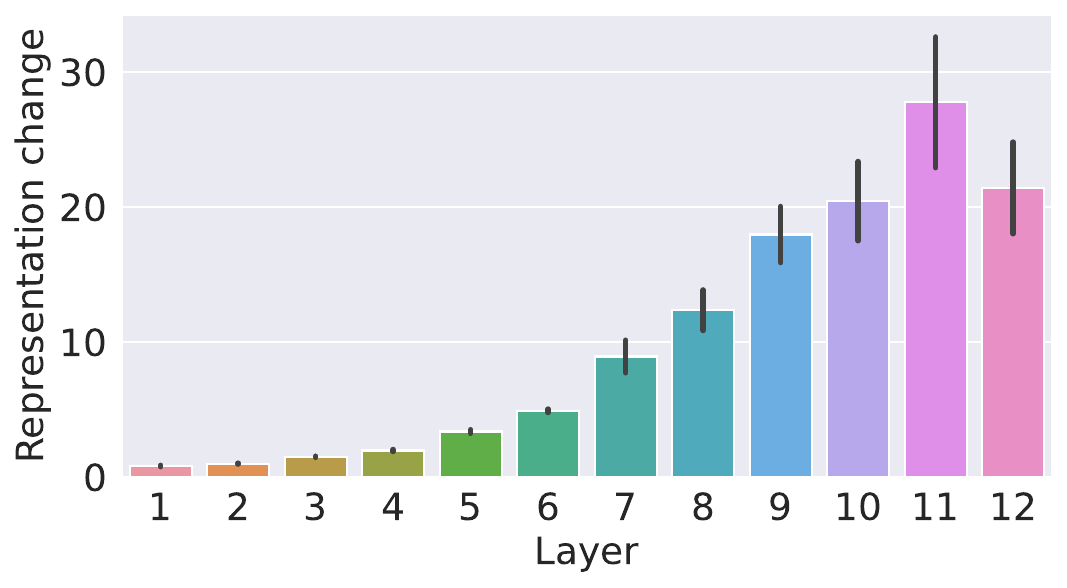}
  \caption{RoBERTa's representation change}
\end{subfigure}
\begin{subfigure}{.49\linewidth}
  \centering
  \includegraphics[width=\linewidth]{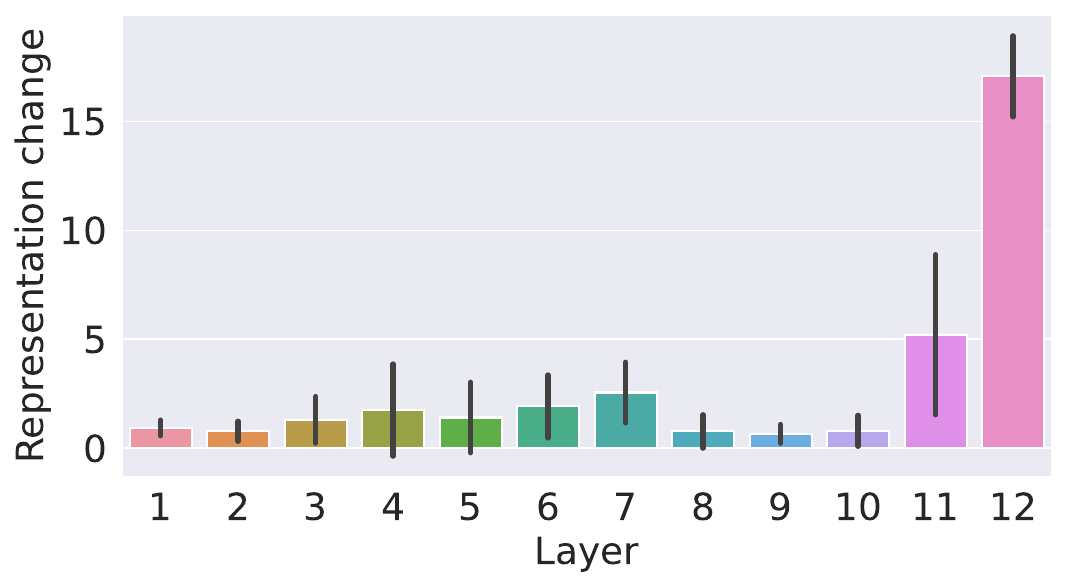}
  \caption{ELECTRA's representation change}
\end{subfigure}
\begin{subfigure}{.49\linewidth}
  \centering
  \includegraphics[width=\linewidth]{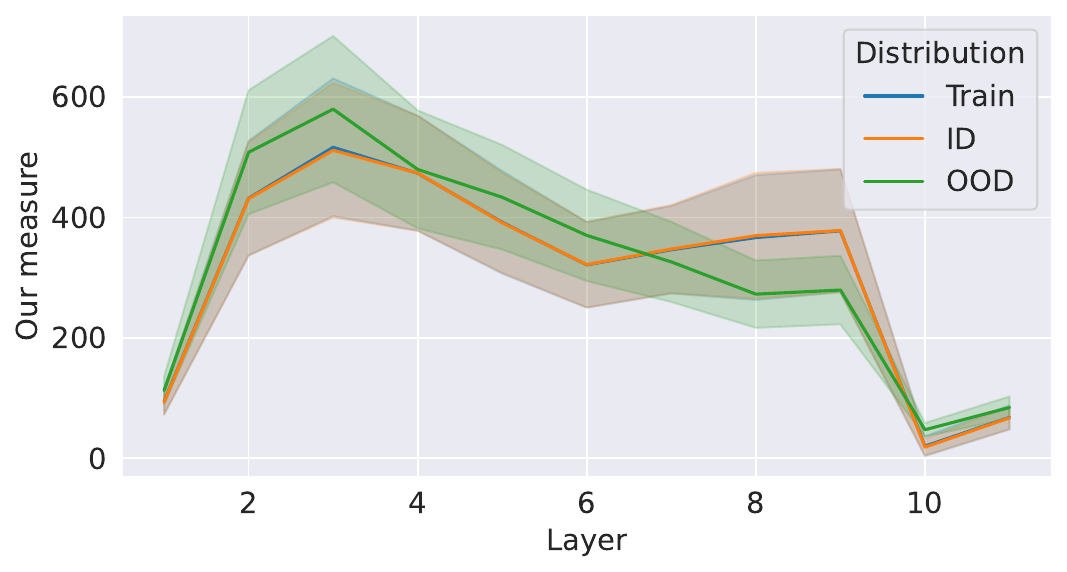}
  \caption{Pre-trained RoBERTa's \ourmethod{} score by layer}
\end{subfigure}
\begin{subfigure}{.49\linewidth}
  \centering
  \includegraphics[width=\linewidth]{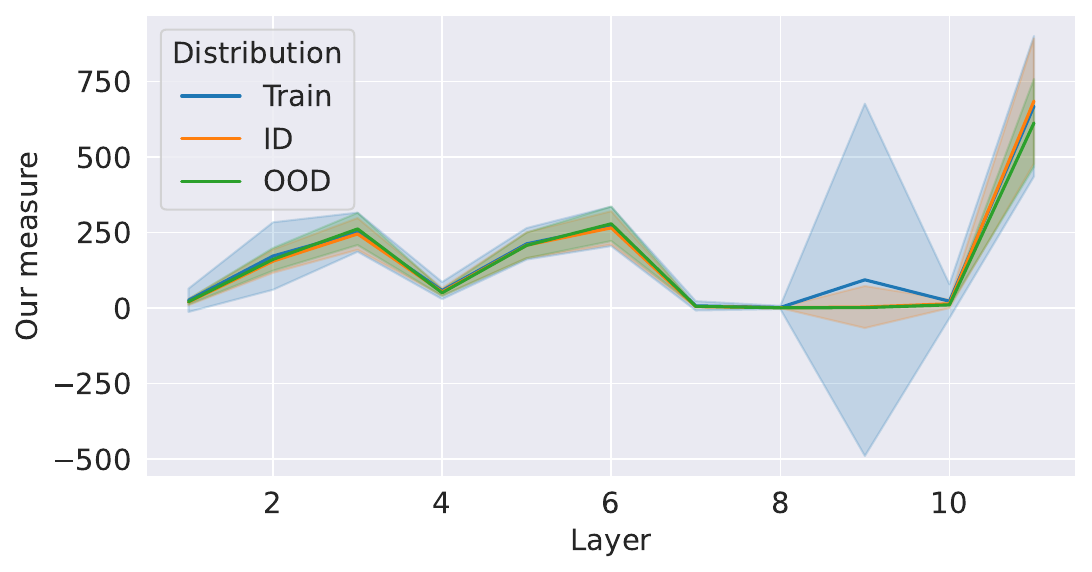}
  \caption{Pre-trained ELECTRA \ourmethod{} score by layer}
\end{subfigure}
\begin{subfigure}{.49\linewidth}
  \centering
  \includegraphics[width=\linewidth]{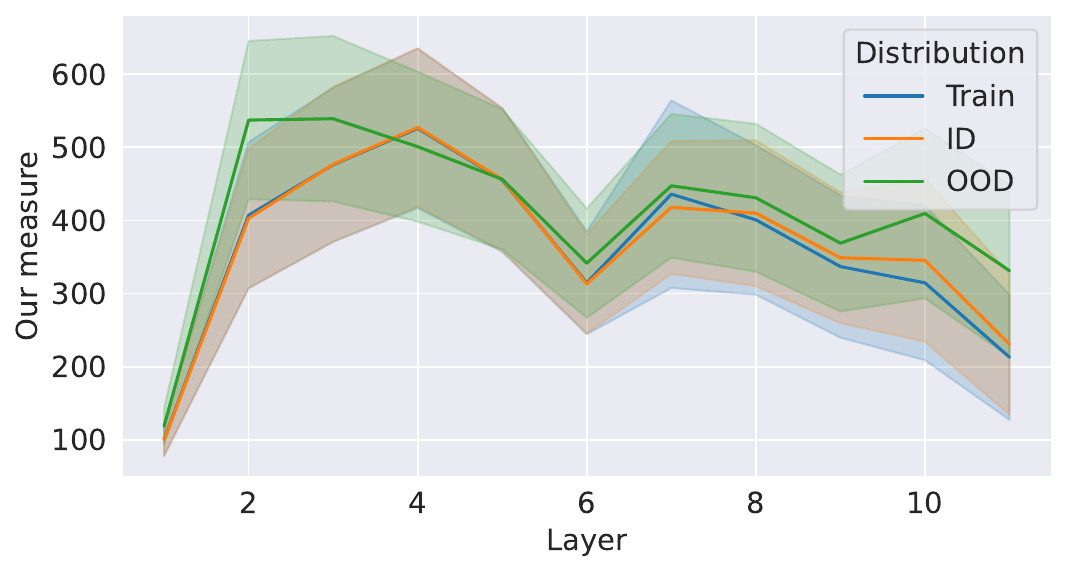}
  \caption{Fine-tuned RoBERTa's \ourmethod{} score by layer}
\end{subfigure}
\begin{subfigure}{.49\linewidth}
  \centering
  \includegraphics[width=\linewidth]{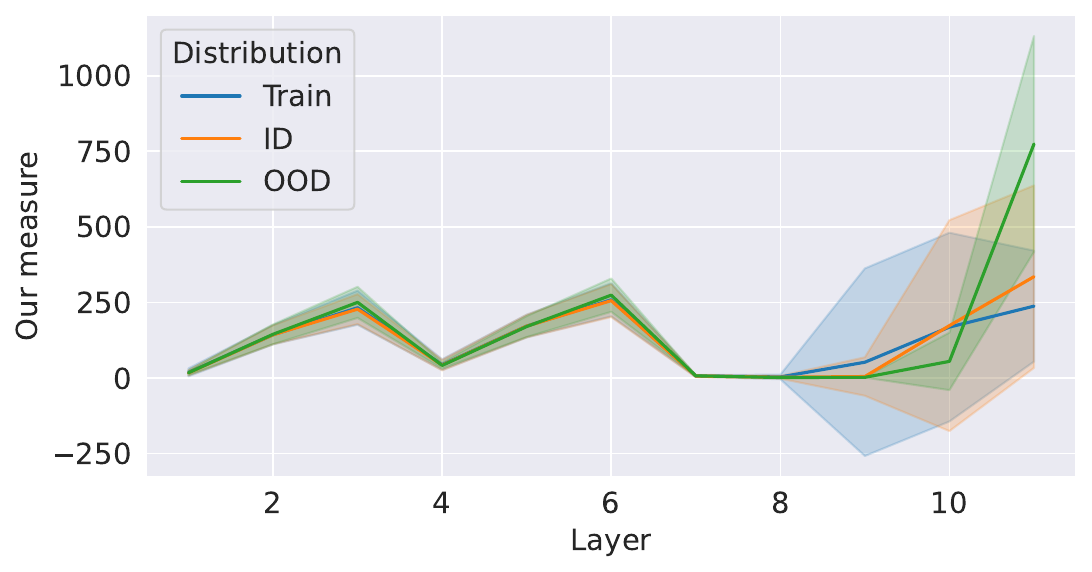}
  \caption{Fine-tuned ELECTRA \ourmethod{} score by layer}
\end{subfigure}
\caption{The impact of change of each layer on \ourmethod{} score across layers. Top row: Change in intermediate representations of training instances by layer for (a) RoBERTa and (b) ELECTRA. The scores are averaged across instances for the \textsc{ng} dataset. The black error bars denote the standard deviation. Middle row: \ourmethod{} score by layer of models for \textsc{ng} before fine-tuning. Bottom row: \ourmethod{} score by layer of models for \textsc{ng} after fine-tuning.}
\label{fig:repr-change-ng}
\end{figure}

\end{document}